%% file: tex_arXiv/arXiv_main.tex
\documentclass{article} 
\usepackage{iclr2027_conference,times}
\usepackage{amsmath}
\DeclareMathOperator{\msign}{msign}

\input{math_commands.tex}

\usepackage{afterpage}
\usepackage{hyperref}
\usepackage{url}
\usepackage{algorithm}
\usepackage{algpseudocode}
\usepackage{subcaption}
\usepackage{booktabs}
\usepackage{cleveref}
\crefname{assumption}{Assumption}{Assumptions}
\Crefname{assumption}{Assumption}{Assumptions}
\usepackage{amsmath,amssymb,amsthm}
\usepackage{graphicx}
\usepackage[table]{xcolor}
\usepackage{dsfont}
\definecolor{okgreen}{RGB}{34,139,34}
\definecolor{badred}{RGB}{200,40,40}
\definecolor{oursblue}{RGB}{228,238,252}
\usepackage{enumitem}
\usepackage{pifont}
\usepackage[most]{tcolorbox}
\newtcolorbox{takeaway}{
  colback=blue!4, colframe=blue!45!black,
  boxrule=0.6pt, arc=2pt,
  left=6pt, right=6pt, top=4pt, bottom=4pt,
  fontupper=\itshape
}

\crefname{appendix}{Appendix}{Appendices}
\Crefname{appendix}{Appendix}{Appendices}
\crefname{appxlemma}{Lemma}{Lemmas}
\Crefname{appxlemma}{Lemma}{Lemmas}
\crefname{appxcorollary}{Corollary}{Corollaries}
\Crefname{appxcorollary}{Corollary}{Corollaries}

\newcommand{\yes}{\textcolor{okgreen}{\ding{52}}}
\newcommand{\no}{\textcolor{badred}{\ding{56}}}

\newtheorem{theorem}{Theorem}
\newtheorem{proposition}{Proposition}
\newtheorem{lemma}{Lemma}

\newtheorem{definition}{Definition}

\newtheorem{assumption}{Assumption}
\newtheorem{appxtheorem}{Theorem}[section]
\newtheorem{appxlemma}[appxtheorem]{Lemma}
\newtheorem{appxproposition}[appxtheorem]{Proposition}
\newtheorem{appxcorollary}[appxtheorem]{Corollary}
\newtheorem{appxdefinition}[appxtheorem]{Definition}

\theoremstyle{remark}
\newtheorem*{remark*}{Remark}

\newcommand{\NS}{\ensuremath{\mathrm{NS}}}
\newcommand{\RN}{\ensuremath{\mathrm{RN}}}
\newcommand{\CN}{\ensuremath{\mathrm{CN}}}

\newcommand{\AN}{\ensuremath{\mathrm{AN}}}

\newcommand{\Sgn}{\ensuremath{\mathrm{Sgn}}}
\DeclareMathOperator{\diag}{diag}
\newcommand{\op}{\mathrm{op}}
\newcommand{\tr}{\operatorname{tr}}
\newcommand{\Acc}{\operatorname{Acc}}
\newcommand{\col}{\operatorname{col}}

\newcommand{\caE}{\mathcal{E}}
\newcommand{\K}{\mathcal{K}}

\crefname{appxtheorem}{Theorem}{Theorems}
\Crefname{appxtheorem}{Theorem}{Theorems}

\crefname{appxlemma}{Lemma}{Lemmas}
\Crefname{appxlemma}{Lemma}{Lemmas}

\crefname{appxproposition}{Proposition}{Propositions}
\Crefname{appxproposition}{Proposition}{Propositions}

\crefname{appxcorollary}{Corollary}{Corollaries}
\Crefname{appxcorollary}{Corollary}{Corollaries}

\crefname{appxdefinition}{Definition}{Definitions}
\Crefname{appxdefinition}{Definition}{Definitions}

\title{Optimizing Large Language Models with Chained LMOs}

\author{
Sungyoon Kim$^{1,*}$, Kaan Ozkara$^{2}$, Youngsuk Park$^{2}$\\
$^{1}$Department of Electrical Engineering, Stanford University\\
$^{2}$Amazon Annapurna Labs\\
\texttt{sykim777@stanford.edu, kaanozka@amazon.com, pyoungsu@amazon.com}
}

\iclrfinalcopy 
\begin{document}
\maketitle
\fancyhead{}
\renewcommand{\headrulewidth}{0pt}

\begingroup
\renewcommand{\thefootnote}{*}
\footnotetext{Work done during an internship at Amazon Annapurna Labs.}
\endgroup

\begin{abstract}
Muon has motivated a growing family of optimizers that compose multiple matrix normalizations, but these methods remain fragmented and lack a unified perspective. We introduce \emph{chained linear minimization oracles (chained LMOs)}, which cast these methods as compositions of LMOs. Despite their empirical success, many chains fall outside the standard LMO framework and can diverge on smooth convex objectives. To explain why composition can nevertheless help, we turn to linear associative memory and show that chaining can improve over Muon under anisotropic embeddings. Empirically, we propose \emph{TensorChain}, a novel optimizer within the framework that stacks compatible weight matrices across different layers and normalizes the 3d tensor across its axes. In Qwen3 0.6B and 1.7B pretraining, TensorChain outperforms all chained baselines in average token efficiency, with average token savings of $9.6\%$ over Muon at matched validation loss.
\end{abstract}

\section{Introduction}
\label{sec:introduction}
\input{tex_arXiv/1.intro}

\section{Optimizing with chained LMOs}
\label{sec:chainedLMO}
\input{tex_arXiv/2.chainedLMO}

\section{Provable Benefit of Chaining in Linear Associative Memory Learning}
\label{sec:lam_separation}
\input{tex_arXiv/3.assocmem}

\section{Exploring the Space of Chains}
\label{sec:nav}
\input{tex_arXiv/4.TC}

\section{Experiments}
\label{sec:experiments}
\input{tex_arXiv/5.experiments}

\section{Conclusion}
\label{sec:conclusion}
\input{tex_arXiv/6.conclusion}

\newpage
\subsection*{AI use statement}

In this paper, AI tools were used in the following capacities:
\begin{itemize}[leftmargin=*]
    \item Assistance in proof verification and organization

    \item Assistance in novel proofs in Section 3; specifically on probabilistic rigor
    
    \item Coding assistants were used to help implement the proposed methods, run experimental pipelines, and retrieve and organize experimental results.
\end{itemize}



\bibliography{iclr2027_conference}
\bibliographystyle{iclr2027_conference}

\newpage
\appendix

\input{tex_arXiv/appendix.tex}

\end{document}

%% file: math_commands.tex
\usepackage{amsmath,amsfonts,bm}

\def\eqref#1{equation~\ref{#1}}

\def\1{\bm{1}}

\DeclareMathAlphabet{\mathsfit}{\encodingdefault}{\sfdefault}{m}{sl}
\SetMathAlphabet{\mathsfit}{bold}{\encodingdefault}{\sfdefault}{bx}{n}

\DeclareMathOperator*{\argmax}{arg\,max}
\DeclareMathOperator*{\argmin}{arg\,min}

\DeclareMathOperator{\sign}{sign}

%% file: tex_arXiv/1.intro.tex
After the success of the Muon optimizer \citep{modded_nanogpt_2024} in pretraining large language models, numerous theoretical and empirical follow-up works have been proposed. Consider minimizing a differentiable objective \(f:\mathbb{R}^{m\times n}\to\mathbb{R}\) over a weight matrix \(W\in\mathbb{R}^{m\times n}\). Muon maintains a momentum matrix \(M_t\in\mathbb{R}^{m\times n}\), the exponential moving average (EMA) of the gradient, and updates the weights along its matrix sign:
\begin{equation}
\label{eq1}
M_t=\beta M_{t-1}+(1-\beta)\nabla f(W_t),\quad
W_{t+1}=W_t-\eta_t\msign(M_t).
\end{equation}
Here, \(\beta\in[0,1)\) is the momentum coefficient and \(\eta_t>0\) is the learning rate. For a matrix \(M\) with singular value decomposition \(M=U\Sigma V^\top\), the matrix sign is defined as \(\msign(M)=UV^\top\); in practice, it is approximated by a few Newton--Schulz iterations rather than by an explicit SVD.

A principled way to understand this update is through the linear minimization oracle (LMO) \citep{pethick2025training,kovalev2025understanding}. For a norm ball \(\mathcal D=\{X:\|X\|\le\rho\}\), define
\[
\operatorname{lmo}_{\mathcal D}(G)
\in
\arg\min_{X\in\mathcal D}\langle G,X\rangle .
\]
Over the spectral-norm ball, \(\operatorname{lmo}_{\mathcal D}(G)=-\rho\,\msign(G)\), so Muon is exactly the iteration
\begin{equation}
\label{eq:lmo_update}
M_t=\beta M_{t-1}+(1-\beta)\nabla f(W_t),\quad
W_{t+1}=W_t+\eta_t\operatorname{lmo}_{\mathcal D}(M_t),
\end{equation}
with \(\rho=1\). This iteration updates in the steepest descent direction under the geometry of \(\mathcal{D}\), and inherits the standard convergence guarantees under the corresponding norm-smoothness assumptions \citep{pethick2025training}. Other norms recover other familiar normalizations: up to sign and scale, the max-row- and max-column-\(\ell_2\) balls yield row normalization (\(\RN\)) and column normalization (\(\CN\)), and the entrywise-\(\ell_\infty\) ball yields the elementwise sign (\(\mathrm{Sgn}\)).

Recent Muon variants augment the matrix-sign step with exactly these normalizations, applied before the matrix sign \citep{chang2026muoneq,liu2026muon}, after it \citep{zhang2026muon+,li2025normuon}, on both sides \citep{si2025adamuon}, or in repeated alternation \citep{dewulf2026aurora}. Although these methods were proposed separately, each with its own motivation and method-specific analysis, they share a common structure: each is a composition of LMOs.

We introduce \textbf{chained LMOs} as a unifying framework for these variants and evaluate them in a unified experimental protocol. Given constraint sets \(\mathcal D_1,\dots,\mathcal D_k\), we define the chained LMO
\begin{equation}
\label{eq:chain}
T(G)
:=
(-1)^{k+1}
\bigl(
\operatorname{lmo}_{\mathcal D_k}
\circ\cdots\circ
\operatorname{lmo}_{\mathcal D_1}
\bigr)(G),
\end{equation}
and the corresponding iteration
\begin{equation}
\label{eq:chainedlmo_update}
M_t=\beta M_{t-1}+(1-\beta)\nabla f(W_t),\quad
W_{t+1}=W_t+\eta_t T(M_t).
\end{equation}
The factor \((-1)^{k+1}\) accounts for the negation introduced by each LMO. Muon is the single-LMO case \(k=1\) with the spectral-norm ball, while many recent variants become instances of \cref{eq:chainedlmo_update}, either directly or, for methods with Adam-like state, after turning off the EMA terms following \citet{bernstein2024old} (\Cref{table:intro_lmochain}).


Having identified chained LMOs as a common representation of Muon variants, a natural question is whether the favorable theory of individual LMOs survives composition. We show that for proposed Muon variants in the literature \citep{zhang2026muon+, liu2026muon, chang2026muoneq, si2025adamuon, li2025normuon, dewulf2026aurora}, it does not. This exposes the limitations of the existing LMO perspective and motivates a complementary view through linear associative memory, where we show that chaining can itself provide a learning advantage. Empirically, we compare these variants under a common protocol. Since each was originally evaluated in a different setup, their reported results are not directly comparable; a controlled comparison is therefore valuable in that it isolates the effect of the chain itself from differences in architecture, scale, and tuning. We further introduce a novel normalization axis that outperforms existing chained baselines under the same protocol, suggesting that the chained-LMO design space remains far from exhausted. Specifically,
\begin{itemize}[leftmargin=*]

    \item \textbf{A unified theory of chained LMOs.} We introduce chained LMOs as a common representation of recent Muon variants and characterize the limits of the standard LMO perspective. Many practical chains are not LMOs of any fixed constraint set (\Cref{sec:separation_lmo}). For positively aligned chains such as $\RN\circ\msign$, a $\nu$-approximate LMO framework recovers convergence guarantees, but with unavoidable dimension dependence (\Cref{prop:dimension}). Other chains admit negatively aligned directions and can diverge even on smooth convex objectives (\Cref{prop:divergence}).
    \item \textbf{A provable learning advantage from chaining.}
    Building on recent analyses of Muon in linear associative memory \citep{wang2026muon,kim2026sharp}, we show that different normalizations can address distinct sources of anisotropy. Under anisotropic answer embeddings and power-law pair frequencies, $\msign\circ\RN$ can achieve asymptotically perfect retrieval while both Muon and gradient descent fail.
    \item \textbf{A new normalization axis and unified empirical comparison.}
    We introduce \textbf{TensorChain}, which treats the layer dimension of stacked parameter matrices as a novel axis of normalization. Under a common training setup   (\Cref{sec:setup}), TensorChain outperforms existing chained Muon baselines in the 1M-token batch regime, and remains competitive in the 4M-token batch regime.
\end{itemize}

The rest of the paper is organized as follows.
In \Cref{sec:chainedLMO}, we show that many chains in \Cref{table:intro_lmochain} are not LMOs of any fixed constraint set, and analyze their convergence and divergence on smooth objectives. \Cref{sec:lam_separation} offers a complementary perspective through linear associative memory, where chaining provably improves learning. \Cref{sec:nav} shows our exploration of different chains and introduces TensorChain, which is compared with existing chained-LMO methods under a matched training protocol \Cref{sec:experiments}. \Cref{sec:conclusion} concludes the paper. Related optimization frameworks and their relation to chained LMOs are discussed in \Cref{sec:appendix_framework}.





%% file: tex_arXiv/2.chainedlmo.tex
\afterpage{%
\begin{table}[t]
\centering
\caption{\textbf{Existing and proposed Muon variants as chained LMOs.} Methods with Adam-like updates reduce to chains once the EMA is turned off. RN/CN: row/column normalization; $\msign$: matrix sign; Sgn: elementwise sign; \emph{Chain conv.}: whether the stateless chain drives the gradient norm to zero. $^\dagger$Also considers $\RN\circ\CN\circ\msign$ and $\CN\circ\RN\circ\msign$.}
\vspace{0.5em}
\label{table:intro_lmochain}
\small
\begin{tabular}{llcl}
\hline
Method & LMO Chain & Chain conv. & Notes \\
\hline
Muon$+$ \citep{zhang2026muon+}     & $\RN/\CN\circ\msign$$^\dagger$                & \yes & -- \\
NorMuon \citep{li2025normuon}      & $\RN\circ\msign$                              & \yes & EMA \\
MuonEq \citep{chang2026muoneq}     & $\msign\circ\RN/\CN$                          & \yes & -- \\
Muon$^2$ \citep{liu2026muon}       & $\msign\circ\Sgn$                             & \no  & EMA \\
AdaMuon \citep{si2025adamuon}      & $\Sgn\circ\msign\circ\Sgn$                    & \no  & EMA \\
Aurora \citep{dewulf2026aurora}    & $(\msign\circ\RN)^{\circ 2}$                  & \no  & tall only; accumulated row scaling\\

\hline
\end{tabular}
\end{table}
}

We work with the stateless chains in \Cref{table:intro_lmochain} (\Cref{sec:chain})
and write $N=N_k\circ\cdots\circ N_1$ for the corresponding chain of normalization
maps; since each $N_i$ is odd and invariant to positive rescaling, the chained LMO
(\eqref{eq:chain}) is $T=-N$ up to scale.
As $-N(G)\in\argmin_{X\in\mathcal{K}}\langle G,X\rangle$ exactly when
$N(G)\in\argmax_{X\in-\mathcal{K}}\langle G,X\rangle$, chained LMO is an LMO over some set if and only if $N$ is a \emph{linear \textbf{maximization} oracle}
over some set $\mathcal{C}$, that is,
$N(G)\in\argmax_{X\in\mathcal{C}}\langle G,X\rangle$ for every $G$.
For each chain, we therefore ask whether $N$ is a linear maximization oracle over any set.

\subsection{Chained LMOs are not necessarily LMOs}
\label{sec:separation_lmo}

\Cref{sec:chain} shows that many Muon variants optimize with chained LMOs. 
We now make precise the claim from the introduction that these chains generally leave the LMO class, so that the standard LMO theory does not apply to them.
The claim has to be established chain by chain, since some compositions do stay inside the class: for instance, positive row scaling preserves entrywise signs, so
$\Sgn\circ\RN=\Sgn$ and the chain $-(-\Sgn)\circ(-\RN)=-\Sgn$, which is indeed the LMO of the entrywise-$\ell_\infty$ ball.
We consider three groups: the chains used in practice (\Cref{table:intro_lmochain}), finite alternations of $\RN$ and $\CN$, and their limit $\caE$. 

\begin{theorem}\label{thm:not_lmo}
Let \(m,n\ge5\), and let normalizations map zero rows and columns to zero.
There is no fixed set \(\mathcal{K}\subseteq\mathbb{R}^{m\times n}\) that satisfies: for every \(G\) in its domain
\[
T(G)\in\argmax_{X\in\mathcal{K}}\langle G,X\rangle
\]
for each map \(T\) listed below:
\begin{enumerate}[label=(\roman*), leftmargin=*, nosep]
\item \(\msign\circ\RN\), \(\msign\circ\RN\circ\CN\), \(\RN\circ\msign\), \(\CN\circ\RN\circ\msign\), and \(\RN\circ\CN\circ\msign\), for \(m>n\).
\item Alternation with the matrix sign: \((\msign\circ\RN)^{\circ k}\) and \((\RN\circ\msign)^{\circ k}\), for every \(k\ge1\) and \(m>n\).
\item Every finite composition of \(\RN\) and \(\CN\) that contains both, in particular \((\RN\circ\CN)^{\circ k}\); its limit \(\mathcal{E}\); and \(\msign\circ\mathcal{E}\), the last two on matrices with no zero entries.
\item Every finite composition of \(\msign\) and \(\Sgn\) that contains both.
\end{enumerate}
Hence the map $-T$ fails to be a LMO of any fixed set.
\end{theorem}

The statement concerns the exact matrix sign and the stateless chains, and it covers every chain in \Cref{table:intro_lmochain} up to transpose. Normalizing after $\msign$ can be vacuous: on tall inputs of full column rank, the columns of $\msign(G)$ are already orthonormal, so $\CN\circ\msign=\msign$, and symmetrically, on wide inputs of full row rank, $\RN\circ\msign=\msign$. The latter is why the restriction $m>n$ exists. 

The proof relies on properties that a linear maximization oracle over a fixed set $\mathcal{C}$ must satisfy.

\begin{lemma}\label{lem:idempotence}
Let $N(G)\in\argmax_{X\in\mathcal{C}}\langle G,X\rangle$ for every $G$, for a
fixed set $\mathcal{C}$, and let $Y=N(G)$ for some $G$.
If $\|N(Y)\|_F=\|Y\|_F$, then $N(Y)=Y$.
\end{lemma}

\begin{lemma}\label{lem:alignment}
Let $N(G)\in\argmax_{X\in\mathcal{C}}\langle G,X\rangle$ for every $G$, for a
fixed set $\mathcal{C}$. If $N(-G)=-N(G)$, then $\langle G,N(G)\rangle\ge0$.
\end{lemma}

For (i)--(ii) and the finite compositions in (iii), the normalizations keep undoing each other: on a suitable input, applying the chain a second time changes the output but not its norm, contradicting \Cref{lem:idempotence}.
For $\msign\circ\mathcal{E}$ and (iv), the first map is invariant to a family of transformations of its input: diagonal rescalings for $\mathcal{E}$, and changes of singular values or entry magnitudes for $\msign$ and $\Sgn$.
Transforming $G$ within this family leaves the output unchanged but can make it negatively aligned with $G$; since these chains are odd, this contradicts \Cref{lem:alignment}.
The limit $\mathcal{E}$ is idempotent, odd, and positively aligned, so both lemmas hold for it; we instead violate the LMO condition directly, by finding $G$ and $H$ with $\langle G,\mathcal{E}(G)\rangle<\langle G,\mathcal{E}(H)\rangle$. 


\subsection{Convergence behavior of chained LMOs}
\label{sec:convergence}

An LMO update always has a negative alignment with the gradient, which makes it a descent direction. When there exists a positive alignment, which is possible for chained LMOs, we can construct a counterexample where an optimization method diverges even for smooth functions (\Cref{prop:divergence}).

\begin{proposition}\label{prop:divergence}
Let $T$ be positively scale invariant with $\langle G, T(G)\rangle > 0$ for some $G$. For the convex, $1$-smooth quadratic $f(X) = \langle G, X\rangle^2 / (2\|G\|_F^2)$ with $X_0 = G/\|G\|_F^2$, the iterates $X_{t+1} = X_t + \eta_t T(\nabla f(X_t))$ satisfy $f(X_t) \to \infty$ for every step size sequence with $\sum_t \eta_t = \infty$.
\end{proposition}

The counterexamples of \Cref{thm:not_lmo} give such inputs for every composition of $\msign$ and $\Sgn$ that contains both, including the stateless chains of Muon$^2$ and AdaMuon, and for $\msign \circ\ \mathcal{E}$ in all sufficiently large shapes. The length-4 chain of Aurora, $(\msign\circ\RN)^{\circ 2}$, also admits such input for $n = 2$ and $m \ge 258$ (\Cref{appx:lem:aurora-negative}). None of these methods can therefore have a convergence guarantee even over smooth functions.

\textbf{Approximate LMOs.} For chains with negative alignment, we measure how far they are from an exact LMO. We formalize this notion using approximate LMOs.
\begin{definition}[Approximate LMO]\label{def:approx_lmo}
Let $\mathcal K$ be a compact set symmetric with respect to the origin, and define
$m_{\mathcal K}(G)=\min_{X\in\mathcal K}\langle G,X\rangle$.
For $\nu\in(0,1]$, a map $U$ is a \emph{$\nu$-approximate LMO} of
$\mathcal K$ if $U(G)\in\mathcal K$ and
$\langle G,U(G)\rangle\le \nu\, m_{\mathcal K}(G)$
for every $G$.
\end{definition}

The feasibility condition $U(G)\in\mathcal K$ prevents the approximation factor from being increased trivially by rescaling $U$. It also implies $\nu\le1$, where $\nu = 1$ implies exact LMO. Using the same machinery from the LMO framework, we can prove convergence of approximate-LMOs. 
\begin{proposition}\label{prop:convergence}
Let $f$ have an $L$-Lipschitz gradient, let $\Delta = f(X_0) - \inf f$, and let $U$ be a $\nu$-approximate LMO of $\mathcal{K}$.  Define $R := \max_{X\in\mathcal{K}} \|X\|_F$.
\begin{enumerate}[label=(\roman*), leftmargin=*, nosep]
    \item \emph{Gradient descent.} With $X_{t+1} = X_t + \eta\, U(\nabla f(X_t))$ and $\eta = \sqrt{2\Delta/(L R^2 T)}$,
    \[
        \frac{1}{T} \sum_{t=0}^{T-1} h_{\mathcal{K}}(\nabla f(X_t)) \le \frac{R}{\nu} \sqrt{\frac{2 L \Delta}{T}}.
    \]
    \item \emph{SGD with momentum.} With stochastic gradients of variance at most $\sigma^2$, $X_{t+1} = X_t + \eta\, U(M_t)$ with the momentum of \Cref{eq:chainedlmo_update}, $1 - \beta = T^{-1/2}$ and $\eta = c\, T^{-3/4}$, the same average in expectation is $O\big(\nu^{-1} [ \Delta/c + R (L R c + \sigma) ] T^{-1/4}\big)$.
\end{enumerate}
\end{proposition}

Compared with an exact LMO, both rates lose exactly a factor $1/\nu$. Two chains from \Cref{table:intro_lmochain}, $\msign \circ \RN$ and $\RN \circ \msign$, fall into the class of approximate-LMOs (\Cref{prop:approx}).
\begin{proposition}\label{prop:approx}
On all $m \times n$ matrices, $- \RN \circ \msign$ is a $1/\sqrt{m}$-approximate LMO of $\mathcal{K}_{\mathrm{row}} = \{X : \max_i \|X_{i:}\|_2 \le 1\}$, and $- \msign \circ \RN$ is a $1/\sqrt{m}$-approximate LMO of the spectral-norm ball. By transposition, the same holds for $\CN$ with $1/\sqrt{n}$.
\end{proposition}

This gives convergence guarantees for the stateless versions of NorMuon, Muon$+$ with a single normalization and MuonEq, with additional $\sqrt{m}$ appearing in the upper bound. The additional dimension factor cannot be removed within this analysis: any set for which a map is feasible must contain the convex hull of its image, and a larger set only makes the approximation factor smaller.

\begin{proposition}\label{prop:dimension}
Let $n = 2$. For every closed convex set $\mathcal{K}$ containing the image, the best approximation factor satisfies $\nu \le 3(m-1)^{-1/4}$ for $-\RN\circ\msign$ (for odd $m$) and $\nu \le 3/\sqrt{m-1}$ for $-\msign\circ\RN$.
\end{proposition}

Both bounds tend to zero, so any guarantee derived from \Cref{prop:convergence} degrades with dimension. A tighter analysis may need to go beyond uniform approximation, e.g.\ by bounding the approximation factor along the trajectory, which is out of scope for this work.

%% file: tex_arXiv/3.assocmem.tex
\Cref{sec:chainedLMO} shows that the standard worst-case guarantees of LMO-based optimization do not generally survive chaining: some chains can diverge even on smooth convex objectives, while others admit only dimension-dependent approximate-LMO guarantees. Any positive theory for chaining must therefore exploit additional problem structure beyond generic smoothness.


We instead demonstrate a benefit on the linear associative memory
(\Cref{sec:lam}), following recent theoretical work on Muon
\citep{wang2026muon, kim2026sharp, li2026muon}.
Muon helps there by whitening the anisotropy of the label frequencies, which suffices when the frequencies are the only source of anisotropy, as for the isotropic Gaussian embeddings of \citet{kim2026sharp}. With a second source, such as anisotropic embeddings, Muon alone may not suffice, and chains such as $\msign \circ \RN$ can provably help.
Anisotropic embeddings also make the model a closer abstraction of trained transformers, whose hidden representations are dominated by a few outlier coordinates \citep{timkey2021all,dettmers2022gpt3}.
We make this argument precise in the following sections.

We store $N$ facts with queries $y_i\in\mathbb R^n$ and answers $v_i=\Sigma x_i\in\mathbb R^d$, where $y_i\sim\mathcal N(0,I_n/n)$ and $x_i\sim\mathcal N(0,I_d/d)$ are independent, and $\Sigma=\diag(\sigma) \in \mathbb{R}^{d \times d}$ has positive entries. Fact $i$ occurs with probability $p_i=i^{-\alpha}/\sum_{r=1}^N r^{-\alpha}$ for a fixed $\alpha>1$. A memory $W\in\mathbb R^{d\times n}$ is trained with the cross-entropy loss and evaluated by its retrieval accuracy,
\[
\mathcal L(W)
=
\sum_{i=1}^N
p_i
\left(
-v_i^\top Wy_i
+
\log\sum_{j=1}^N e^{v_j^\top Wy_i}
\right),
\ \ 
\Acc(W)
=
\frac1N
\sum_{i=1}^N
\mathds{1}
\left\{
v_i^\top Wy_i
>
\max_{j\ne i}v_j^\top Wy_i
\right\}.
\]
Following prior theoretical analysis of linear associative memory \citep{kim2026sharp}, we study the retrieval accuracy after a single update from $W_0=0$. Let
$G=-\nabla\mathcal L(0)$ denote the negative gradient at initialization. We compare three methods, gradient descent, Muon, and row normalization followed by Muon:
\[
W_{\mathrm{GD}}=G,\quad
W_{\mathrm{Muon}}=\msign(G),\quad
W_{\mathrm{chain}}=\msign(\RN(G)).
\]
Step sizes are irrelevant because retrieval accuracy is invariant to positive rescaling of $W$. This chain is tailored to the anisotropy in our setting, and other forms of anisotropy may give provable advantages to other chains.


\textbf{Asymptotic regime.} Fix $\alpha$, an infinite positive sequence
$\sigma=(\sigma_j)_{j\ge1}$, and dimension sequences $N_n,d_n$.
For each $n$, the instance uses $N=N_n$ facts, query dimension $n$, answer dimension $d=d_n$,
and the scales $\sigma_{1:d_n}$.
For a method $\mathrm{M}\in\{\mathrm{GD},\mathrm{Muon},\mathrm{chain}\}$, write
\[
\mathrm{Acc}_n^{\mathrm{M}}
=
\mathrm{Acc}\bigl(
W_{\mathrm{M}};\,
N_n,n,d_n,\sigma_{1:d_n}
\bigr)
\]
for the random accuracy of the instance.
The quantity of interest is its limit in probability as $n\to\infty$.

\begin{assumption}[Scaling]
\label{ass:scaling}
Fix $0<\varepsilon<(\beta-1/2)/2$ and $\gamma>1/(1-\beta)$. Eventually, $n^{1+(1-\beta+\varepsilon)/\alpha}\le N_n\le n^{1+(1/2-\varepsilon)/\alpha}$ and $n^{1/(1-\beta)}\le d_n\le n^\gamma$.
\end{assumption}

\begin{assumption}[Anisotropy]
\label{ass:anisotropy}
For some fixed $1/2<\beta<1$ and constants $0<c_\sigma\le C_\sigma<\infty$, we have $c_\sigma j^{-\beta}\le \sigma_j\le C_\sigma j^{-\beta}$ for every $j\ge1$.
\end{assumption}

The scaling regime (\cref{ass:scaling}) is roughly $n\ll N\ll d$: there are more facts than query dimensions, so retrieval requires extracting a weak fact-specific signal, while the large answer dimension provides many coordinates over which that signal can accumulate. 

\Cref{ass:anisotropy} ensures that $\sum_j \sigma_j^2$ remains bounded while $\sum_{j\le d}\sigma_j$ diverges. Thus row normalization, which removes the answer-coordinate scales before the matrix sign, allows the chain to accumulate signal across coordinates in a regime where GD and Muon fail.

The core result is that the asymptotic one-step accuracy of chained Muon tends to
1, whereas those of Muon and GD go to 0.
\begin{theorem}[One-step separation]
\label{thm:lam-separation}
Under \Cref{ass:scaling,ass:anisotropy}, as \(n\to\infty\),
\[
\left(
\Acc_n^{\mathrm{chain}},
\Acc_n^{\mathrm{Muon}},
\Acc_n^{\mathrm{GD}}
\right)
\xrightarrow{\mathbb P}
(1,0,0).
\]
\end{theorem}

The separation is driven by how the answer-coordinate scales enter the three
updates. For GD and Muon, the relevant normalized score quantities depend on
quadratic and quartic powers of \(\sigma_a\); because
\(\sum_a \sigma_a^2<\infty\) and \(\sum_a \sigma_a^4<\infty\), their normalized
score vectors have only \(O(n)\) total energy, while the correct-answer signal
remains too weak to overcome the maximum of the \(N\gg n\) incorrect scores.
By contrast, row normalization removes the answer-coordinate scales before the
matrix-sign step, since
\(\RN(\Sigma XK)=\RN(XK)\). The useful contributions can then add coherently in
proportion to \(\sum_a \sigma_a\), while the residual fluctuations are controlled
by \(\sqrt{\sum_a \sigma_a^2}\). Since
\(\sum_{a\le d}\sigma_a\to\infty\) while the quadratic and quartic sums remain
bounded, the chained update can accumulate signal across answer coordinates in
a regime where GD and Muon cannot, yielding the separation above. The full
proof is deferred to \Cref{app:lam_separation}.

%% file: tex_arXiv/4.TC.tex
\Cref{sec:lam_separation} suggests that equalizing the update across coordinates is what makes chaining useful. We push this further by exploring three directions in the space of chained LMOs. First, we introduce \textbf{TensorChain}, which adds a new normalization axis across layers. Second, we explore \textbf{Aurora-LN}, a simple long-side generalization of Aurora that applies the same chain to both tall and wide matrices. Third, motivated by recent work replacing Newton--Schulz with simpler axis-wise normalization \citep{glentis2025minimalist,deng2026rmnp}, we study whether the spectral step itself can be replaced by iterated row--column normalization, which we refer to as \textbf{SinkhornChain}. 

\begin{figure}[htbp]
\begin{minipage}[t]{0.48\linewidth}
\begin{algorithm}[H]
\caption{TensorChain}
\label{alg:tensor_chain}
\begin{algorithmic}[1]
\Require Momenta $M_{1:g}\in(\mathbb{R}^{m\times n})^g$
\State $T \gets \operatorname{Stack}(M_1,\ldots,M_g)
\in\mathbb{R}^{g\times m\times n}$
\State $A \gets
\AN_{\mathrm{layer}}
\circ\AN_{\mathrm{col}}
\circ\AN_{\mathrm{row}}$
\State $T \gets A^{\circ k}(T)$
\State $T \gets \NS_{\mathrm{slice}}(T)$
\State $T \gets A^{\circ k}(T)$
\State \Return $\operatorname{Unstack}(T)$
\end{algorithmic}
\end{algorithm}
\end{minipage}\hfill
\begin{minipage}[t]{0.48\linewidth}
\begin{algorithm}[H]
\caption{Aurora-LN}
\label{alg:aurora_ln}
\begin{algorithmic}[1]
\Require momentum $M\in\mathbb{R}^{m\times n}$
\If{$m\ge n$}
    \State $U \gets
    (\NS\circ\RN\circ\NS\circ\RN)(M)$
\Else
    \State $U \gets
    \bigl[(\NS\circ\RN\circ\NS\circ\RN)(M^\top)\bigr]^\top$
\EndIf
\State \Return $U$
\end{algorithmic}
\end{algorithm}
\end{minipage}
\end{figure}

\subsection{TensorChain: Normalizing Across Layers}
\label{sec:tensorchain}

Existing chains normalize only within individual weight matrices, through row, column, or spectral normalization, and treat each layer independently. TensorChain instead jointly normalizes parameters of the same shape across layers (\Cref{alg:tensor_chain}).

Parameters with the same shape and role, for example, the query projections \(W_Q\) of transformer blocks, are partitioned into groups. For a group of size \(g\), we stack their momenta into a tensor \(M\in\mathbb R^{g\times m\times n},\)
whose first axis indexes layers. Normalizing fibers along this axis prevents individual layers from dominating the corresponding coordinates of the group's update.

Let
\(\mathrm{AN}_{\mathrm{layer}},\mathrm{AN}_{\mathrm{row}},\mathrm{AN}_{\mathrm{col}}\)
denote \(\ell_2\)-normalization along the corresponding tensor axes, define
\[
\mathcal A
:=
\mathrm{AN}_{\mathrm{layer}}
\circ
\mathrm{AN}_{\mathrm{col}}
\circ
\mathrm{AN}_{\mathrm{row}},
\]
and let \(\mathrm{NS}_{\mathrm{slice}}\) apply Newton--Schulz independently to each matrix slice. TensorChain maps
\[
\mathrm{TensorChain}(M)
=
\mathcal A^{\circ k}
\circ
\mathrm{NS}_{\mathrm{slice}}
\circ
\mathcal A^{\circ k}(M),
\]
whose slices give the per-layer updates.

TensorChain fits directly into the chained-LMO framework of
\Cref{sec:chainedLMO} when we assume exact $\msign$ instead of Newton--Schulz. We will denote this as idealized TensorChain. Each fiber normalization is an exact LMO over a product of \(\ell_2\) balls, while slice-wise matrix sign is an exact LMO over the corresponding product of spectral-norm balls.

\begin{proposition}[Idealized TensorChain as chained LMO]
\label{prop:tensorchain}
For \(k\ge1\), Idealized TensorChain is a composition of exact LMOs over product norm balls, but the resulting map is not itself an LMO of any fixed constraint set
when \(m,n\ge5\). Moreover, Idealized TensorChain map admits a
negative-alignment input and can diverge on a convex smooth
objective under non-summable step sizes.
\end{proposition}

A closely related approach is Teon~\citep{zhang2026teon}, which also stacks matrices across layers but applies Newton--Schulz to a 2d matricization of the tensor. TensorChain instead applies Newton--Schulz independently to each matrix slice, with axis-wise normalization before and after. Empirical comparison with the Teon optimizer can be found in \Cref{app:teon}.

We can consider the overhead of TensorChain in two axes: additional flops due to additional normalization, and additional communication cost to construct the tensor. Fiber normalizations cost \(O(gmn)\) per group, negligible compared with the
\(O(gmn\min(m,n))\) Newton--Schulz steps, so TensorChain matches Muon's optimizer FLOPs up to lower-order terms. Under FSDP, layer-axis normalization couples matrices within a group and requires additional communication. Empirical performance of TensorChain can be found in \cref{sec:results}.

\subsection{Aurora-LN: Normalizing Along the Long Side}
\label{sec:aurora_ln}

Aurora \citep{dewulf2026aurora} includes several design choices beyond the normalization chain itself: the Aurora update is applied only to tall matrices, while wide matrices fall back to Muon, and a $\beta$ parameter accumulates the diagonal row preconditioner across steps. From the perspective of chaining, we consider a simplified variant that removes this accumulation and applies the same chain along the long side of every matrix. For tall and square matrices we apply the chain directly, while for wide matrices we apply it to the transpose and transpose the resulting update back. We refer to this variant as \emph{Aurora-LN} (\Cref{alg:aurora_ln}, LN stands for Longside Normalization).

Long-side or wide-matrix variants of Aurora-style normalization have previously appeared in modded-nanoGPT experiments \citep{modded_nanogpt_2024}, so we do not present Aurora-LN as a novel optimizer or algorithmic primitive. Rather, we revisit this design choice as a controlled simplification of Aurora. Our experiments in \Cref{sec:results} show that this simplified variant can be more effective than the default Aurora implementation.

\subsection{Replacing Newton--Schulz with Sinkhorn Normalization}
\label{sec:sinkhorn}

A complementary direction is to ask whether chaining can eliminate the spectral step altogether. Motivated by recent work on inexpensive normalizations replacing Newton--Schulz \citep{deng2026rmnp,glentis2025minimalist}, we consider \emph{SinkhornChain}, which alternates row and column rescaling as in Sinkhorn balancing, but uses $\ell_2$ rather than $\ell_1$ normalization. For an $m\times n$ matrix with $m>n$, each row--column normalization sweep costs $O(mn)$ FLOPs, whereas each
Newton--Schulz step costs $O(mn^2)$. Thus, for a fixed number
of iterations, SinkhornChain has substantially lower arithmetic
complexity than Newton--Schulz. 

However, none of our tested normalization variants or layerwise replacement configurations outperformed Muon; the strongest replaces NS only on $Q,K$, reaching validation loss $2.5945$ versus Muon's $2.5905$. Nevertheless, we still report these results, as in large-scale training the computational cost of Muon may be problematic and Sinkhorn shows a possibility of replacing it. Full details and ablations are given in \Cref{app:ns_replacement}.


%% file: tex_arXiv/5.experiments.tex
\subsection{Setup}
\label{sec:setup}
\textbf{Models and data.}
We train Qwen3 models \citep{yang2025qwen3} of sizes 0.6B and 1.7B on FineWeb \citep{penedo2024fineweb} with sequence length 4096, using pre-norm RMSNorm \citep{zhang2019root}, QK-norm, untied embeddings, and the default \texttt{torchtitan} initialization \citep{liang2025torchtitan}. Training uses a 30B-token corpus, with experiments exceeding this budget spanning multiple epochs. To assess the effect of data repetition, we additionally conduct experiments on a 112B-token corpus without repeated training tokens, obtaining similar results and consistent relative performance across best performing optimizers (see \Cref{fig:raw-cos-17b-4m-fixed}). We report loss on a held-out FineWeb split of 26M tokens.

\textbf{Training.}
We evaluate token batch sizes of 1M and 4M. The 0.6B model uses 8-way FSDP, while the 1.7B model uses 4-way data parallelism and 8-way FSDP. Learning rate schedules of both cosine decay with 10\% warmup and WSD results (10\% warmup and 10\% decay) were evaluated. Matrix parameters use the optimizer under comparison with Nesterov momentum, with each update $U$ rescaled by $0.2\sqrt{mn}/\|U\|_F$ \citep{liu2025muon}. Embeddings, output, and normalization parameters use AdamW with learning rate $8\times10^{-3}$, and all parameters use decoupled weight decay $0.1$ \citep{loshchilov2017decoupled}. We use a fixed random seed.

\textbf{Baselines and tuning.}
We compare the methods in \Cref{table:intro_lmochain} together with the two explorations in \Cref{sec:nav}, TensorChain and Aurora-LN. We use the official implementation where available and our own otherwise, with the same five-step \texttt{bfloat16} Newton--Schulz for all methods. For TensorChain, we grouped all the weights that have the same role as a same group for simplicity, and used $k=5$. On 0.6B at $1\times$ Chinchilla (20 tokens per non-embedding parameter; \citealt{hoffmann2022training}), we sweep learning rates over $\{1,2,3,5,8\}\times10^{-3}$ for WSD and $\{5,8,12,16,20\}\times10^{-3}$ for cosine decay.

Across these sweeps, we observe that most baselines attain their best validation loss at the same learning rate within each schedule: $3\times10^{-3}$ under WSD and $12\times10^{-3}$ under cosine decay. This motivates a matched-learning-rate comparison in which the selected baselines are evaluated at a shared learning rate. AdaMuon is the only clear exception. Under WSD, it shares the same optimum and shows no clear advantage. However, under cosine decay, its optimum shifts to a substantially larger learning rate and its performance improves markedly, though it remains worse than TensorChain. We therefore exclude AdaMuon from the matched-learning-rate comparison, report its full sweep separately in \Cref{app:rawresult}, and leave a systematic study to future work.

Among the remaining baselines, we select MuonEq, Aurora-LN, and Muon$^2$ as the best-performing representatives using one Newton--Schulz, two Newton--Schulz, and additional momentum-dependent state, respectively, while retaining Muon and Aurora as references. For $4\times$ Chinchilla, we rescale it by $4^{-0.32}\approx0.64$ from the learning rate tuned at 1x Chinchilla \citep{bjorck2025scaling}.

\subsection{Results}
\label{sec:results}

\begin{figure}[t]
\centering

\includegraphics[width=0.75\linewidth]
{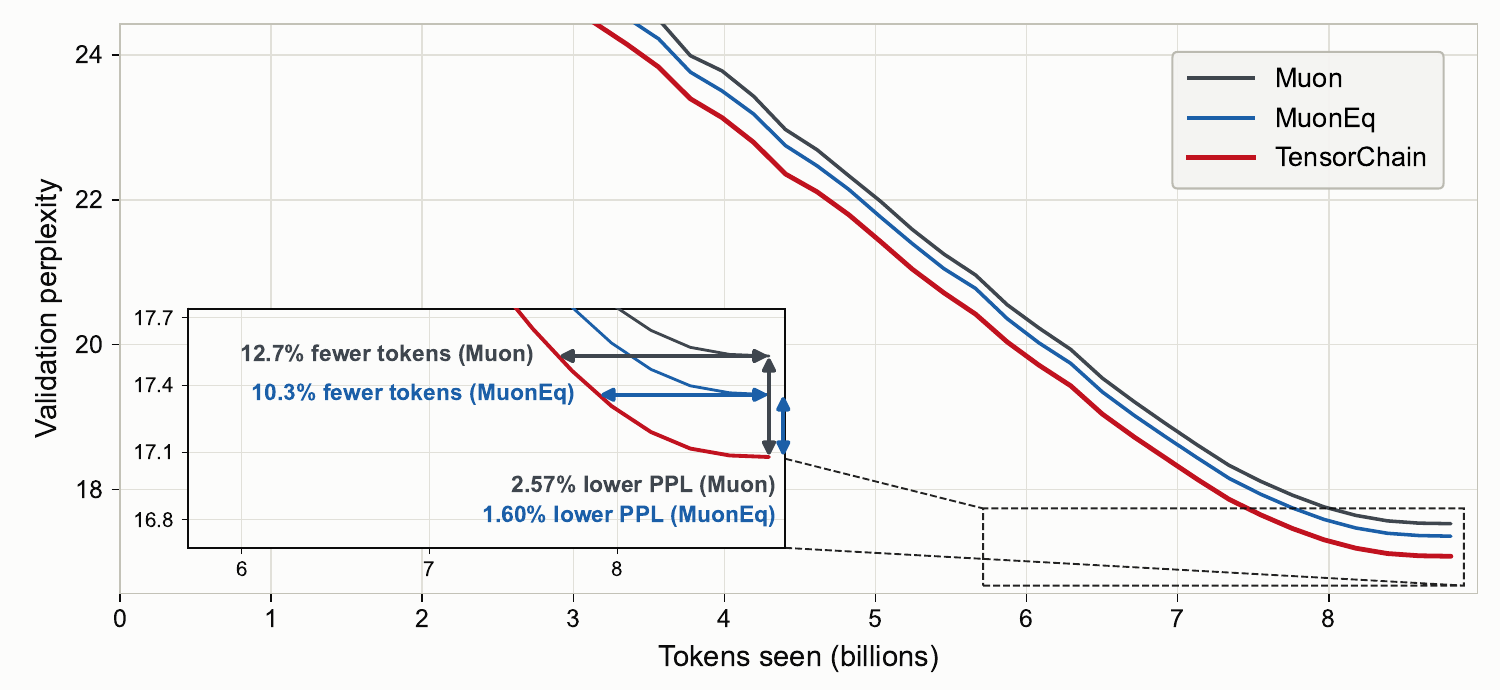}

\vspace{0.4em}

\includegraphics[width=0.75\linewidth]
{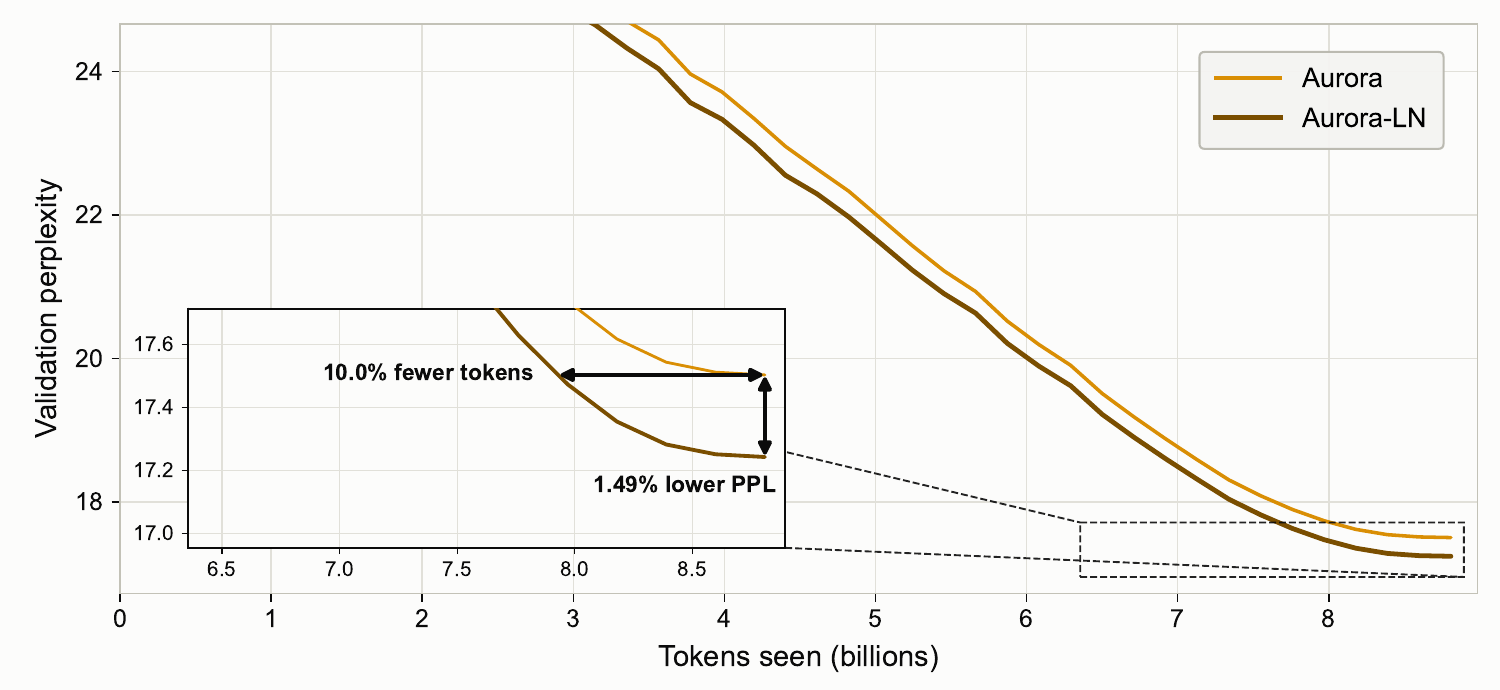}

\caption{
\textbf{Top:} Comparison of Muon, MuonEq and TensorChain;
\textbf{Bottom:} Comparison of Aurora and Aurora-LN, both setups are training on Qwen3-0.6B with a
1M-token batch at $1\times$ Chinchilla under cosine LR. For other setups see \cref{app:token_efficiency}.
}
\label{fig:results}
\end{figure}

\begin{table}[htbp!]
\centering
\caption{
Final validation loss under WSD and cosine schedules.
Bold marks the best method in each row. The two methods on the right (Aurora, Aurora-LN) use Newton--Schulz twice, whereas other methods apply Newton--Schulz once. Muon$^2$ has an additional momentum state. (*) were introduced in \Cref{sec:nav}.
}
\label{tab:main-final}
\small
\setlength{\tabcolsep}{4pt}
\begin{tabular}{l l l l || c c c c | c c}
\toprule
Schedule & Model & Batch & Tokens &
Muon & MuonEq & TensorChain$^*$ & Muon$^2$ &
Aurora & Aurora-LN$^*$ \\
\midrule
WSD
& 0.6B & 1M & 8.8B
& 2.9201 & 2.9151 & \textbf{2.9000} & 2.9104
& 2.9184 & 2.9116 \\
WSD
& 1.7B & 4M & 112B
& 2.5310 & 2.5279 & \textbf{2.5236} & 2.5264
& 2.5301 & 2.5242 \\
\midrule
Cosine
& 0.6B & 1M & 8.8B
& 2.8639 & 2.8540 & \textbf{2.8379} & 2.8497
& 2.8624 & 2.8474 \\
Cosine
& 0.6B & 4M & 35B
& 2.7527 & 2.7434 & \textbf{2.7344} & 2.7443
& 2.7487 & 2.7391 \\
Cosine
& 1.7B & 1M & 28.1B
& 2.5974 & 2.5924 & \textbf{2.5857} & 2.5885
& 2.5964 & 2.5882 \\
Cosine
& 1.7B & 4M & 112B
& 2.4823 & 2.4784 & 2.4747 & 2.4744
& 2.4810 & \textbf{2.4734} \\
\bottomrule
\end{tabular}
\end{table}

\textbf{Main comparison.}
TensorChain shows a consistent advantage over the more directly comparable chained baselines. Under WSD, it achieves the best final validation loss in both reported settings; under cosine decay, it is best in three of the four settings and outperforms Muon and MuonEq throughout. The only cases in which TensorChain is surpassed are by methods that use either an additional Newton--Schulz application or additional momentum-dependent state. Thus, among variants with comparable computational and state complexity, TensorChain provides a consistent improvement across model scales, batch sizes, and learning-rate schedules.

The original Aurora performs only slightly better than Muon in our experiments, whereas the simplified Aurora-LN variant consistently improves over Aurora and becomes the best-performing method in the largest cosine setting. This suggests that seemingly minor choices in how a chain is applied across matrix shapes can have a substantial effect on performance. \Cref{fig:results} visualizes these two comparisons: TensorChain against Muon and MuonEq, and Aurora-LN against the original Aurora.


\begin{table}[htbp!]
\centering
\caption{
Token-efficiency gain of TensorChain: percentage$(\%)$ reduction in tokens needed to reach a common validation-loss target, estimated by linear interpolation between adjacent checkpoints. Positive values favor TensorChain.
}
\label{tab:token-efficiency}
\small
\setlength{\tabcolsep}{4pt}
\begin{tabular}{lll||ccccc}
\toprule
Model & Batch & Tokens & Aurora-LN & Muon$^2$ & MuonEq & Aurora & Muon \\
\midrule
0.6B & 1M & 8.8B  & 8.2  & 9.0  & 10.3 & 12.4 & 12.8 \\
0.6B & 4M & 35B   & 5.4  & 7.6  & 7.4  & 8.9  & 10.0 \\
1.7B & 1M & 28.1B & 3.9  & 4.0  & 6.2  & 7.5  & 7.8 \\
1.7B & 4M & 112B  & -1.9 & -0.5 & 6.7  & 7.4  & 7.8 \\
\bottomrule
\end{tabular}
\end{table}
The token-efficiency comparison in \Cref{tab:token-efficiency} makes the same trend more explicit. We report the percentage reduction in training tokens required to reach the same validation loss, estimated by linear interpolation between adjacent checkpoints; within each fixed-batch setting,
this is equivalently the reduction in training steps.
Across the four settings, TensorChain requires on average $9.6\%$ fewer tokens than Muon and $7.7\%$ fewer than MuonEq.

\subsection{TensorChain Ablations}

\textbf{Grouping.} We compare consecutive, stride, greedy load-balanced, and all- compatible-layer
grouping (\Cref{tab:tc-grouping_maintext}); implementation details are deferred to \Cref{app:grouping-results}. Consecutive grouping is consistently worse with validation loss gap of $\approx 0.007$ for WSD and $\approx 0.01$ for cosine, while stride, greedy, and all performed similarly compared to consecutive grouping. This suggests that mixing matrices across
distant layers matters.

\begin{table}[htbp]
\centering
\small
\caption{TensorChain grouping ablation on Qwen3-0.6B, 1M-token batch, $1\times$ Chinchilla. Each entry reports the best validation loss from the corresponding full learning-rate sweep.}
\label{tab:tc-grouping_maintext}
\begin{tabular}{l||cccc}
\toprule
Schedule & Consecutive & Stride & Greedy & All \\
\midrule
WSD    & 2.9070 & \textbf{2.8990} & 2.8994 & 2.9000 \\
Cosine & 2.8474 & \textbf{2.8357} & 2.8370 & 2.8379 \\
\bottomrule
\end{tabular}
\end{table}
\textbf{Cross-layer normalization and iteration count.} We conduct ablation studies to investigate two aspects of TensorChain: the contribution of cross-layer normalization and the effect of repeated normalization. All experiments use Qwen3-0.6B with a 1M-token batch size and a (1$\times$) Chinchilla training budget, under both WSD and cosine learning-rate schedules.

To isolate the effect of cross-layer normalization, we compare TensorChain against Sink--NS--Sink, which applies iterative row and column normalization before and after Newton--Schulz orthogonalization without stacking parameters across layers. As shown in Table \ref{tab:normalization-ablation}, TensorChain consistently achieves lower validation loss across all settings, with improvements of 0.0208–0.0220 under WSD and 0.0248–0.0250 under cosine. These results highlight the contribution of cross-layer normalization beyond within-matrix normalization alone.

We next examine the effect of the number of normalization iterations \(k\), varying it from 1 to 5 for both methods. Increasing \(k\) does not consistently improve performance: validation loss remains nearly unchanged across iteration counts, with (\(k=1\)) achieving the lowest or tied-lowest loss in every setting. This suggests that the benefits of TensorChain arise primarily from cross-layer normalization rather than convergence of the normalization procedure, making smaller \(k\) an attractive choice when computational efficiency is important.

\begin{table}[htbp]
\centering
\small
\caption{Ablation of cross-layer normalization and iteration count
$k$ on Qwen3-0.6B (1M-token batch, $1\times$ Chinchilla).
Each entry reports final validation loss.}
\label{tab:normalization-ablation}
\begin{tabular}{ll||ccccc}
\toprule
Schedule & Method & $k=1$ & $k=2$ & $k=3$ & $k=4$ & $k=5$ \\
\midrule
WSD & TensorChain
& \textbf{2.8992} & 2.9003 & 2.9005 & 2.9002 & 2.9000 \\
& Sink--NS--Sink
& 2.9210 & 2.9218 & 2.9214 & 2.9216 & 2.9213 \\
\midrule
Cosine & TensorChain
& \textbf{2.8368} & 2.8381 & 2.8386 & 2.8382 & 2.8379 \\
& Sink--NS--Sink
& 2.8624 & 2.8630 & 2.8626 & 2.8624 & 2.8627 \\
\bottomrule
\end{tabular}
\end{table}

%% file: tex_arXiv/6.conclusion.tex
We introduced the chained-LMO framework as a unified view of recent Muon variants. Although chaining generally falls outside the standard LMO framework and therefore may lose its usual smooth optimization guarantee, it can still provide substantial practical benefits. To close this gap, we analyzed linear associative memory with multiple sources of anisotropy and proved that chaining can yield a strict learning advantage. Empirically, we introduced TensorChain, which adds normalization across the layer axis, and showed that it can outperform existing chained variants in language-model pretraining.

An interesting direction for future work is to better understand how TensorChain's gains vary across model scales and training settings, including the particularly strong improvements observed at 0.6B. On the theory side, extending the associative-memory analysis to more general chains and forms of anisotropy would clarify when chaining should help. Finally, TensorChain naturally extends to additional tensor axes arising in multi-head attention and mixture-of-experts models, where understanding the potential gains is an interesting direction for future work.

%% file: tex_arXiv/appendix.tex
\section{Appendix: Related work}
\label[appendix]{sec:appendix_relwork}
\subsection{Frameworks around Muon and variants}
\label[appendix]{sec:appendix_framework}
\citet{bernstein2024old} take a descent step that minimizes the local quadratic model
\[
\mathrm{argmin}_{\Delta W}\langle \nabla f(W), \Delta W\rangle + \tfrac{\lambda}{2}\|\Delta W\|^2
\]
with the norm $\|\cdot\|$ fixed in advance. The solution splits into a step size and a direction, $-(\|\nabla f(W)\|^\dagger/\lambda) \cdot \arg\max_{\|X\|=1} \langle \nabla f(W), X\rangle$. Taking $\|\cdot\|$ to be the induced operator norm $\|W\|_{\alpha\to\beta} = \max_x \|Wx\|_\beta/\|x\|_\alpha$ gives a family of optimizers including popular adaptive algorithms. Switching off their EMAs, Adam \citep{kingma2014adam}, Shampoo \citep{gupta2018shampoo} and Prodigy  \citep{mishchenko2023prodigy} are each steepest descent under a particular member: $(\alpha,\beta) = (\ell_1, \ell_\infty)$ for Adam and Prodigy, which differ only in their step-size rule, and $(\ell_2, \ell_2)$ for Shampoo, which is a positive scaling of the Muon update when the EMA is switched off. \citet{pethick2025training} write the same direction as a \emph{linear minimization oracle}
\[
\mathrm{lmo}_{\mathcal{B}}(\nabla f(W)) =
\mathrm{argmin}_{X \in \mathcal{B}} \langle \nabla f(W), X \rangle
\]
over a norm ball, and interpret existing normalized gradient methods as optimizing with LMOs. By understanding the induced norm as how the output is amplified in a linear layer under bounded input, \citet{pethick2025training} propose the Scion optimizer, whose design principle is to bound the RMS norm of intermediate layers in neural network training. \citet{kovalev2025understanding} interprets the constraint set $\mathcal{B}$ as a trust region, and both works prove convergence guarantees for their respective classes of algorithms.

These frameworks cast optimizer choice as an appropriate choice of norm or constraint set, and subsequent work varies that choice along several axes. \citet{lau2025polargrad} write vector preconditioning and matrix-aware updates as the same subproblem under different norms: the time-dependent quadratic metric gives the variable-metric methods that date back to \citet{fletcher1970new} and are made data-driven in AdaGrad \citep{duchi2011adaptive}, while the spectral norm gives $\|G\|_{\mathrm{nuc}}\,\msign(G)$, the solution of \citet{bernstein2024old}. \citet{crawshaw2025exploration} aggregate per-layer norms into a product norm, whose oracle decomposes into scaled per-layer oracles; this expresses the standard practice of running Muon and Adam on separate parameter classes. Others vary the geometry itself: the Schatten-$p$ family interpolates between the Frobenius and spectral cases \citep{ma2026schattor, pethick2026use}, and preconditioned matrix norms widen the family further \citep{veprikov2025preconditioned}.

Certain frameworks deviate from the norm-choice template. Lion-$\mathcal{K}$ \citep{chen2024lion} replaces the sign map by the subgradient of a convex $\mathcal{K}$, and shows that the iterates converge to KKT points of $\min_X F(X)$ subject to $\mathcal{K}^*(X) \le 1/\lambda$ for decoupled weight decay $\lambda$. Muon is the special case $\mathcal{K} = \|\cdot\|_{\mathrm{tr}}$ \citep{chen2025muon}. \citet{shumaylov2026muon} instead exhibit updates of the form $(GG^\top)^{-c}G$ whose best-performing exponents fall in the quasi-norm regime $c > 1/2$, where the update is provably not the LMO of any unitarily invariant norm ball; they establish convergence directly, through alignment and a descent potential rather than through the oracle.

These frameworks share one structural feature: a single geometry determines the update map, and the same geometry supplies the analysis.
An LMO over a norm ball is aligned with the negative gradient, $\langle G, \mathrm{lmo}_{\mathcal{B}}(G)\rangle = -\rho\|G\|_*$, which yields a descent lemma under the matching norm-smoothness assumption; Lion-$\mathcal{K}$ replaces this with a Lyapunov function built from $\mathcal{K}$.
Chained LMOs do not in general inherit either structure, because the LMO class is not closed under composition.
Even when every constituent map is an LMO, many practical chains are not the LMO of any fixed constraint set (\Cref{sec:chainedLMO}), nor in general the subgradient of a convex $\mathcal{K}$, so neither the oracle-based nor the Lyapunov-based argument transfers.
We therefore analyze chains directly, in the spirit of \citet{shumaylov2026muon}.
For positively aligned chains, an $\alpha$-approximate LMO framework recovers convergence, at the cost of an unavoidable dimension dependence.
Chains that admit negative-alignment directions, in contrast, can diverge even on smooth convex objectives (\Cref{sec:convergence}).

Closest to our setting is \citet{lau2026symmetry}, whose contribution is a design principle: update maps derived from the symmetry group of each weight block, with NorMuon, Muon$+$, MuonEq and Aurora \citep{li2025normuon, zhang2026muon+, chang2026muoneq, dewulf2026aurora} placed within that perspective. We instead study the effect of chaining itself: how composed updates relate to existing frameworks, whether they converge, and how they can improve linear associative memory learning. Their construction composes row-norm and one-sided spectral maps only, so the elementwise chains of \citet{si2025adamuon} and \citet{liu2026muon} fall outside it; our framework places no such restrictions.

\subsection{Existing Muon variants as chained LMOs}
\label[appendix]{sec:chain}

Recent Muon variants add normalization at different stages of the orthogonalization pipeline.
Post-orthogonalization methods correct the row- or column-wise imbalance left after the matrix-sign step.
NorMuon \citep{li2025normuon} does so with a row-wise second-moment state, while Muon$+$ \citep{zhang2026muon+} applies stateless row, column, column--row, or row--column normalization directly after $\msign$.
Pre-orthogonalization methods instead precondition the momentum before Newton--Schulz to improve the spectrum entering the orthogonalization step: MuonEq \citep{chang2026muoneq} uses row/column equilibration, and Muon$^2$ \citep{liu2026muon} uses Adam-style second-moment preconditioning.
AdaMuon \citep{si2025adamuon} acts on both sides, applying an elementwise sign before orthogonalization and a second-moment normalization afterward, which combines coordinate-wise adaptivity with orthogonal updates.
Aurora \citep{dewulf2026aurora} alternates row normalization and the matrix sign on tall matrices through the chain $(\msign\circ\RN)^{\circ 2}$, so that the update is approximately balanced in both its spectrum and its row norms.

Some of these optimizers are stateful because of exponential moving averages.
\citet{bernstein2024old} recover norm-based steepest-descent methods from Adam \citep{kingma2014adam}, Shampoo \citep{gupta2018shampoo}, and Prodigy \citep{mishchenko2023prodigy} by switching off their EMAs, and we use the same reduction to expose the underlying normalization maps.
For example, Adam with $\beta_1=\beta_2=0$ reduces to $\Sgn$; similarly, NorMuon reduces to $\RN\circ\msign$, Muon$^2$ to $\msign\circ\Sgn$, and AdaMuon to $\Sgn\circ\msign\circ\Sgn$.
We use this reduction only to isolate the normalization chain, not to identify the full stateful optimizer with its stateless counterpart.

\subsection{Linear Associative Memory and Muon}
\label[appendix]{sec:lam}

An associative memory stores key--value pairs $\{(k_i, v_i)\}_{i=1}^N$, with $k_i \in \mathbb{R}^{d_k}$ and $v_i \in \mathbb{R}^{d_v}$, in a matrix $W \in \mathbb{R}^{d_v \times d_k}$ so that presenting $k_i$ recalls $v_i$; for orthonormal keys, $W = \sum_i v_i k_i^\top$ does so exactly. Attention can be read this way \citep{zhong2025understanding}: without the softmax normalization it computes $o_t = \sum_{i \le t} \exp(q_t^\top k_i / \sqrt{d_k})\, v_i = S_t \phi(q_t)$ with $S_t = \sum_{i \le t} v_i \phi(k_i)^\top$, where $\phi$ is the infinite-dimensional feature map of the exponential kernel, and FFN layers are similar through ReLU kernels. From an interpretability perspective, factual associations in a transformer can be edited by a rank-one update to an MLP projection \citep{meng2022locating, meng2022mass}.

A growing line of work relates the success of Muon to associative memory. \citet{wang2026muon} find empirically that Muon's advantage over Adam comes mainly from the value/output and FFN weights, precisely the parameters that act as associative memories. To study this theoretically, \citet{wang2026muon, kim2026sharp, li2026muon} turn to the linear associative memory, where $W$ retrieves pair $i$ correctly when $v_i^\top W k_i > v_j^\top W k_i$ for all $j \neq i$, and is trained with the cross-entropy loss $\mathcal{L}$ on these logits, with pair $i$ occurring with frequency $p_i$. When the frequencies are imbalanced, Muon provably retrieves better than gradient descent, because it \emph{whitens} the frequencies. At initialization, $-\nabla \mathcal{L}(0) \approx \sum_i p_i\, v_i k_i^\top$, so gradient descent learns each pair at a rate proportional to $p_i$. With orthonormal embeddings, this sum is a singular value decomposition with singular values $p_i$, and $\msign\big(\sum_i p_i\, v_i k_i^\top\big) = \sum_i v_i k_i^\top$ gives every pair the same weight. \citet{wang2026muon, li2026muon} analyze orthonormal embeddings, and \citet{kim2026sharp} extend the result to Gaussian embeddings.

\newpage
\section{Appendix: Proofs in \cref{sec:chainedLMO}}
\label[appendix]{app:sec3}

\paragraph{Sign convention.}
Throughout the paper, LMO means \emph{linear minimization oracle}:
\[
\operatorname{lmo}_K(G)\in\arg\min_{X\in K}\langle G,X\rangle.
\]
We write \(N=N_k\circ\cdots\circ N_1\) for a chain of normalization maps
and \(T=-N\) for the signed update used in minimization. For the norm balls
and odd oracle selections considered here, \(\operatorname{lmo}_{\mathcal D_j}=-N_j\), so
\[
T=(-1)^{k+1}
\bigl(\operatorname{lmo}_{\mathcal D_k}\circ\cdots\circ
\operatorname{lmo}_{\mathcal D_1}\bigr)=-N,
\qquad
X_{t+1}=X_t+\eta_tT(M_t)=X_t-\eta_tN(M_t).
\]
For any fixed set \(C\), with \(-C:=\{-X:X\in C\}\),
\[
N(G)\in\arg\max_{Z\in C}\langle G,Z\rangle
\quad\Longleftrightarrow\quad
-N(G)\in\arg\min_{X\in-C}\langle G,X\rangle.
\]
Thus the non-LMO proofs below test whether \(N\) is a \emph{linear
maximization oracle}; their conclusions concern the signed update
\(T=-N\) as a linear minimization oracle. The reflected set is the same set
only when \(C=-C\), as for the norm balls used here.

We write \(\operatorname{msign}(G)=UV^\top\) for the canonical polar factor
of \(G=U\Sigma V^\top\), with zero singular directions mapped to zero,
and \(\operatorname{Sgn}\) for the entrywise sign with
\(\operatorname{Sgn}(0)=0\). Row and column normalization are denoted by
\(\RN\) and \(\CN\), respectively. The proofs use two necessary conditions
for a linear maximization oracle.

\subsection{Non-LMO results}
\label{app:nonlmo}

\begin{appxlemma}[\Cref{lem:idempotence} of the paper]
\label{appx:lem:lmo-idempotence}
Let \(N(G)\in\arg\max_{X\in K}\langle G,X\rangle\) for every \(G\),
for a fixed set \(K\), and let \(Y=N(G)\). If
\(\|N(Y)\|_F=\|Y\|_F\), then \(N(Y)=Y\).
\end{appxlemma}

\begin{proof}
Since every oracle output belongs to \(K\), \(Y\in K\). Optimality at the
input \(Y\) gives
\[
\langle Y,N(Y)\rangle\ge\langle Y,Y\rangle=\|Y\|_F^2.
\]
By Cauchy--Schwarz and the assumed equality of norms,
\[
\langle Y,N(Y)\rangle\le\|Y\|_F\|N(Y)\|_F=\|Y\|_F^2.
\]
Hence equality holds in Cauchy--Schwarz. The equal norms imply
\(N(Y)=Y\), including the case \(Y=0\).
\end{proof}

\begin{appxlemma}[\Cref{lem:alignment} of the paper]
\label{appx:lem:lmo-odd}
If \(N(G)\in\arg\max_{X\in K}\langle G,X\rangle\) for every \(G\), for a fixed set \(K\), and \(N(-G)=-N(G)\), then
\[
\langle G,N(G)\rangle\ge0.
\]
\end{appxlemma}

\begin{proof}
Because \(N(-G)=-N(G)\) is itself an output of the oracle, it belongs to \(K\). Therefore optimality at \(G\) implies \(\langle G,N(G)\rangle \ge \langle G,-N(G)\rangle = -\langle G,N(G)\rangle.\) Thus \(2\langle G,N(G)\rangle\ge0\).
\end{proof}

Consequently, to rule out a linear maximization-oracle representation for
\(N\), it suffices to find \(G\) such that
\(N(N(G))\ne N(G)\) but \(\|N(N(G))\|_F=\|N(G)\|_F\), or, for an odd
map, to find \(G\) such that \(\langle G,N(G)\rangle<0\).
By the sign convention above, either argument rules out an LMO representation
for the signed update \(-N\).

\begin{appxlemma}[Tall matrix-sign and row-normalization witnesses]
\label{appx:lem:tall-witnesses}
Assume \(m>n\). Each normalization map \(N\) listed in parts (i) and (ii) of Theorem~1 has an input \(G\) for which \(N(N(G))\ne N(G)\) and \(\|N(N(G))\|_F=\|N(G)\|_F\).
\end{appxlemma}

\begin{proof}
Assume throughout this subsection that \(m>n\), and put \(r=m-n\ge1.\) For \(0<t<1\), define \(A(t)\in\mathbb R^{m\times n}\) by the rows

\[
\begin{aligned}
A(t)_{1,:}&=(\sqrt{1-t},\sqrt t,0,\ldots,0),\\
A(t)_{2,:}&=(\sqrt{1-t},-\sqrt t,0,\ldots,0),
\end{aligned}
\]

followed by \(r\) copies of \(e_2^\top\), and then \(e_3^\top,\ldots,e_n^\top.\) Thus \(A(t)\) has exactly \(2+r+(n-2)=m\) rows. Every row has Euclidean norm one, so \(\RN(A(t))=A(t).\) Moreover, \(A(t)^\top A(t) = \diag\bigl(2(1-t),\,r+2t,\,1,\ldots,1\bigr).\)

Define \(Q(t):=\operatorname{msign}(A(t)) = A(t)\bigl(A(t)^\top A(t)\bigr)^{-1/2}.\) Since the columns of \(A(t)\) are orthogonal, column normalization produces exactly the same matrix: \(\CN(A(t))=Q(t).\) Since \(Q(t)\) has orthonormal columns, \(\operatorname{msign}(Q(t))=Q(t),  \CN(Q(t))=Q(t).\) The only nontrivial transition is row normalization of \(Q(t)\). For either of its first two rows, the squared second coordinate after normalization is \(\frac{t/(r+2t)} {\frac{1}{2}+t/(r+2t)} = \frac{2t}{r+4t}.\) Hence, with \(\phi(t):=\frac{2t}{r+4t},\) we have the exact identities
\begin{equation}
\label{eq:AQ-dynamics}
\begin{aligned}
\RN(A(t))&=A(t),\\
\CN(A(t))&=\operatorname{msign}(A(t))=Q(t),\\
\CN(Q(t))&=\operatorname{msign}(Q(t))=Q(t),\\
\RN(Q(t))&=A(\phi(t)).
\end{aligned}
\end{equation}

The map \(\phi\) is strictly increasing: \(\phi'(t)=\frac{2r}{(r+4t)^2}>0.\) At \(t_0=\frac{1}{2},\) we have \(\phi(t_0)=\frac1{r+2}<\frac{1}{2}=t_0.\) Strict monotonicity therefore gives \(t_0>\phi(t_0)>\phi^2(t_0)>\cdots\) for every finite number of iterations.

We now apply this one-dimensional dynamics to the maps in parts (i)--(ii) of Theorem~1.

For part (i), \eqref{eq:AQ-dynamics} gives the following transitions.
The map \(N=\operatorname{msign}\circ\RN\) sends \(A(t)\) to \(Q(t)\),
and a second application sends it to \(Q(\phi(t))\).
The map \(\RN\circ\operatorname{msign}\) sends \(A(t)\) to
\(A(\phi(t))\), then to \(A(\phi^2(t))\).
The maps \(\operatorname{msign}\circ\RN\circ\CN\) and
\(\CN\circ\RN\circ\operatorname{msign}\) both send \(A(t)\) to
\(Q(\phi(t))\), then to \(Q(\phi^2(t))\).
Finally, \(\RN\circ\CN\circ\operatorname{msign}\) has the same
transitions as \(\RN\circ\operatorname{msign}\) on this family, because
\(\CN(Q(t))=Q(t)\).
All five maps are therefore nonidempotent at \(t=t_0\).
Along these witness orbits, outputs have Frobenius norm \(\sqrt n\) when
the final operation is \(\operatorname{msign}\) or \(\CN\), and
\(\sqrt m\) when it is \(\RN\).
Lemma~\ref{appx:lem:lmo-idempotence} excludes a fixed-set linear
maximization-oracle representation for each normalization map, and hence
an LMO representation for its negative.

For part (ii),
let \(F=\operatorname{msign}\circ\RN\). Then \(F(A(t))=Q(t)\) and \(F(Q(t))=Q(\phi(t))\), so for every \(k\ge1\), \(F^k(A(t))=Q(\phi^{k-1}(t))\). Applying the same map \(F^k\) once more replaces \(\phi^{k-1}(t)\) by \(\phi^{2k-1}(t)\), which is distinct at \(t=t_0\).

For the reverse composition \(H=\RN\circ\operatorname{msign}\), the same argument gives the distinct iterates \(A(\phi^k(t_0))\) and \(A(\phi^{2k}(t_0))\). Again the outputs along these witness orbits have constant Frobenius norm, so Lemma~\ref{appx:lem:lmo-idempotence} proves part (ii).
\end{proof}

\begin{appxlemma}[Finite mixed row/column-normalization words]
\label{appx:lem:rncn-words}
Every finite composition of \(\RN\) and \(\CN\) that contains both is not a linear maximization oracle of a fixed set $\K$.
\end{appxlemma}
\begin{proof}
Repeated adjacent copies of the same normalization may be removed because \(\RN^2=\RN, \CN^2=\CN.\) Thus it suffices to consider an alternating sequence of $\RN$ and $\CN$.

Put \(r=m-1, s=n-1,\) and fix any constant \(a>1\). For \(t>0\), define the positive row-normalized matrix \(R(t)\) by its squared entries:

\[
R(t)_{ij}^2=
\begin{cases}
\dfrac{at}{at+s},&i=1,\ j=1,\\[2mm]
\dfrac1{at+s},&i=1,\ j\ge2,\\[2mm]
\dfrac t{t+s},&i\ge2,\ j=1,\\[2mm]
\dfrac1{t+s},&i\ge2,\ j\ge2.
\end{cases}
\]

Let \(C_1(t)\) denote the squared norm of the first column and \(C_2(t)\) the squared norm of any other column: \(C_1(t) = \frac{at}{at+s} + r\frac t{t+s},\) and \(C_2(t) = \frac1{at+s} + r\frac1{t+s}.\) Column normalization multiplies the squared first-to-other coordinate ratio in every row by \(C_2(t)/C_1(t)\). Row normalization preserves these within-row ratios. Therefore \(\RN(\CN(R(t)))=R(F(t)),\) where \(F(t) = t\frac{C_2(t)}{C_1(t)} = \frac{(1+ra)t+ms} {am\,t+s(a+r)}.\) This scalar map is strictly increasing. Indeed, \(F'(t) = \frac{rs(a-1)^2} {\bigl(am\,t+s(a+r)\bigr)^2} >0.\) The fixed-point equation \(F(t)=t\) is a quadratic with positive leading coefficient and negative constant term: \(am\,t^2+ \bigl[s(a+r)-(1+ra)\bigr]t-ms=0.\) It therefore has exactly one positive root, denoted \(t_\star\).

Moreover, \(F(1)-1 = \frac{n(1-a)} {am+s(a+r)} <0.\) Thus \(1>t_\star\). Since \(F\) is increasing, \(t_\star<F(t)<t \,(t>t_\star).\) Consequently \(1>F(1)>F^2(1)>\cdots>t_\star,\) and no two finite iterates coincide.

Suppose first that a reduced chain of $\RN$ and $\CN$ ends in \(\RN\). Starting from \(R(t)\), every effective \(\CN\)-then-\(\RN\) pair applies one copy of \(F\); any initial \(\RN\) is inert. Hence for some integer \(k\ge1\), \(N(R(t))=R(F^k(t)).\) Applying the same word again gives \(N(N(R(t))) = R(F^{2k}(t)).\) At \(t=1\) these are distinct.

If instead the word ends in \(\CN\), transpose the entire construction. Transposition exchanges \(\RN\) and \(\CN\), and \(\langle G,X\rangle=\langle G^\top,X^\top\rangle.\) Thus a linear maximization-oracle representation for the original word would induce one for the transposed RN-ending word, which we just ruled out.

Finally, on the positive witness family, an RN-ending word has output norm \(\sqrt m\), while a CN-ending word has output norm \(\sqrt n\). Lemma~\ref{appx:lem:lmo-idempotence} therefore proves that every finite composition containing both \(\RN\) and \(\CN\) is not a linear maximization oracle of any fixed set.
\end{proof}

The limiting row/column equilibration map \(\caE\) uses the compatible normalization \(\|\caE(G)_{i,:}\|_2^2=n/m\) and \(\|\caE(G)_{:,j}\|_2^2=1\), so every equilibrated matrix has squared Frobenius norm \(n\). The following lemma identifies the balanced representative inside a positive diagonal orbit.

\begin{appxlemma}[Uniqueness in a positive diagonal orbit]
\label{appx:lem:balanced-unique}
Let \(B,C\) be strictly positive matrices in the same positive diagonal orbit: \(C=D_1BD_2, D_1,D_2\succ0.\) If \(B\) and \(C\) have the same squared row sums and the same squared column sums, then \(B=C\).
\end{appxlemma}

\begin{proof}
Write \(S_{ij}=B_{ij}^2, U_{ij}=C_{ij}^2.\) For suitable scalars \(\alpha_i,\beta_j\), \(\log\frac{U_{ij}}{S_{ij}}=\alpha_i+\beta_j.\) Therefore

\[
\begin{aligned}
\sum_{ij}(U_{ij}-S_{ij})
\log\frac{U_{ij}}{S_{ij}}
= \sum_i\alpha_i
\sum_j(U_{ij}-S_{ij})
+ \sum_j\beta_j
\sum_i(U_{ij}-S_{ij}) =0,
\end{aligned}
\]

because the row and column marginals agree. But each summand satisfies \((u-s)\log(u/s)\ge0\) for \(u,s>0\), with equality only when \(u=s\). Hence \(U=S\), and positivity gives \(B=C\).
\end{proof}

\begin{appxlemma}[The limiting equilibration map is not a linear maximization oracle]
\label{appx:lem:E-nonlmo}
The limiting equilibration map \(E\) in Theorem~1 - (iii) is not a linear maximization oracle of any fixed set.
\end{appxlemma}

\begin{proof}
If \(B\) is already balanced and \(L,R\succ0\) are diagonal, then
\begin{equation}
\label{eq:E-diagonal-invariance}
\caE(LBR)=B.
\end{equation}

We now construct a positive balanced matrix explicitly. First suppose \(m\ge n\). Set \(b=\frac1{\sqrt{2m}}, a=\sqrt{\frac{n+1}{2m}}, c=\frac1{\sqrt m}.\) Let the first \(n\) rows form the \(n\times n\) matrix \((a-b)I_n+b\mathbf1\mathbf1^\top,\) and let each of the remaining \(m-n\) rows equal \(c\mathbf1^\top.\) Call the resulting matrix \(B\).

For a top row, \(a^2+(n-1)b^2 = \frac{n+1}{2m} + \frac{n-1}{2m} = \frac nm.\) For a lower row, \(nc^2=\frac nm.\) Each column therefore has squared norm \(\frac nm+(m-n)c^2 = \frac nm+\frac{m-n}{m} = 1.\) Thus \(B\) is balanced.

When \(m<n\), use \(B_{m,n} = \sqrt{\frac nm}\,B_{n,m}^\top,\) which has squared row norms \(n/m\) and squared column norms one.

We prove that \(E\) is not a linear maximization oracle by a local perturbation inside the balanced manifold. Let \(L=\diag(2,1,\ldots,1), R=\diag(2,1,\ldots,1),\) and define \(G=LBR.\) By \eqref{eq:E-diagonal-invariance}, \(E(G)=B.\) Let \(Z\) be supported on the leading \(2\times2\) block with \(Z_{11}=Z_{22}=1,\, Z_{12}=Z_{21}=-1.\) Its row sums and column sums are zero. For sufficiently small \(t\), define \(B(t)\) entrywise by \(B(t)_{ij} = \sqrt{B_{ij}^2+tZ_{ij}}.\) All entries remain positive. Since \(Z\) has zero row and column sums, \(B(t)\) has exactly the same squared row and column norms as \(B\), so \(B(t)\) is balanced.

Hence \(E(B(t))=B(t),\) so \(B(t)\) is itself an output of \(E\). On the other hand, \(\left.\frac{d}{dt}\langle G,B(t)\rangle\right|_{t=0} = \frac{1}{2}\sum_{ij}L_{ii}R_{jj}Z_{ij}.\) Only the first \(2\times2\) block contributes, giving \(\left.\frac{d}{dt}\langle G,B(t)\rangle\right|_{t=0} = \frac{1}{2}(2-1)(2-1) = \frac{1}{2}>0.\) Thus for all sufficiently small \(t>0\), \(\langle G,B(t)\rangle > \langle G,B\rangle = \langle G,E(G)\rangle.\) If \(E\) were a linear maximization oracle of a fixed set \(K\), then every output \(B(t)\) would belong to \(K\), contradicting optimality of \(E(G)=B\) at the input \(G\).

Therefore the limiting equilibration \(E\) is not a linear maximization oracle of any fixed set.
\end{proof}

Now we find explicit counterexamples of \Cref{lem:alignment} for $\msign \circ\  \mathcal{E}$ and the iteration of $\sign$ and $\msign$, proving these chains are not linear maximization oracles.
\begin{appxlemma}[Negative alignment after equilibration]
\label{appx:lem:msign-E}
For every shape covered by Theorem~1, the odd map \(\operatorname{msign}\circ \caE\) admits an input \(G\) such that
\[
\langle G,(\operatorname{msign}\circ \caE)(G)\rangle<0.
\]
Hence \(\operatorname{msign}\circ \caE\) is not a linear maximization oracle of any fixed set.
\end{appxlemma}

\begin{proof}
Write \(N=\operatorname{msign}\circ\caE\). Both constituents are odd, so \(N\) is odd. By Lemma~\ref{appx:lem:lmo-odd}, it is therefore enough to find \(G\) with \(\langle G,\operatorname{msign}(\caE(G))\rangle<0.\)

First suppose \(m>n\), and use the balanced positive matrix \(B\) from the previous subsection. Put \(r=m-n, \, s=a+(n-1)b, \, \lambda=s^2+nrc^2.\) The leading \(n\times n\) block of \(B\) is \((a-b)I_n+bJ_n.\) A direct calculation gives

\[
B^\top B
=
(a-b)^2\left(I-\frac{1}{n}J_n\right)
+
\lambda\,\frac{1}{n}J_n.
\]

Hence, writing \(Q=\operatorname{msign}(B),\) the leading \(n\times n\) block of \(Q\) is
\begin{equation}
\label{eq:balanced-polar-block}
Q_{\rm top}
=
I-\frac{1}{n}J_n
+
\frac{s}{\sqrt{\lambda}}\frac{1}{n}J_n.
\end{equation}
Since \(r>0\), \(\lambda>s^2,\) and therefore every off-diagonal entry of the block \eqref{eq:balanced-polar-block} is strictly negative: \((Q_{\rm top})_{ij} = \frac{1}{n}\left(\frac{s}{\sqrt{\lambda}}-1\right)<0, \, i\ne j.\) Choose, say, the entry \((1,2)\), and set \(\delta=-B_{12}Q_{12}>0.\) For a scalar \(M\ge1\), let \(L_M=\diag(M,1,\ldots,1), \, R_M=\diag(1,M,1,\ldots,1),\) and define \(G_M=L_MBR_M.\) By diagonal invariance, \(E(G_M)=B,\) so \(N(G_M)=Q.\) The selected \((1,2)\) contribution to
\(\langle G_M,Q\rangle\) is \(-\delta M^2\). Every other entry receives weight at most \(M\). Moreover \(\sum_{ij}|B_{ij}Q_{ij}| \le \|B\|_F\|Q\|_F = n,\) because \(B\) has squared Frobenius norm \(n\) and \(Q\) has \(n\) orthonormal columns. Therefore \(\langle G_M,Q\rangle \le -\delta M^2+nM.\) Taking \(M>\frac n\delta\) gives \(\langle G_M,N(G_M)\rangle<0.\) For \(m<n\), transpose the corresponding \(n\times m\) witness. Compatible equilibration obeys the transpose relation up to the positive scalar \(\sqrt{n/m}\), and the polar factor removes this scalar. Thus the same negative inner product is obtained.

Now suppose \(m=n=q\ge3.\) Consider first the \(3\times3\) positive circulant matrix

\[
B_3
=
\frac1{\sqrt{14}}
\begin{pmatrix}
1&2&3\\
3&1&2\\
2&3&1
\end{pmatrix}.
\]

Every row and every column has norm one. Furthermore, \(B_3^\top B_3 = \frac{3}{14}I_3+\frac{11}{14}J_3,\) and direct evaluation of the polar factor gives

\[
Q_3:=\operatorname{msign}(B_3)
=
\frac{1}{3}
\begin{pmatrix}
1-\sqrt{3}&1&1+\sqrt{3}\\
1+\sqrt{3}&1-\sqrt{3}&1\\
1&1+\sqrt{3}&1-\sqrt{3}
\end{pmatrix}.
\]

In particular, \((Q_3)_{11}=\frac{1-\sqrt{3}}{3}<0.\) For \(q>3\), let \(B_0=B_3\oplus I_{q-3}.\) This matrix is balanced and invertible, but has zero entries. To obtain a strictly positive witness, define the squared-entry matrix \(S_\varepsilon = (1-\varepsilon)(B_0\circ B_0) + \varepsilon\frac{J_q}{q},\) where \(\circ\) denotes the Hadamard product, and set \((B_\varepsilon)_{ij}=\sqrt{(S_\varepsilon)_{ij}}.\) Every row and column of \(S_\varepsilon\) sums to one, so \(B_\varepsilon\) is balanced. It is strictly positive. Since \(B_\varepsilon\to B_0\) and \(B_0\) is invertible, \(B_\varepsilon\) is invertible for all sufficiently small \(\varepsilon>0\). Continuity of the polar factor at invertible matrices gives \(\operatorname{msign}(B_\varepsilon) \longrightarrow Q_3\oplus I_{q-3}.\) Hence, for sufficiently small fixed \(\varepsilon\), \((B_\varepsilon)_{11}>0, \, [\operatorname{msign}(B_\varepsilon)]_{11}<0.\) We can therefore repeat the preceding diagonal-amplification argument, now scaling the first row and first column. For sufficiently large \(M\), \(G_M=L_MB_\varepsilon L_M\) satisfies \(E(G_M)=B_\varepsilon\) and \(\left\langle G_M, \operatorname{msign}(B_\varepsilon) \right\rangle<0.\) Thus \(\operatorname{msign}\circ E\) has a negative-alignment witness in every shape \(m,n\ge5\). Since it is odd, Lemma~\ref{appx:lem:lmo-odd} proves that it is not a linear maximization oracle of any fixed set.
\end{proof}

\begin{appxlemma}[Mixed matrix-sign and entrywise-sign words]
\label[appxlemma]{appx:lem:mixed-sign}
Every finite composition of \(\operatorname{msign}\) and \(\operatorname{Sgn}\) that contains both admits an input \(G\) with negative alignment,
\[
\langle G,N(G)\rangle<0.
\]
Hence no such composition is a linear maximization oracle of any fixed set.
\end{appxlemma}
\begin{proof}
Write \(P=\operatorname{msign}\) and \(S=\operatorname{Sgn}\). Since both maps are idempotent, every finite chain containing both reduces to an alternating chain of length at least two. Both maps are odd, so by Lemma~\ref{appx:lem:lmo-odd} it suffices to construct a negative-alignment witness for every alternating chain.

We first consider square matrices of size \(q\ge5\). Suppose the first applied operation is \(S\). Put \(k=q-1\), \(J=\mathbf1\mathbf1^\top\), \(B=J-2I\), and \(A=B+2e_1e_1^\top\). Thus \(A\) has \(+1\) in its \((1,1)\) entry, \(-1\) on the remaining diagonal, and \(+1\) off the diagonal.

Let \(Q=P(A)\). To determine its signs, let \(u=k^{-1/2}(0,1,\ldots,1)^\top\) and decompose \(\mathbb R^q=\operatorname{span}\{e_1,u\}\oplus\{e_1,u\}^\perp\). On the orthogonal complement, \(A=-2I\), while on \(\operatorname{span}\{e_1,u\}\) it is represented by

$$
H_A=
\begin{pmatrix}
1&\sqrt{k}\\
\sqrt{k}&k-2
\end{pmatrix}.
$$

Since \(\det H_A=-2\), its eigenvalues have opposite signs. Setting \(D=\sqrt{k^2-2k+9}\), its matrix sign is

$$
\operatorname{msign}(H_A)
=
\frac1D
\begin{pmatrix}
3-k&2\sqrt{k}\\
2\sqrt{k}&k-3
\end{pmatrix}.
$$

It follows that \(Q\) has negative diagonal entries and positive off-diagonal entries for \(k\ge4\). Hence \(S(Q)=B\).

The eigenvalues of \(B\) are \(q-2\) along \(\mathbf1\) and \(-2\) on \(\mathbf1^\perp\), so \(P(B)=2J/q-I=:R\). For \(q\ge5\), \(R\) again has negative diagonal and positive off-diagonal entries, hence \(S(R)=B\). Thus every alternating chain beginning with \(S\) outputs \(Q\), \(B\), or \(R\).

Now take \(G_t=B+(t+1)e_1e_1^\top\). Since \(S(G_t)=A\), the possible chain outputs remain \(Q,B,R\), independently of \(t\). Their \((1,1)\) entries are all strictly negative. Therefore \(\langle G_t,X\rangle=\langle B,X\rangle+(t+1)X_{11}\to-\infty\) for every \(X\in\{Q,B,R\}\). Taking \(t\) sufficiently large gives a negative-alignment witness for every chain beginning with \(S\).

Now suppose the first applied operation is \(P\). Let \(w=(8q,8q,1,\ldots,1)^\top\), let \(d_w=\|w\|_2^2\), and define the Householder reflection \(Q=I-2ww^\top/d_w\). Its diagonal entries are positive and its off-diagonal entries are negative, so \(S(Q)=2I-J=:B\). Since \(B\) has eigenvalue \(2-q\) along \(\mathbf1\) and eigenvalue \(2\) on \(\mathbf1^\perp\), we have \(P(B)=I-2J/q=:R\), and again \(S(R)=B\). Thus every such alternating chain eventually alternates between \(B\) and \(R\).

It remains to choose an input with polar factor \(Q\) and negative alignment with both \(B\) and \(R\). Define \(z=(I-ww^\top/d_w)\mathbf1\). Then \(z\perp w\), so \(Qz=z\). Moreover, if \(r=\|z\|_2^2\), then

$$
r
=
q-\frac{(17q-2)^2}{128q^2+q-2}
>
\frac{q}{2}
$$

for \(q\ge5\). Consequently \(z^\top Bz=2r-r^2<0\) and \(z^\top Rz=r-2r^2/q<0\).

For \(\gamma>0\), set \(G_\gamma=Q(I+\gamma zz^\top)\). The second factor is positive definite, so \(P(G_\gamma)=Q\). Since \(Qz=z\), we also have \(G_\gamma=Q+\gamma zz^\top\). Therefore both \(\langle G_\gamma,B\rangle\) and \(\langle G_\gamma,R\rangle\) tend to \(-\infty\) as \(\gamma\to\infty\). Taking \(\gamma\) sufficiently large gives a negative-alignment witness for every chain beginning with \(P\).

Finally, for rectangular matrices, take \(q=\min\{m,n\}\) and zero-pad the corresponding \(q\times q\) witness. Both \(P\) and \(S\) commute with zero padding, and Frobenius inner products are unchanged. Lemma~\ref{appx:lem:lmo-odd} therefore rules out a linear maximization-oracle representation for every finite composition containing both \(\operatorname{msign}\) and \(\operatorname{Sgn}\).
\end{proof}

\begin{appxtheorem}[\Cref{thm:not_lmo} of the paper]
\label{appx:thm:nonlmo}
Let \(m,n\ge5\), and let normalizations map zero rows and columns to zero.
For each map \(N:\mathbb{R}^{m\times n}\to\mathbb{R}^{m\times n}\) listed below,
there is no fixed set \(K\subseteq\mathbb{R}^{m\times n}\) such that
$$
N(G)\in\arg\max_{X\in K}\langle G,X\rangle
$$
for every \(G\):
\begin{enumerate}[label=(\roman*), leftmargin=*, nosep]
\item[(i)] \(\operatorname{msign}\circ\RN\), \(\operatorname{msign}\circ\RN\circ\CN\), \(\RN\circ\operatorname{msign}\), \(\CN\circ\RN\circ\operatorname{msign}\), and \(\RN\circ\CN\circ\operatorname{msign}\) for \(m>n\).
\item[(ii)] Alternation with matrix sign: \((\operatorname{msign}\circ\RN)^{\circ k}\) and \((\RN\circ\operatorname{msign})^{\circ k}\) for every \(k\ge1\) and \(m>n\).
\item[(iii)] Every finite composition of \(\RN\) and \(\CN\) that contains both, in particular \((\RN\circ\CN)^{\circ k}\); its limit \(\caE\); and \(\operatorname{msign}\circ \caE\), with the last two considered on matrices with no zero entries.
\item[(iv)] Every finite composition of \(\operatorname{msign}\) and \(\operatorname{Sgn}\) that contains both.
\end{enumerate}
Equivalently, for every listed normalization map \(N\), the signed update \(-N\) is not an LMO of any fixed set.
\end{appxtheorem}
\begin{proof}
Parts (i) and (ii) follow from Lemma~\ref{appx:lem:tall-witnesses} and Lemma~\ref{appx:lem:lmo-idempotence}. The finite-word claim in part (iii) follows from Lemma~\ref{appx:lem:rncn-words} and Lemma~\ref{appx:lem:lmo-idempotence}; the limit \(E\) is excluded by Lemma~\ref{appx:lem:E-nonlmo}, and \(\operatorname{msign}\circ E\) is excluded by Lemmas~\ref{appx:lem:msign-E} and \ref{appx:lem:lmo-odd}. Part (iv) follows from Lemmas~\ref{appx:lem:mixed-sign} and \ref{appx:lem:lmo-odd}.
\end{proof}

\subsection{Convergence and non-convergence guarantees}
\label{app:convergence}

\begin{appxproposition}[\Cref{prop:divergence} of the paper]
\label{appx:prop:divergence}
Let \(T\) be positively scale invariant with \(\langle G,T(G)\rangle>0\)
for some \(G\). For the convex, \(1\)-smooth quadratic
\[
f(X)=\frac{\langle G,X\rangle^2}{2\|G\|_F^2}
\]
with \(X_0=G/\|G\|_F^2\), the iterates
\(X_{t+1}=X_t+\eta_tT(\nabla f(X_t))\) satisfy \(f(X_t)\to\infty\)
for every nonnegative step-size sequence with \(\sum_t\eta_t=\infty\).
\end{appxproposition}

\begin{proof}
Let \(a:=\langle G,T(G)\rangle>0\). The Hessian of \(f\) is the
rank-one orthogonal projector \(H[Z]=G\langle G,Z\rangle/\|G\|_F^2\),
so \(f\) is convex, nonnegative, and \(1\)-smooth.
Put \(c_t=\langle G,X_t\rangle\). Since \(c_0=1\), whenever \(c_t>0\),
\(\nabla f(X_t)=(c_t/\|G\|_F^2)G\) is a positive multiple of \(G\).
Positive scale invariance gives \(T(\nabla f(X_t))=T(G)\).
Thus \(c_{t+1}=c_t+\eta_t a\), and induction yields
\[
c_t=1+a\sum_{j<t}\eta_j>0,
\qquad
X_t=X_0+\left(\sum_{j<t}\eta_j\right)T(G).
\]
Consequently,
\[
f(X_t)=\frac{\left(1+a\sum_{j<t}\eta_j\right)^2}{2\|G\|_F^2}
\longrightarrow\infty.
\]
\end{proof}

The witnesses in Lemmas~\ref{appx:lem:mixed-sign} and
\ref{appx:lem:msign-E} have \(\langle G,N(G)\rangle<0\) for the
\emph{normalization map} \(N\). Their signed updates \(T=-N\) therefore
satisfy \(\langle G,T(G)\rangle>0\), as required by
Proposition~\ref{appx:prop:divergence}. Equivalently, the stateless iteration
\(X_{t+1}=X_t-\eta_tN(\nabla f(X_t))\) diverges.
The Aurora claim follows by the same sign conversion from the next lemma.

\begin{appxlemma}[Negative alignment for two Aurora cycles]
\label[appxlemma]{appx:lem:aurora-negative}
If \(n=2\) and \(m\ge258\), then \((\operatorname{msign}\circ\RN)^{\circ2}\) has a full-column-rank input \(G\) with nonzero entries such that
\[
\left\langle G,(\operatorname{msign}\circ\RN)^{\circ2}(G)\right\rangle<0.
\]
\end{appxlemma}

\begin{proof}
Let \(m=r+2\) with \(r\ge256\), put
\(s=\sqrt{5+1/r}\), and define
\[
 q=\frac1s\left(\underbrace{r^{-1},\ldots,r^{-1}}_{r\text{ entries}},-2,1\right)^\top,
 \,
 Q=\begin{bmatrix}m^{-1/2}\mathbf1&q\end{bmatrix}.
\]
Then \(\mathbf1^Tq=0\) and \(\|q\|=1\), so \(Q^TQ=I_2\). Write
\(c=\sqrt m/s\) and \(a=c/r\). The rows of \(\RN(Q)\) have slopes
\(a,-2c,c\), with the first slope repeated \(r\) times. Let
\(C=\RN(Q)^T\RN(Q)\). Direct summation gives \(C_{22}<3\), \(C_{11}\le m\), and
\[
 cC_{12}\ge \frac{c^2}{1+a^2}-\frac12
 =\frac{(r+2)r^2}{5r^2+2r+2}-\frac12
 \ge \frac r6-\frac12.
\]
For \(h=\sqrt{\det C}\), the inverse square root of a positive definite \(2\times2\)
matrix is a positive scalar multiple of
\(\bigl(\begin{smallmatrix}C_{22}+h&-C_{12}\\-C_{12}&C_{11}+h\end{smallmatrix}\bigr)\).
Hence the first coordinate of the last row of
\(V:=\operatorname{msign}(\RN(Q))\) has the sign of
\(C_{22}+h-cC_{12}\). Since \(h\le\sqrt{3m}\) and, for \(r\ge256\),
\(r/6-1/2>3+\sqrt{3(r+2)}\), this coordinate is strictly negative.

Now let \(A_\varepsilon=\RN([\mathbf1,\varepsilon q])\). A QR expansion gives
\(\operatorname{msign}(A_\varepsilon)\to Q\) as \(\varepsilon\downarrow0\): the
second QR diagonal is \(\varepsilon+O(\varepsilon^3)\), the off-diagonal term is
\(O(\varepsilon^3)\), and the polar factor of the resulting upper-triangular
\(2\times2\) matrix tends to the identity. Therefore
\[
 (\operatorname{msign}\circ\RN)^2([\mathbf1,\varepsilon q])
 \longrightarrow V.
\]
For sufficiently small fixed \(\varepsilon>0\), the inner product between the last row of
\(A_\varepsilon\) and the last row of this output is negative. Multiplying only the last input row
by a sufficiently large positive scalar leaves the first row-normalization, and hence the chained
output, unchanged, while making the total alignment negative. The input remains full column rank
with nonzero entries. This proves the claimed negative-alignment example for
\((\operatorname{msign}\circ\RN)^{\circ2}\) when \(n=2\) and \(m\ge258\).
\end{proof}

\begin{appxdefinition}[\Cref{def:approx_lmo} of the paper]
\label{appx:def:approx-lmo}
Let \(K\) be a compact set containing the origin and define
\[
m_K(G):=\min_{X\in K}\langle G,X\rangle.
\]
For \(\nu\in(0,1]\), a map \(U\) is a \emph{\(\nu\)-approximate LMO}
of \(K\) if, for every \(G\),
\[
U(G)\in K,
\qquad
\langle G,U(G)\rangle\le\nu m_K(G).
\]
\end{appxdefinition}

Write \(h_K(G):=\max_{X\in K}\langle G,X\rangle\). The nonnegative
minimization gap is \(-m_K(G)=h_K(-G)=h_{-K}(G)\).
For an origin-symmetric set \(K=-K\), this equals \(h_K(G)\), the
notation used for the norm-ball bounds in the main text.

\begin{appxproposition}[\Cref{prop:convergence} of the paper]
\label{appx:prop:approx-descent}
Let \(f\) have an \(L\)-Lipschitz gradient, let
\(\Delta=f(X_0)-\inf f<\infty\), and let \(U\) be a
\(\nu\)-approximate LMO of a compact set \(K\) containing the origin.
Write \(R:=\max_{X\in K}\|X\|_F\).
\begin{enumerate}
\item[(i)] Gradient descent. With
\(X_{t+1}=X_t+\eta U(\nabla f(X_t))\) and
\[
\eta=\sqrt{\frac{2\Delta}{LR^2T}},
\qquad
\frac1T\sum_{t=0}^{T-1}[-m_K(\nabla f(X_t))]
\le\frac{R}{\nu}\sqrt{\frac{2L\Delta}{T}}.
\]
\item[(ii)] SGD with momentum. With unbiased stochastic gradients of
conditional variance at most \(\sigma^2\),
\[
M_t=\beta M_{t-1}+(1-\beta)\widehat g_t,
\qquad
X_{t+1}=X_t+\eta U(M_t),
\]
\(1-\beta=T^{-1/2}\), \(\eta=cT^{-3/4}\), and fixed initialization
with finite initial tracking error, the same average in expectation is
\[
O\!\left(
\nu^{-1}\bigl[\Delta/c+R(LRc+\sigma)\bigr]T^{-1/4}
\right).
\]
\end{enumerate}
For \(K=-K\), both statements can equivalently be written with
\(h_K(\nabla f(X_t))\) in place of \(-m_K(\nabla f(X_t))\).
\end{appxproposition}

\begin{proof}
Feasibility implies \(\|U(v)\|_F\le R\). The minimum-value function is
\(R\)-Lipschitz:
\(\lvert m_K(g)-m_K(v)\rvert\le R\|g-v\|_F\).
The robust descent-alignment inequality is
\begin{equation}
\label{eq:robust-alignment}
\langle g,U(v)\rangle
\le\nu m_K(g)+(1+\nu)R\|g-v\|_F.
\end{equation}
Indeed,
\[
\langle g,U(v)\rangle
=\langle v,U(v)\rangle+\langle g-v,U(v)\rangle
\le\nu m_K(v)+R\|g-v\|_F,
\]
and the Lipschitz bound for \(m_K\) gives \eqref{eq:robust-alignment}.

For part (i), let \(g_t=\nabla f(X_t)\). Smoothness and approximate
optimality give
\[
f(X_{t+1})
\le f(X_t)+\nu\eta m_K(g_t)+\frac{LR^2\eta^2}{2}.
\]
Summing and using \(f(X_0)-f(X_T)\le\Delta\) yields
\[
\frac1T\sum_{t=0}^{T-1}[-m_K(g_t)]
\le\frac{\Delta}{\nu\eta T}+\frac{LR^2\eta}{2\nu}.
\]
The stated choice of \(\eta\) balances the two terms and proves part (i).
The degenerate cases \(R=0\) or \(\Delta=0\) have zero minimization gap
along the corresponding constant iterates and can be treated separately.

For part (ii), write \(\widehat g_t=g_t+\zeta_t\), where
\(\mathbb E[\zeta_t\mid\mathcal F_t]=0\) and
\(\mathbb E[\|\zeta_t\|_F^2\mid\mathcal F_t]\le\sigma^2\).
Let \(e_t=M_t-g_t\). Applying \eqref{eq:robust-alignment} with \(v=M_t\)
and then smoothness gives
\[
\nu\eta[-m_K(g_t)]
\le f(X_t)-f(X_{t+1})+(1+\nu)R\eta\|e_t\|_F
+\frac{LR^2\eta^2}{2}.
\]
Thus it remains to bound \(T^{-1}\sum_t\mathbb E\|e_t\|_F\).
For \(t\ge1\), the tracking recursion is
\[
e_t=\beta e_{t-1}+\beta(g_{t-1}-g_t)+(1-\beta)\zeta_t.
\]
Since \(\|X_t-X_{t-1}\|_F\le R\eta\), smoothness gives
\(\|g_t-g_{t-1}\|_F\le LR\eta\). Unrolling the recursion, bounding
the drift by the triangle inequality, and using martingale orthogonality
for the noise yields
\begin{equation}
\label{eq:ema-track}
\mathbb E\|e_t\|_F
\le\beta^{t+1}E_0+
\frac{\beta}{1-\beta}LR\eta+
\sigma\sqrt{\frac{1-\beta}{1+\beta}},
\end{equation}
where \(E_0:=\mathbb E\|M_{-1}-g_0\|_F<\infty\).
Averaging and summing the descent inequality gives
\[
\frac1T\sum_{t<T}\mathbb E[-m_K(g_t)]
\le\frac1\nu\left[
\frac{\Delta}{\eta T}
+CR\left(
\frac{LR\eta}{1-\beta}+\sigma\sqrt{1-\beta}
+\frac{E_0}{(1-\beta)T}
\right)
+\frac{LR^2\eta}{2}
\right]
\]
for a universal constant \(C\). Setting \(1-\beta=T^{-1/2}\) and
\(\eta=cT^{-3/4}\) gives
\[
\frac1T\sum_{t<T}\mathbb E[-m_K(g_t)]
=O\!\left(\nu^{-1}
[\Delta/c+R(LRc+\sigma)]T^{-1/4}\right).
\]
For fixed initialization, the initial-error and direct-smoothness terms
have lower orders \(T^{-1/2}\) and \(T^{-3/4}\), respectively.
If \(r\mathbb B_F\subseteq K\) for some \(r>0\), then
\(-m_K(g)\ge r\|g\|_F\), so these bounds also control the gradient norm.
Here \(\mathbb B_F\) is the Frobenius unit ball.
The proof uses momentum before applying \(U\); an EMA of a statistic
internal to \(U\) is not covered.
\end{proof}

\begin{appxproposition}[\cref{prop:approx} of the paper]
\label{appx:prop:positive-chains}
Let $K_{\rm row}:=\{X:\max_i\|X_{i,:}\|_2\le1\}.$
Then $-(\RN\circ\operatorname{msign})$ is a $1/\sqrt m$-approximate LMO of $K_{\rm row}$, while $-(\operatorname{msign}\circ\RN)$ is a $1/\sqrt m$-approximate LMO of the spectral-norm unit ball.
By transposition, the same holds for \(\CN\) with \(1/\sqrt n\).
\end{appxproposition}
\begin{proof}
We prove the two row-normalized statements; the column-normalized versions follow by transposition. The zero input is immediate, so assume \(G\ne0\).
First let \(Q=\operatorname{msign}(G)\) and write the polar factorization \(G=QH\), where
\(H=(G^TG)^{1/2}\succeq0\). Let \(D\) row-normalize the nonzero rows of \(Q\), and set
\(D_{ii}=1\) on zero rows, so that \(\RN(Q)=DQ\). Since \(QQ^T\) is an orthogonal
projection, every row of \(Q\) has norm at most one and hence \(D\succeq I\). Also
\(HQ^TQ=H\). Therefore
\[
 \langle G,\RN(Q)\rangle
 =\tr(HQ^TDQ)
 =\tr H+\tr\!\bigl(HQ^T(D-I)Q\bigr)
 \ge \tr H=\|G\|_*.
\]
The support function of
\(K_{\rm row}=\{X:\max_i\|X_{i,:}\|_2\le1\}\) is
\(h_{K_{\rm row}}(G)=\sum_i\|G_{i,:}\|_2\). Since
\(\sum_i\|G_{i,:}\|_2\le\sqrt m\|G\|_F\le\sqrt m\|G\|_*\), we obtain
\[
 \langle G,(\RN\circ\operatorname{msign})(G)\rangle
 \ge \frac1{\sqrt m}h_{K_{\rm row}}(G).
\]
Because \(K_{\rm row}=-K_{\rm row}\), the signed output
\(-\RN(Q)\) is feasible and
\[
\langle G,-\RN(Q)\rangle
\le -\frac1{\sqrt m}h_{K_{\rm row}}(G)
=\frac1{\sqrt m}m_{K_{\rm row}}(G).
\]
This is the required approximate \emph{minimization} inequality.

For the reverse order, let \(A=\RN(G)\) and write
\(G=\diag(d_1,\ldots,d_m)A\), where \(d_i=\|G_{i,:}\|_2\). For each nonzero
row, \(\|A_{i,:}\|_2=1\). Put \(M=A^TA\). On \(\operatorname{range}(M)\),
\((M^\dagger)^{1/2}\succeq \|A\|_{\op}^{-1}I\), and each \(A_{i,:}^T\) belongs to this
range. Hence, with \(P=\operatorname{msign}(A)=A(M^\dagger)^{1/2}\),
\[
 \langle A_{i,:},P_{i,:}\rangle
 \ge \frac1{\|A\|_{\op}}
 \ge \frac1{\sqrt m}
\]
for every nonzero row. Zero rows have \(d_i=0\) and contribute nothing. Thus
\[
 \langle G,P\rangle
 \ge \frac1{\sqrt m}\sum_i d_i
 \ge \frac1{\sqrt m}\|G\|_*.
\]
The last inequality is the nuclear-norm triangle inequality applied to the single-row decomposition
of \(G\). Write \(K_{\rm sp}:=\{X:\|X\|_{\op}\le1\}\).
Since \(\|P\|_{\op}\le1\), the signed output \(-P\) is feasible and
\[
\langle G,-P\rangle\le-\frac1{\sqrt m}\|G\|_*
=\frac1{\sqrt m}m_{K_{\rm sp}}(G).
\]
Thus \(-(\operatorname{msign}\circ\RN)\) is a
\(1/\sqrt m\)-approximate LMO.
\end{proof}

\begin{appxproposition}[\Cref{prop:dimension} of the paper]
\label{appx:prop:upper-factor}
Let \(n=2\). For every closed convex set \(K\) containing the image of
the \emph{signed} map in question, the best uniform approximation factor satisfies
\[
\nu\le3(m-1)^{-1/4}
\]
for \(-(\RN\circ\operatorname{msign})\) when \(m\) is odd, and
\[
\nu\le\frac3{\sqrt{m-1}}
\]
for \(-(\operatorname{msign}\circ\RN)\).
\end{appxproposition}

\begin{proof}
For an unbounded \(K\), interpret \(m_K\) as an infimum; if it is
\(-\infty\) in a tested direction, no positive approximation factor is
possible. Set \(n=2\). First let \(N=\RN\circ\operatorname{msign}\)
and \(U=-N\). Put \(r=m-1\), and assume \(r\) is even. Define
\(p,q\in\mathbb R^m\) by
\[
p_1=\sqrt{1-r^{-1/2}},\qquad p_i=r^{-3/4}\quad(i\ge2),
\]
\[
q_1=0,\qquad
q_i=\begin{cases}
r^{-1/2},&2\le i\le r/2+1,\\
-r^{-1/2},&r/2+2\le i\le m.
\end{cases}
\]
Then \(\|p\|=\|q\|=1\) and \(p^\top q=0\), so \(Q=[p,q]\) has
orthonormal columns. For \(G=Q\diag(1,\varepsilon)\), \(\varepsilon>0\),
we have \(\operatorname{msign}(G)=Q\). With \(a=r^{-3/4}\) and
\(b=r^{-1/2}\), direct rowwise normalization gives
\[
\langle G,N(G)\rangle
=\sqrt{1-r^{-1/2}}+
 r\frac{a^2+\varepsilon b^2}{\sqrt{a^2+b^2}}
\le3
\]
for sufficiently small \(\varepsilon\), for example
\(0<\varepsilon\le r^{-1}\).
Define \(Q'\) by \(Q'_{1,:}=e_2^\top\) and
\(Q'_{i,:}=r^{-1/2}e_1^\top\) for \(i\ge2\).
Then \((Q')^\top Q'=I_2\), so
\(N(Q')=\RN(Q')\), whose first row is \(e_2^\top\) and remaining
rows are \(e_1^\top\). Since \(-N(Q')\in K\),
\[
m_K(G)\le\langle G,-N(Q')\rangle=-ra=-r^{1/4}.
\]
If \(U\) is a \(\nu\)-approximate LMO of \(K\), then
\[
-3\le\langle G,U(G)\rangle
\le\nu m_K(G)\le-\nu r^{1/4},
\]
which proves \(\nu\le3(m-1)^{-1/4}\).

For the reverse chain, let \(N=\operatorname{msign}\circ\RN\) and
\(U=-N\). Set \(\varepsilon=m^{-2}\) and define
\(G\in\mathbb R^{m\times2}\) by
\[
G_{1,:}=e_1^\top,\qquad
G_{i,:}=\varepsilon e_1^\top\ (2\le i\le m-1),\qquad
G_{m,:}=\varepsilon e_2^\top.
\]
Then \(A=\RN(G)\) has \(m-1\) copies of \(e_1^\top\) and one copy
of \(e_2^\top\), so \(A^\top A=\diag(m-1,1)\). Hence
\[
\langle G,N(G)\rangle
=\frac{1+(m-2)\varepsilon}{\sqrt{m-1}}+\varepsilon
\le\frac3{\sqrt{m-1}}.
\]
Take \(Q'\) with first row \(e_1^\top\) and all remaining rows
\(e_2^\top/\sqrt{m-1}\). It has orthonormal columns and satisfies
\(N(Q')=Q'\). Thus \(-Q'\in K\), and
\[
m_K(G)\le\langle G,-Q'\rangle
=-1-\frac{\varepsilon}{\sqrt{m-1}}\le-1.
\]
Approximate minimization gives
\[
-\frac3{\sqrt{m-1}}
\le\langle G,U(G)\rangle
\le\nu m_K(G)\le-\nu,
\]
so \(\nu\le3/\sqrt{m-1}\). Both bounds vanish as \(m\to\infty\),
proving \Cref{prop:dimension}.
\end{proof}

\section{Appendix: Proofs in \cref{sec:lam_separation}}
\label[appendix]{app:lam_separation}

We use the model and scaling assumptions of \Cref{sec:lam_separation}.
For clarity, we first specify the notation used throughout this appendix.
Let
\[
X=[x_1,\ldots,x_N]\in\mathbb R^{d\times N},\qquad
Y=[y_1,\ldots,y_N]\in\mathbb R^{n\times N},
\]
where the entries of $X$ and $Y$ are independent centered Gaussians with
variances $1/d$ and $1/n$, respectively, and $X$ is independent of $Y$.
Write $p_i=i^{-\alpha}/\sum_{j=1}^N j^{-\alpha}$, and define
\[
P_N=I_N-\frac1N\mathbf1\mathbf1^\top,\qquad
s=Yp,\qquad K=P_N\diag(p)Y^\top\in\mathbb R^{N\times n}.
\]
Differentiating the cross-entropy loss at zero gives the exact identity
\[
G=-\nabla\mathcal L(0)
 =\sum_{i=1}^Np_i
 \left(\Sigma x_i-\frac1N\sum_{j=1}^N\Sigma x_j\right)y_i^\top
 =\Sigma XK.
\]
In particular,
\[
K^\top e_i=p_i y_i-\frac{s}{N},\qquad
K^\top K=\sum_{i=1}^Np_i^2y_i y_i^\top-\frac1Nss^\top.
\]
We abbreviate the three updates by
\[
W_G=G,\qquad W_M=\operatorname{pol}(G),\qquad
W_C=\operatorname{pol}(\RN(G)).
\]
Here $\operatorname{pol}$ is the exact canonical polar factor, equivalently
$\operatorname{msign}$ in the main text. All inverse matrix powers below
are taken on the full-column-rank event, which has probability one for
all sufficiently large $n$. The statements concern these exact updates,
not a fixed number of Newton--Schulz iterations.

Set
\[
h=(n/N)^\alpha,\qquad L=\log(e nNd),\qquad
s_1=\sum_{a=1}^d\sigma_a,\qquad s_2=\sum_{a=1}^d\sigma_a^2.
\]
The assumptions imply $L\asymp\log n$, $N/n\to\infty$, $d/n\to\infty$,
and
\begin{align*}
&n^{-1/2+\epsilon}\le h\le n^{-(1-\beta+\epsilon)},\qquad
s_1\asymp d^{1-\beta}\ge c n,\qquad c\le s_2\le C,\\
&h\to0,\qquad hn^{1-\beta}\le n^{-\epsilon},\qquad
\frac{\sqrt{L/n}}{h}\to0,\qquad
\frac{\sqrt n}{Nh}\to0,\\
&L^\alpha\sqrt{\frac nd}\to0,\qquad
\frac{\sqrt{ns_2}\,L^{\alpha+2}}{s_1h}
 \le Cn^{-\epsilon}L^{\alpha+2}\to0.
\end{align*}
Constants $c,C$ may change between displays and depend only on the fixed
parameters in the assumptions. A subscript $b$ also permits dependence
on a fixed $b\in(0,1)$. The eigenvalues of a positive semidefinite matrix
are indexed in nonincreasing order. Conditional $O_{\mathbb P}$ bounds
below are uniform over the specified query event; bounds for an index
$i\ge bN$ are uniform over deterministic indices in this range.

The proof is organized by conditioning: first on the query realization
$Y$, then on the observed gradient $G$, and finally on the residual score
randomness left after conditioning on $(Y,G)$.

\subsection{Query-side event and population geometry}

\begin{appxlemma}[High-probability regularity of the queries]
\label{appx:lem:lam-query-event}
There are constants $c,C>0$ such that the event $\mathcal E_Y$ defined by
\begin{align}
&\max_{i\le N}\|y_i\|_2\le 2,\qquad
\left\|\frac1NYY^\top\right\|_{\op}\le \frac{C}{n},
\label{eq:EY-basic}\\
&c n^{-2\alpha}I\preceq K^\top K\preceq CI,\qquad
c\le \tr(K^\top K)\le C,
\label{eq:EY-gram}\\
&\tr\!\bigl((K^\top K)^{-1/2}\bigr)\ge c n^{\alpha+1},\qquad
\|(K^\top K)^{-1/2}s\|_2\le C\sqrt n,
\label{eq:EY-inverse}\\
&\lambda_\ell(K^\top K)\ge c\ell^{-2\alpha}
\qquad (1\le \ell\le \lfloor A_*L\rfloor),
\label{eq:EY-leading}
\end{align}
and
\begin{equation}
\sup_{B\ne0}\frac1N\#\left\{
i:\|By_i\|_2\le\eta\frac{\|B\|_F}{\sqrt n}
\right\}
\le C\left(\eta+\sqrt{\frac nN}\right)
\qquad (0<\eta<1)
\label{eq:EY-smallball}
\end{equation}
satisfies $\mathbb P_Y(\mathcal E_Y^c)\to0$.
Here $A_*>0$ is a fixed sufficiently large constant, and the supremum
in \eqref{eq:EY-smallball} is over all nonzero real matrices with $n$
columns and any finite number of rows. The inequalities hold
simultaneously for all $B$ and $\eta$.

For a deterministic index $i$, define
\[
s_{-i}=\sum_{j\ne i}p_jy_j,\qquad
K_{-i}=P_N\bigl(\diag(p)-p_ie_ie_i^\top\bigr)Y^\top.
\]
Let $\mathcal E_{Y,-i}$ consist of
\eqref{eq:EY-gram}--\eqref{eq:EY-leading} with $K,s$ replaced by
$K_{-i},s_{-i}$. Thus $\mathcal E_{Y,-i}$ is measurable with respect to
$Y_{-i}=(y_j)_{j\ne i}$; it does not include conditions involving $y_i$.
For every fixed $b\in(0,1)$,
\[
\sup_{i\ge bN}
\mathbb P_Y\!\left((\mathcal E_Y\cap\mathcal E_{Y,-i})^c\right)\to0.
\]
\end{appxlemma}

\begin{proof}
Gaussian norm and singular-value bounds give \eqref{eq:EY-basic} with
probability tending to one, since $\log N=o(n)$ and $N/n\to\infty$.
Write
\[
S=\sum_{j=1}^Np_j^2y_jy_j^\top,\qquad
K^\top K=S-\frac1Nss^\top.
\]
To control the centering term, represent $y_j=\varepsilon_j\widetilde y_j$,
where the $\varepsilon_j$ are independent Rademacher variables, independent
of the Gaussian vectors $\widetilde y_j$. Conditional on $\widetilde Y$,
let $A=[p_j\widetilde y_j]_j$. Then
\[
v=s^\top S^{-1}s
 =\varepsilon^\top A^\top(AA^\top)^{-1}A\varepsilon.
\]
The matrix between the signs is a rank-$n$ orthogonal projection.
Consequently
\[
\mathbb E_\varepsilon v=n,\qquad
\operatorname{Var}_\varepsilon(v)
 =2\sum_{j\ne k}\bigl[A^\top(AA^\top)^{-1}A\bigr]_{jk}^{\,2}
 \le2n.
\]
Thus $v\le2n$ with probability at least $1-2/n$. On this event, for
all sufficiently large $n$,
\[
(1-v/N)S\preceq K^\top K\preceq S,\qquad
s^\top(K^\top K)^{-1}s=\frac{v}{1-v/N}\le Cn.
\]
The equality is the rank-one inverse formula. This proves the second
bound in \eqref{eq:EY-inverse} and allows spectral bounds for $S$ to be
transferred to $K^\top K$ with fixed constants.

The first $4n$ queries satisfy
$\sum_{j\le4n}y_jy_j^\top\succeq cI$ except with probability $e^{-cn}$.
Since $p_j\ge c n^{-\alpha}$ on this block,
$S\succeq c n^{-2\alpha}I$.
The upper trace bound follows from
$\max_j\|y_j\|\le2$ and $\sum_jp_j^2\le C$.
The first term $p_1^2\|y_1\|^2$ gives a positive lower trace bound with
probability tending to one. For $q=\lfloor A_*L\rfloor=o(n)$,
the $n\times q$ matrix $[y_1,\ldots,y_q]$ has all singular values in
$[c,C]$ except with probability $e^{-cn}$. The nonzero eigenvalues of
its weighted Gram matrix therefore give
$\lambda_\ell(S)\ge cp_\ell^2\ge c\ell^{-2\alpha}$ for $\ell\le q$.
This proves \eqref{eq:EY-gram} and \eqref{eq:EY-leading}.

For the inverse trace, put $r=\lfloor n/4\rfloor$ and
$S_{>r}=\sum_{j>r}p_j^2y_jy_j^\top$.
For a fixed unit vector $u$, $nu^\top S_{>r}u$ is a weighted sum of
independent squared standard Gaussians. Its weights obey
\[
\sum_{j>r}p_j^2\le Cn^{1-2\alpha},\qquad
\sum_{j>r}p_j^4\le Cn^{1-4\alpha},\qquad
\max_{j>r}p_j^2\le Cn^{-2\alpha}.
\]
The Gaussian moment-generating-function bound, at a tail parameter
of order $n$, followed by a $1/4$-net of the sphere, yields
$\|S_{>r}\|_{\op}\le Cn^{-2\alpha}$ except with probability $e^{-cn}$.
The first $r$ terms of $S$ have rank at most $r$. Hence at least $n-r$
eigenvalues of $K^\top K\preceq S$ are at most $Cn^{-2\alpha}$, and
\[
\tr\!\bigl((K^\top K)^{-1/2}\bigr)\ge cn^{\alpha+1}.
\]

We next prove \eqref{eq:EY-smallball}, including its uniformity.
Write $z_i=\sqrt n\,y_i$ and $a_n=\sqrt{n/N}$.
On \eqref{eq:EY-basic},
$\|N^{-1}\sum_i z_i z_i^\top\|_{\op}\le C_0$ for a fixed $C_0$.
For a fixed unit vector $v$ and $0<t\le1$,
\[
\mathbb P(|v^\top z_i|\le2t)\le C_1t.
\]
A Bernoulli Chernoff bound therefore gives
\[
\mathbb P\!\left\{
\frac1N\#\{i:|v^\top z_i|\le2t\}>C_2t
\right\}\le e^{-cNt}
\]
for a sufficiently large fixed $C_2$; when $C_2t\ge1$ the event is empty.
For each dyadic $t\in[a_n,1]$, also including the smallest dyadic
$t$ not below $a_n$, take a sphere net of mesh
$\delta=c_0t^{3/2}$. Its cardinality is at most $(3/\delta)^n$.
A union bound over these nets has failure probability at most
\[
C\log(e/a_n)
\exp\!\left\{Cn\log(e/a_n)-cNa_n\right\}=o(1),
\]
because $N/n\to\infty$ and
$Na_n/n=\sqrt{N/n}\gg\log(eN/n)$.
For an arbitrary unit vector $u$, choose a net point $v$ within
$\delta$. If $|u^\top z_i|\le t$ but $|v^\top z_i|>2t$, then
$|(u-v)^\top z_i|>t$. Thus
\[
\frac1N\#\{i:|(u-v)^\top z_i|>t\}
 \le \frac1{Nt^2}\sum_i((u-v)^\top z_i)^2
 \le \frac{C_0\delta^2}{t^2}\le Ct.
\]
Combining these bounds and using monotonicity between dyadic thresholds
and below $a_n$ proves, simultaneously for all unit $u$ and all $t>0$,
\[
\frac1N\#\{i:|u^\top z_i|\le t\}\le C(t+a_n).
\]
For any nonzero $B$, diagonalize
\[
\frac{B^\top B}{\|B\|_F^2}
 =\sum_k w_k u_k u_k^\top,\qquad w_k\ge0,\quad\sum_k w_k=1.
\]
If $\|Bz_i\|\le\eta\|B\|_F$, at least half of the weight is on
indices $k$ for which $|u_k^\top z_i|\le\sqrt2\eta$. Therefore
\[
\begin{split}
\frac1N\#\{i:\|Bz_i\|\le\eta\|B\|_F\}
&\le2\sum_k w_k\frac1N\#\{i:|u_k^\top z_i|\le\sqrt2\eta\}\\
&\le C(\eta+a_n).
\end{split}
\]
This is \eqref{eq:EY-smallball}; it is pathwise and allows $B$ to depend
on the data.

Finally, for a fixed $b>0$, a deterministic index $i\ge bN$ lies outside
the first $4n$ and first $\lfloor A_*L\rfloor$ queries once $n$ is large.
Deleting its weighted term leaves these blocks intact and only decreases
the positive semidefinite tail matrix. The same centering argument
applies with its weight set to zero; the projection still has rank $n$.
All failure bounds are independent of the particular deterministic
$i\ge bN$. This proves the final assertion.
\end{proof}

\begin{appxlemma}[Bulk leverage on $\mathcal E_Y$]
\label{appx:lem:lam-bulk-leverage}
On \(\mathcal E_Y\), \(t_j:=(K^\top K)^{-1/2}K^\top e_j\) satisfies \(\|t_j\|_2\le1\) for every \(j\). For every fixed \(b>0\),
\[
\max_{i\ge bN}\|t_i\|_2\le C_bh.
\]
\end{appxlemma}

\begin{proof}
The matrix \(K(K^\top K)^{-1/2}\) has orthonormal columns, so every row has norm at most one. For \(i\ge bN\),
\[
\|t_i\|_2
\le
p_i\|(K^\top K)^{-1/2}\|_{\op}\|y_i\|_2
+\frac1N\|(K^\top K)^{-1/2}s\|_2.
\]
On \(\mathcal E_Y\), the first term is at most \(C_bN^{-\alpha}n^\alpha=C_bh\), while the second is \(C\sqrt n/N=o(h)\). Absorbing the latter term into the constant proves the result.
\end{proof}

On $\mathcal E_Y$, define the response matrices for a fresh
$\xi\sim N(0,I_n)$ by
\[
S_\xi=\xi^\top K^\top K\,\xi,\qquad
C_Y=\mathbb E\frac{\xi\xi^\top}{S_\xi},\qquad
H_Y=\mathbb E\frac{\xi\xi^\top}{\sqrt{S_\xi}},\qquad
J_Y=H_YC_Y^{-1/2}.
\]
Once $Y$ is fixed, these are deterministic matrices.

\begin{appxlemma}[Response moments on $\mathcal E_Y$]
\label{appx:lem:lam-response-moments}
On $\mathcal E_Y$,
\[
cI\preceq C_Y\preceq CI,\qquad cI\preceq J_Y\preceq I.
\]
The matrices $C_Y,H_Y,J_Y$ commute with $K^\top K$.
For every fixed positive integer $r$, for all sufficiently large $n$,
\[
\sup_{Y\in\mathcal E_Y}\mathbb E\bigl[(\xi^\top K^\top K\xi)^{-r}\bigr]
 \le C_r<\infty.
\]
If
\[
(1-e)K^\top K\preceq\widetilde K^\top\widetilde K
 \preceq(1+e)K^\top K,\qquad 0\le e\le\frac12,
\]
and $\widetilde J$ is defined from $\widetilde K^\top\widetilde K$
in the same way, then
\[
\|\widetilde J-J_Y\|_{\op}\le Ce.
\]
The same conclusions apply to $K_{-i}$ on $\mathcal E_{Y,-i}$.
\end{appxlemma}

\begin{proof}
For a fixed $r$, choose a fixed integer $q>2r$.
In an eigenbasis of $K^\top K$, \eqref{eq:EY-leading} gives
\[
S_\xi\ge cq^{-2\alpha}\sum_{\ell=1}^q\xi_\ell^2.
\]
The negative $r$-th moment of $\chi_q^2$ is finite; more explicitly it is
$2^{-r}\Gamma(q/2-r)/\Gamma(q/2)$. This proves the inverse-moment
bound, uniformly on $\mathcal E_Y$.

Sign symmetry in this eigenbasis makes $C_Y$ and $H_Y$ diagonal. Put
\[
c_j=\mathbb E\frac{\xi_j^2}{S_\xi},\qquad
h_j=\mathbb E\frac{\xi_j^2}{\sqrt{S_\xi}}.
\]
Cauchy--Schwarz gives $h_j^2\le c_j$. Also
\[
1\le h_j\,\mathbb E[\xi_j^2\sqrt{S_\xi}],\qquad
\mathbb E[\xi_j^2S_\xi]
 =\tr(K^\top K)+2\lambda_j(K^\top K)\le C.
\]
Since $\mathbb E[\xi_j^2\sqrt{S_\xi}]
\le\sqrt{\mathbb E[\xi_j^2S_\xi]}$, it follows that $h_j\ge c$.
A further application of Cauchy--Schwarz and the inverse second-moment
bound gives $c_j\le(\mathbb E\xi_j^4)^{1/2}
(\mathbb E S_\xi^{-2})^{1/2}\le C$.
Thus $c\le c_j\le C$ and $c\le h_j/\sqrt{c_j}\le1$, proving the
spectral bounds and commutation statements.

The relative Loewner comparison gives
$|\widetilde S_\xi^{-1}-S_\xi^{-1}|\le CeS_\xi^{-1}$ and
$|\widetilde S_\xi^{-1/2}-S_\xi^{-1/2}|\le CeS_\xi^{-1/2}$ pointwise.
Testing the resulting matrix expectations against arbitrary vectors
shows $\|\widetilde C-C_Y\|_{\op}+\|\widetilde H-H_Y\|_{\op}\le Ce$.
Both covariance matrices are bounded below by a fixed positive multiple
of the identity. The integral formula
\[
D^{-1/2}=\frac1\pi\int_0^\infty t^{-1/2}(D+tI)^{-1}\,dt
\]
and the resolvent identity therefore give
$\|\widetilde C^{-1/2}-C_Y^{-1/2}\|_{\op}\le Ce$.
Taking the difference of the products $\widetilde H\widetilde C^{-1/2}$
and $H_YC_Y^{-1/2}$ proves the last bound.
Only the Gram-matrix conditions were used, so the same proof applies
on $\mathcal E_{Y,-i}$.
\end{proof}

\begin{appxlemma}[Population retrieval margin on $\mathcal E_Y$]
\label{appx:lem:lam-population-margin}
Fix \(b\in(0,1)\). There is \(c_b>0\) such that for every fixed \(\varepsilon>0\),
\[
\sup_{i\ge bN}
\mathbb P_Y\!\left(
t_i^\top J_Yy_i<c_bh
\ \ \text{or} \ \ 
\max_{j\ne i}|t_j^\top J_Yy_i|>\varepsilon h
\,\middle|\,
\mathcal E_Y
\right)
\to0.
\]
\end{appxlemma}

\begin{proof}
We first record the relative square-root estimate needed below.
For positive definite $A_0,A_1$, suppose
\[
\delta=\|A_0^{-1/2}(A_1-A_0)A_0^{-1/2}\|_F\le\frac12.
\]
Then
\[
\|A_1^{-1/2}A_0^{1/2}-I\|_{\op}\le C\delta.
\]
Indeed, let $A(t)=A_0+t(A_1-A_0)$ and $H(t)=A(t)^{1/2}$.
Differentiating $H(t)^2=A(t)$ gives
$H\dot H+\dot H H=A_1-A_0$.
If $F(t)=H(t)^{-1}(A_1-A_0)H(t)^{-1}$, then in an eigenbasis of $H(t)$
with eigenvalues $\lambda_j>0$,
\[
[H^{-1}\dot H]_{jk}=\frac{\lambda_k}{\lambda_j+\lambda_k}F(t)_{jk},
\qquad \|H^{-1}\dot H\|_F\le\|F(t)\|_F.
\]
Since $A(t)\succeq(1-t\delta)A_0$,
\[
\|F(t)\|_F\le\frac{\delta}{1-t\delta},\qquad
\|H(t)^{-1}A_0^{1/2}\|_{\op}\le(1-t\delta)^{-1/2}.
\]
It follows that
\[
\left\|\frac{d}{dt}\bigl(H(t)^{-1}A_0^{1/2}\bigr)\right\|_F
\le\frac{\delta}{(1-t\delta)^{3/2}}.
\]
Integration from zero to one proves the claim. In particular, this
estimate uses the Frobenius norm of the relative perturbation, not a
condition-number-dependent additive perturbation bound.

Fix a deterministic \(i\ge bN\), and work on \(\mathcal E_Y\cap\mathcal E_{Y,-i}\). Let \(J_{-i}\) denote the response operator obtained from \(K_{-i}^\top K_{-i}\). Direct expansion gives
\[
K^\top K-K_{-i}^\top K_{-i}
=
p_i^2(1-1/N)y_iy_i^\top
-
\frac{p_i}{N}(y_is_{-i}^\top+s_{-i}y_i^\top).
\]
Set
\[
e_i:=
\left\|
(K_{-i}^\top K_{-i})^{-1/2}
\bigl(K^\top K-K_{-i}^\top K_{-i}\bigr)
(K_{-i}^\top K_{-i})^{-1/2}
\right\|_F.
\]
Since \(p_i n^\alpha\asymp_b h\),
\[
e_i\le C_b\left(h^2+\frac{h\sqrt n}{N}\right)=O_b(h^2).
\]
Hence
\[
(1-e_i)K_{-i}^\top K_{-i}
\preceq
K^\top K
\preceq
(1+e_i)K_{-i}^\top K_{-i}.
\]
Lemma~\ref{appx:lem:lam-response-moments} and the relative square-root
estimate just proved (with $A_0=K_{-i}^\top K_{-i}$ and
$A_1=K^\top K$) give
\begin{align}
\|J_Y-J_{-i}\|_{\op}&\le C_bh^2,
\label{eq:lam-J-delete}\\
\left\|
(K^\top K)^{-1/2}(K_{-i}^\top K_{-i})^{1/2}-I
\right\|_{\op}
&\le C_bh^2,
\label{eq:lam-root-delete}\\
\left\|
(K^\top K)^{-1/2}J_Y
-
(K_{-i}^\top K_{-i})^{-1/2}J_{-i}
\right\|_{\op}
&\le C_bn^\alpha h^2.
\label{eq:lam-op-delete}
\end{align}

For \(j\ne i\), define
\[
t_j^{-i}:=
(K_{-i}^\top K_{-i})^{-1/2}
\left(p_jy_j-\frac{s_{-i}}N\right),
\]
and \(t_i^{-i}:=-(K_{-i}^\top K_{-i})^{-1/2}s_{-i}/N\). These are rows of \(K_{-i}(K_{-i}^\top K_{-i})^{-1/2}\), so \(\|t_j^{-i}\|_2\le1\). Also,
\[
t_j
=
(K^\top K)^{-1/2}
(K_{-i}^\top K_{-i})^{1/2}t_j^{-i}
-
\frac{p_i}{N}(K^\top K)^{-1/2}y_i,
\]
hence
\begin{equation}
\max_{j\ne i}\|t_j-t_j^{-i}\|_2
\le C_b(h^2+h/N).
\label{eq:lam-t-delete}
\end{equation}

Conditional on a realization of $Y_{-i}$ satisfying $\mathcal E_{Y,-i}$,
the deleted objects are fixed and $y_i\sim N(0,I_n/n)$ is independent.
At this step we do not condition on $\mathcal E_Y$, which also involves
$y_i$. Each wrong deleted response is centered Gaussian with variance
at most $1/n$. A Gaussian tail bound and a union bound give
\[
\max_{j\ne i}|(t_j^{-i})^\top J_{-i}y_i|
=
O_{\mathbb P}\!\left(\sqrt{\frac Ln}\right)
=
o_{\mathbb P}(h).
\]
The replacement error from \((t_j^{-i},J_{-i})\) to \((t_j,J_Y)\) is \(o(h)\) by \eqref{eq:lam-J-delete}, \eqref{eq:lam-t-delete}, and \(\|y_i\|\le2\).

For the correct coordinate, put
\[
A_i:=
(K_{-i}^\top K_{-i})^{-1/2}J_{-i}.
\]
This matrix is symmetric positive definite and independent of \(y_i\), with
\[
\tr A_i\ge c n^{\alpha+1},\,
\|A_i\|_{\op}\le Cn^\alpha,\,
\|A_i\|_F\le Cn^{\alpha+1/2}.
\]
Therefore
\[
\mathbb E[y_i^\top A_i y_i\mid Y_{-i}]
=
\frac1n\tr A_i\ge cn^\alpha,
\]
\begin{align*}
\operatorname{Var}(y_i^\top A_i y_i\mid Y_{-i})
&=\frac2{n^2}\|A_i\|_F^2\\
&\le Cn^{2\alpha-1}.
\end{align*}
Chebyshev gives \(y_i^\top A_i y_i\ge cn^\alpha/2\) with failure \(O(1/n)\). Finally,
\[
t_i^\top J_Yy_i
=
p_i y_i^\top(K^\top K)^{-1/2}J_Yy_i
-
\frac1N s^\top(K^\top K)^{-1/2}J_Yy_i.
\]
The matrix replacement error is \(O_b(h^3)\) after multiplying by \(p_i\), while
\[
\frac1N
\left|
s^\top(K^\top K)^{-1/2}J_Yy_i
\right|
\le
C\frac{\sqrt n}{N}
=
o(h).
\]
Since \(p_i n^\alpha\asymp_bh\), the correct coordinate is at least \(c_bh-o(h)\). The preceding bounds show that the margin failure event, intersected with \(\mathcal E_Y\cap\mathcal E_{Y,-i}\), has probability tending to zero, with a bound depending on \(i\) only through \(i\ge bN\). Since
\(\sup_{i\ge bN}\mathbb P_Y((\mathcal E_Y\cap\mathcal E_{Y,-i})^c)\to0\) and
\(\mathbb P_Y(\mathcal E_Y)\to1\), conditioning on \(\mathcal E_Y\) yields the stated supremum.
\end{proof}

\subsection{Condition on \(Y\): randomness revealed through \(G\)}

Fix $Y\in\mathcal E_Y$. For the $a$-th row of $X$, define
\[
\xi_a=\sqrt d\,(K^\top K)^{-1/2}K^\top X_{a,:}^\top.
\]
Conditional on $Y$, the vectors $\xi_1,\ldots,\xi_d$ are independent
$N(0,I_n)$, and
\[
G_{a,:}=\frac{\sigma_a}{\sqrt d}\xi_a^\top(K^\top K)^{1/2}.
\]
The conditional law of the $\xi_a$ does not depend on $Y$, so the family
$\{\xi_a\}_{a=1}^d$ is independent of $Y$. Conversely, $(Y,G)$ determines
all $\xi_a$, because every $\sigma_a$ is positive.

Define the observed-gradient matrices
\begin{align*}
\mathcal H&=d(K^\top K)^{-1/2}G^\top G(K^\top K)^{-1/2},\\
\mathcal J&=d(K^\top K)^{-1/2}G^\top\Sigma^2G(K^\top K)^{-1/2}.
\end{align*}
Conditional on $Y$,
\[
\mathcal H=\sum_{a=1}^d\sigma_a^2\xi_a\xi_a^\top,\qquad
\mathcal J=\sum_{a=1}^d\sigma_a^4\xi_a\xi_a^\top.
\]

\begin{appxlemma}[Observed-gradient estimates]
\label{appx:lem:lam-baseline-X}
There is $C<\infty$ such that
\[
\sup_{Y\in\mathcal E_Y}\mathbb P_\xi\!\left(
\|\mathcal H\mathcal J^{-1}\mathcal H\|_{\op}>Cn\mid Y\right)\to0,
\]
and
\[
\sup_{Y\in\mathcal E_Y}\mathbb E_\xi\!\left[
\sqrt d\,\|G(K^\top K)^{-1/2}\|_*\mid Y\right]
\le Cn^{3/2-\beta}.
\]
For every fixed $\delta>0$,
\[
\sup_{Y\in\mathcal E_Y}\mathbb P_\xi\!\left(
\left|\frac1n\tr\mathcal H-\sum_{a\ge1}\sigma_a^2\right|>\delta
\ \text{or}\ 
\left|\frac1n\tr\mathcal J-\sum_{a\ge1}\sigma_a^4\right|>\delta
\ \middle|\ Y\right)\to0.
\]
\end{appxlemma}

\begin{proof}
Split the first $2n$ rows from the rest:
$\mathcal H=\mathcal H_0+\mathcal H_1$ and
$\mathcal J=\mathcal J_0+\mathcal J_1$.
Let $\Xi_0$ have rows $\xi_a^\top$, $a\le2n$, and let $B$ have rows
$\sigma_a^2\xi_a^\top$ on the same block. Then
$\mathcal J_0=B^\top B$ and $\mathcal H_0=\Xi_0^\top B=B^\top\Xi_0$.
Since $B(B^\top B)^{-1}B^\top$ is an orthogonal projection,
\[
\mathcal H_0\mathcal J^{-1}\mathcal H_0
\preceq\mathcal H_0\mathcal J_0^{-1}\mathcal H_0
\preceq\Xi_0^\top\Xi_0\preceq CnI
\]
except with probability $e^{-cn}$.
The same block gives
$\mathcal J\succeq cn^{-4\beta}\Xi_0^\top\Xi_0
\succeq cn^{1-4\beta}I$.
On each dyadic block of rows with indices of order $2^k n$, the maximum
weight $\sigma_a^2$ is at most $C(2^k n)^{-2\beta}$ and the unweighted
Gaussian Gram matrix has operator norm at most $C2^k n$, with failure
$e^{-c2^k n}$. Summing over these blocks, using $2\beta>1$, gives
\[
\|\mathcal H_1\|_{\op}\le
C\sum_{k\ge1}(2^k n)^{1-2\beta}\le Cn^{1-2\beta}
\]
except with probability $Ce^{-cn}$. Therefore
$\mathcal H_1\mathcal J^{-1}\mathcal H_1\preceq CnI$.
Finally,
\[
(\mathcal H_0+\mathcal H_1)\mathcal J^{-1}(\mathcal H_0+\mathcal H_1)
\preceq2\mathcal H_0\mathcal J^{-1}\mathcal H_0
       +2\mathcal H_1\mathcal J^{-1}\mathcal H_1,
\]
which proves the first assertion.

The matrix $\sqrt d\,G(K^\top K)^{-1/2}$ has rows
$\sigma_a\xi_a^\top$. Bound the nuclear norm of the first $n$ rows by
the sum of their row norms, and bound that of the remaining rows by
$\sqrt n$ times their Frobenius norm. Jensen's inequality gives
\[
\begin{split}
\mathbb E_\xi\!\left[\sqrt d\,\|G(K^\top K)^{-1/2}\|_*\mid Y\right]
&\le\sqrt n\sum_{a\le n}\sigma_a
 +\sqrt n\left(n\sum_{a>n}\sigma_a^2\right)^{1/2}\\
&\le Cn^{3/2-\beta}.
\end{split}
\]
For the trace statements,
\[
\mathbb E_\xi\frac{\tr\mathcal H}{n}=s_2,\qquad
\operatorname{Var}_\xi\!\left(\frac{\tr\mathcal H}{n}\right)
 =\frac2n\sum_{a\le d}\sigma_a^4\le\frac Cn.
\]
The analogous variance for $\tr\mathcal J/n$ is
$2\sum_{a\le d}\sigma_a^8/n\le C/n$.
The corresponding finite sums converge to the stated infinite sums,
so Chebyshev's inequality completes the proof.
All these laws and bounds are uniform in $Y$.
\end{proof}

For any update $W\in\mathbb R^{d\times n}$, write
\[
m_i(W)=K(K^\top K)^{-1}G^\top Wy_i,\qquad
m_{ji}(W)=e_j^\top m_i(W).
\]
When $\|\Sigma Wy_i\|_2>0$, define
\[
a^{(i)}(W)=\frac{\sqrt d\,m_i(W)}{\|\Sigma Wy_i\|_2}.
\]
If $Wy_i=0$, set $a^{(i)}(W)=0$; all scores then vanish and strict
retrieval fails. This convention covers zero denominators in the
statements below. For the two baselines, nonzero queries have nonzero
denominators almost surely.

\begin{appxcorollary}[Conditional mean-energy bound]
\label{appx:cor:lam-mean-energy}
For any update $W$ with $\col(W)\subseteq\col(G)$, on the event
$\|\mathcal H\mathcal J^{-1}\mathcal H\|_{\op}\le Cn$,
\[
\|a^{(i)}(W)\|_2^2\le Cn\qquad\text{for every }i.
\]
\end{appxcorollary}
\begin{proof}
The case $Wy_i=0$ follows from the convention above. Otherwise, write
$Wy_i=Gq$ and $\widetilde q=(K^\top K)^{1/2}q$. Direct substitution gives
\[
a^{(i)}(W)=K(K^\top K)^{-1/2}
 \frac{\mathcal H\widetilde q}
 {\sqrt{\widetilde q^\top\mathcal J\widetilde q}}.
\]
The matrix $K(K^\top K)^{-1/2}$ has orthonormal columns. Hence
\[
\|a^{(i)}(W)\|_2^2
\le\|\mathcal H\mathcal J^{-1/2}\|_{\op}^2
=\|\mathcal H\mathcal J^{-1}\mathcal H\|_{\op}\le Cn.
\]
\end{proof}

\begin{appxlemma}[Correct conditional mean for gradient descent]
\label{appx:lem:lam-gd-correct}
Fix $b\in(0,1)$ and $u>0$. Then
\[
\sup_{Y\in\mathcal E_Y}\sup_{i\ge bN}
\mathbb P_\xi\!\left(|a_i^{(i)}(W_G)|>u\mid Y\right)\to0.
\]
\end{appxlemma}

\begin{proof}
For $W_G=G$, set $q=(K^\top K)^{1/2}y_i$ and $t=t_i$.
The case $q=0$ is immediate. Otherwise,
\[
a_i^{(i)}(W_G)=\frac{t^\top\mathcal Hq}{\sqrt{q^\top\mathcal Jq}}.
\]
Write $e=q/\|q\|$ and $t=\langle t,e\rangle e+t_\perp$.
The Gaussian coordinates along $e$ and $t_\perp$ are independent.
Conditioning on those along $e$ gives the distributional identity
\[
a_i^{(i)}(W_G)\overset d=
\langle t,e\rangle
\frac{\sum_{a\le d}\sigma_a^2g_a^2}
 {\sqrt{\sum_{a\le d}\sigma_a^4g_a^2}}
+\|t_\perp\|_2g_0,
\]
where $g_0,g_1,g_2,\ldots$ are independent standard Gaussians.
Under a common coupling across $d$, the numerator is bounded by
$\sum_{a\ge1}\sigma_a^2g_a^2<\infty$ almost surely, and the denominator
is at least $\sigma_1^2|g_1|>0$ almost surely.
Thus the displayed ratio is uniformly tight, independently of $Y$ and $i$.
Lemma~\ref{appx:lem:lam-bulk-leverage} gives $\|t_i\|\le C_bh$ for
$i\ge bN$, and $h\to0$. This proves the assertion.
\end{proof}

\begin{appxlemma}[Correct conditional mean for Muon]
\label{appx:lem:lam-muon-correct}
For every fixed $u>0$ and $\delta>0$,
\[
\sup_{Y\in\mathcal E_Y}\mathbb P_\xi\!\left[
\frac1N\#\{i:|a_i^{(i)}(W_M)|>u\}>\delta\ \middle|\ Y\right]\to0.
\]
\end{appxlemma}

\begin{proof}
Put $W=W_M$, and define
\[
\nu_i=t_i^\top(K^\top K)^{-1/2}G^\top Wy_i,\qquad
v_i=\|\Sigma Wy_i\|_2.
\]
For $v_i>0$, $a_i^{(i)}(W)=\sqrt d\,\nu_i/v_i$.
The correct expansion is obtained with
\[
A=(K^\top K)^{-1}G^\top W,\qquad
z=A^\top s=W^\top G(K^\top K)^{-1}s:
\qquad
\nu_i=p_i y_i^\top Ay_i-\frac1Nz^\top y_i.
\]
On $\mathcal E_Y$, pathwise even when $A$ depends on the data,
\[
\frac1N\sum_i|y_i^\top Ay_i|
\le\tr\!\left(|\operatorname{sym}A|\frac{YY^\top}{N}\right)
\le\frac Cn\|A\|_*,
\]
where $\operatorname{sym}A=(A+A^\top)/2$ and
$\|\operatorname{sym}A\|_*\le\|A\|_*$.
Since $\|W\|_{\op}\le1$, Lemma~\ref{appx:lem:lam-baseline-X} gives
\[
\begin{split}
\mathbb E_\xi[\sqrt d\,\|A\|_*\mid Y]
&\le\|(K^\top K)^{-1/2}\|_{\op}\,
 \mathbb E_\xi[\sqrt d\,\|G(K^\top K)^{-1/2}\|_*\mid Y]\\
&\le Cn^{\alpha+3/2-\beta}.
\end{split}
\]
Also, although $W$ depends on $G$, its operator norm bound implies
\[
\begin{split}
\mathbb E_\xi[d\|z\|_2^2\mid Y]
&\le\mathbb E_\xi[d\|G(K^\top K)^{-1}s\|_2^2\mid Y]\\
&=s_2\|(K^\top K)^{-1/2}s\|_2^2\le Cn.
\end{split}
\]
Consequently $\sqrt d\,\|A\|_*=O_{\mathbb P}(n^{\alpha+3/2-\beta})$
and $\sqrt d\,\|z\|_2=O_{\mathbb P}(\sqrt n)$, uniformly in $Y$.
For fixed $b>0$, use $p_i\le C_bN^{-\alpha}$ on $i\ge bN$ and
\[
\frac1N\sum_i|z^\top y_i|
\le\left(z^\top\frac{YY^\top}{N}z\right)^{1/2}
\le\frac C{\sqrt n}\|z\|_2
\]
to obtain
\begin{equation}
\begin{split}
\frac1N\sum_{i\ge bN}|\sqrt d\,\nu_i|
&\le C_b\frac{N^{-\alpha}}n\sqrt d\,\|A\|_*
 +\frac C{N\sqrt n}\sqrt d\,\|z\|_2\\
&=O_{\mathbb P,b}\!\left(hn^{1/2-\beta}+N^{-1}\right).
\end{split}
\label{eq:lam-muon-numerator}
\end{equation}

Since $W_M$ is the polar factor of the full-column-rank matrix $G$,
\[
\begin{split}
\|\Sigma W_M\|_F^2
&=\tr\!\left[(G^\top G)^{-1}G^\top\Sigma^2G\right]
 =\tr(\mathcal H^{-1}\mathcal J)\\
&\ge\frac{\tr\mathcal J}{\|\mathcal H\|_{\op}}
 \ge\frac{\tr\mathcal J}{\tr\mathcal H}.
\end{split}
\]
The trace limits in Lemma~\ref{appx:lem:lam-baseline-X} have positive
finite values. Hence $\|\Sigma W_M\|_F\ge c_0>0$ with conditional
probability tending to one, uniformly in $Y$.
On this event, apply \eqref{eq:EY-smallball} pathwise with
$B=\Sigma W_M$. For each fixed $\eta\in(0,1)$, all but at most
$CN(\eta+\sqrt{n/N})$ indices satisfy
\[
v_i\ge c_0\eta/\sqrt n.
\]
No independence of $B$ and $Y$ is required.
Counting the exceptional indices and then using Markov's inequality
for the remaining deterministic list of numerators gives
\[
\begin{split}
\frac1N\#\{i:|a_i^{(i)}(W_M)|>u\}
\le b+\frac1N+C\eta+C\sqrt{\frac nN}
 +\frac{\sqrt n}{uc_0\eta N}\sum_{i\ge bN}|\sqrt d\,\nu_i|.
\end{split}
\]
By \eqref{eq:lam-muon-numerator}, the last term is
\[
O_{\mathbb P,b}\!\left(
\frac{hn^{1-\beta}+\sqrt n/N}{u\eta}\right)=o_{\mathbb P}(1).
\]
For a given $\delta>0$, choose $b,\eta$ small and fixed so that
$b+C\eta<\delta/2$, and then send $n\to\infty$.
The exceptional observed-gradient event also has probability tending
to zero uniformly in $Y$, which proves the assertion.
\end{proof}

For the chained update, the answer-coordinate scales cancel exactly:
$\RN(G)=\RN(XK)$, since every $\sigma_a$ is positive. Define
\[
\widetilde W_{a,:}
 =\frac{\xi_a^\top C_Y^{-1/2}}
 {\sqrt{d\,\xi_a^\top K^\top K\xi_a}}.
\]

\begin{appxlemma}[Fixed-$Y$ approximation for the chain]
\label{appx:lem:lam-chain-X}
There is $C<\infty$ such that
\[
\sup_{Y\in\mathcal E_Y}\mathbb P_\xi\!\left(
\|W_C-\widetilde W\|_{\op}>CL^{\alpha+1}\sqrt{\frac nd}
\ \middle|\ Y\right)\to0,
\]
and
\[
\sup_{Y\in\mathcal E_Y}\mathbb P_\xi\!\left(
\max_{a\le d}\|(W_C)_{a,:}\|_2>CL^{\alpha+1}\sqrt{\frac nd}
\ \middle|\ Y\right)\to0.
\]
Moreover,
\[
\sup_{Y\in\mathcal E_Y}\sup_{i\le N}\mathbb P_\xi\!\left(
\left\|m_i(W_C)-\frac{s_1}{d}
K(K^\top K)^{-1/2}J_Yy_i\right\|_\infty
>C\frac{\sqrt{ns_2}}dL^{\alpha+2}
\ \middle|\ Y\right)\to0.
\]
\end{appxlemma}

\begin{proof}
We give the concentration, polar perturbation, and conditional-mean
steps separately. Throughout the proof $Y\in\mathcal E_Y$ is fixed.

\paragraph{Concentration of the proxy Gram matrix.}
Write
\[
Z_a=\frac{C_Y^{-1/2}\xi_a\xi_a^\top C_Y^{-1/2}}
 {\xi_a^\top K^\top K\xi_a},\qquad
\widetilde W^\top\widetilde W=\frac1d\sum_{a=1}^d Z_a,
\qquad \mathbb E_\xi Z_a=I.
\]
Choose $q=\lceil AL\rceil$ with $A$ a sufficiently large fixed constant;
in Lemma~\ref{appx:lem:lam-query-event}, take $A_*>2A$.
In an eigenbasis of $K^\top K$,
\[
\xi_a^\top K^\top K\xi_a
 \ge cq^{-2\alpha}\sum_{\ell=1}^q\xi_{a,\ell}^2.
\]
Let $\mathcal T_a$ be the event
$\|\xi_a\|^2\le Cn$ and $\sum_{\ell\le q}\xi_{a,\ell}^2\ge q/2$.
By choosing the fixed constants sufficiently large, Gaussian and
chi-square tails give
\[
\sup_{Y\in\mathcal E_Y}\mathbb P_\xi(\mathcal T_a^c\mid Y)
 \le e^{-20L}
\]
for all sufficiently large $n$. On $\mathcal T_a$,
\[
\xi_a^\top K^\top K\xi_a\ge cL^{1-2\alpha},\qquad
0\preceq Z_a\preceq CnL^{2\alpha-1}I.
\]
For $\widehat Z_a=Z_a\mathbf1_{\mathcal T_a}$, positivity gives
$\mathbb E\widehat Z_a\preceq I$ and
\[
\mathbb E\widehat Z_a^2
 \preceq CnL^{2\alpha-1}\mathbb E\widehat Z_a
 \preceq CnL^{2\alpha-1}I.
\]
Thus the centered truncated sum has variance parameter at most
$CdnL^{2\alpha-1}$ and summand norm at most $CnL^{2\alpha-1}$.
The inverse moments in Lemma~\ref{appx:lem:lam-response-moments} also give
\[
\mathbb E\|Z_a\|_{\op}^2
\le C\mathbb E\!\left[
\frac{\|\xi_a\|^4}{(\xi_a^\top K^\top K\xi_a)^2}\right]
\le Cn^2,
\]
so Cauchy--Schwarz controls the truncation bias:
\[
\|\mathbb E\widehat Z_a-I\|_{\op}\le Cn e^{-10L}.
\]
Apply matrix Bernstein to both signs of the centered truncated sum
\citep[Theorem~6.1]{tropp2012user}. Together with
$\mathbb P(\bigcup_{a\le d}\mathcal T_a^c\mid Y)\le de^{-20L}$, it gives
\[
\|\widetilde W^\top\widetilde W-I\|_{\op}
\le C\left(L^\alpha\sqrt{\frac nd}+L^{2\alpha}\frac nd\right)
\le CL^\alpha\sqrt{\frac nd}=:e_n
\]
with conditional probability tending to one uniformly in $Y$.
The last inequality uses $L^\alpha\sqrt{n/d}\to0$; the truncation bias
is negligible on this scale.

\paragraph{A relative polar estimate.}
For positive definite matrices $M,P$ with
$\delta=\|P-I\|_{\op}\le1/2$, we claim
\begin{equation}
\|\operatorname{pol}(PM)-I\|_{\op}
\le C\bigl(1+\log\operatorname{cond}(M)\bigr)\delta,
\label{eq:lam-relative-polar}
\end{equation}
where $\operatorname{cond}(M)=\lambda_{\max}(M)/\lambda_{\min}(M)$.
Here is a direct proof of the spectral-norm estimate.
Set $E=P-I$ and write the polar decomposition
\[
(I+tE)M=U(t)H(t),\qquad 0\le t\le1,
\]
where $U(t)$ is orthogonal and $H(t)$ is positive definite.
The factors are differentiable along this nonsingular path.
Put $\Omega(t)=U(t)^\top\dot U(t)$ and
\[
D(t)=U(t)^\top E(I+tE)^{-1}U(t).
\]
The matrix $D(t)$ is symmetric, because $E$ commutes with $I+tE$,
and $\|D(t)\|_{\op}\le2\delta$.
Differentiating the polar decomposition and subtracting its transpose
gives the Sylvester equation
\[
\Omega H+H\Omega=DH-HD.
\]
Its integral solution is
\[
\Omega=\int_0^\infty e^{-sH}(DH-HD)e^{-sH}\,ds.
\]
If $\mu=\lambda_{\min}(H)$ and $\Lambda=\lambda_{\max}(H)$, then
\[
\|\Omega\|_{\op}
\le2\|D\|_{\op}\int_0^\infty
 \|e^{-sH}\|_{\op}\|He^{-sH}\|_{\op}\,ds
\le C\|D\|_{\op}\bigl(1+\log(\Lambda/\mu)\bigr).
\]
For the last inequality, split the integral at $1/\Lambda$ and $1/\mu$:
on the first interval use $\|He^{-sH}\|_{\op}\le\Lambda$;
on the middle interval use $\|He^{-sH}\|_{\op}\le1/(es)$;
on the last interval the product is at most $\mu e^{-2s\mu}$.
Moreover,
\[
\frac{\Lambda}{\mu}
=\operatorname{cond}((I+tE)M)
\le\frac{1+\delta}{1-\delta}\operatorname{cond}(M)
\le3\operatorname{cond}(M).
\]
Since $U(0)=I$ and $\|\dot U(t)\|_{\op}=\|\Omega(t)\|_{\op}$,
integration over $t\in[0,1]$ proves \eqref{eq:lam-relative-polar}.

\paragraph{Application to $W_C$ and the row bound.}
The factorization
\[
\frac{\RN(G)}{\sqrt d}
=\widetilde W\,C_Y^{1/2}(K^\top K)^{1/2}
\]
is exact. On the proxy-Gram event, set
\[
P=(\widetilde W^\top\widetilde W)^{1/2},\qquad
V=\operatorname{pol}(\widetilde W),\qquad
M=C_Y^{1/2}(K^\top K)^{1/2}.
\]
Then $\widetilde W=VP$ and $V^\top V=I$.
By Lemma~\ref{appx:lem:lam-response-moments}, the two factors in $M$
commute, so $M$ is symmetric positive definite. Its eigenvalues lie
between $cn^{-\alpha}$ and $C$, and thus
$1+\log\operatorname{cond}(M)\le CL$.
Left multiplication by an isometry preserves the polar factorization,
so
\[
W_C=V\operatorname{pol}(PM),\qquad
\|P-I\|_{\op}\le Ce_n.
\]
Applying \eqref{eq:lam-relative-polar} yields
\[
\begin{split}
\|W_C-\widetilde W\|_{\op}
&=\|\operatorname{pol}(PM)-P\|_{\op}\\
&\le\|\operatorname{pol}(PM)-I\|_{\op}+\|P-I\|_{\op}\\
&\le CLe_n\le CL^{\alpha+1}\sqrt{\frac nd}.
\end{split}
\]
For the row bound, all $\mathcal T_a$ hold with conditional probability
at least $1-de^{-20L}$. On these events,
\[
\max_{a\le d}\|\widetilde W_{a,:}\|_2
\le CL^{\alpha-1/2}\sqrt{\frac nd}.
\]
The square matrix $\operatorname{pol}(PM)$ is orthogonal and
$V=\widetilde WP^{-1}$, with $\|P^{-1}\|_{\op}\le2$ for large $n$.
Therefore
\[
\|(W_C)_{a,:}\|_2=\|V_{a,:}\|_2
\le2\|\widetilde W_{a,:}\|_2
\le CL^{\alpha-1/2}\sqrt{\frac nd},
\]
which is stronger than the stated row bound.

\paragraph{Conditional-mean concentration.}
Fix a deterministic $i$ and write $b=C_Y^{-1/2}y_i$, so $\|b\|\le C$.
The contribution from row $a$ to coordinate $j$ of $m_i(\widetilde W)$ is
\[
\frac{\sigma_a}{d}f_{aj},\qquad
f_{aj}=\frac{(t_j^\top\xi_a)(b^\top\xi_a)}
 {\sqrt{\xi_a^\top K^\top K\xi_a}}.
\]
Its expectation is
$\mathbb E f_{aj}=t_j^\top H_YC_Y^{-1/2}y_i=t_j^\top J_Yy_i$.
Gaussian moments, $\|t_j\|\le1$, and the inverse-moment bound imply
\[
\mathbb E f_{aj}^2\le C
\]
uniformly in $a,j,i,Y$.
For each row, refine $\mathcal T_a$ by also requiring
\[
\max_{j\le N}|t_j^\top\xi_a|\le C\sqrt L,
\qquad |b^\top\xi_a|\le C\sqrt L.
\]
Gaussian tails and a union bound over $j$ allow the complement of this
refined event to have probability at most $e^{-20L}$, after increasing
fixed constants. On the refined event,
$|f_{aj}|\le CL^{\alpha+1/2}$.
The truncated variables remain independent across $a$; their total
weighted variance is at most $Cs_2/d^2$, and their centered summand
bound is $C\max_a\sigma_a L^{\alpha+1/2}/d$.
Their total bias is at most $Cs_1e^{-10L}/d$, which is negligible since
$s_1\le\sqrt{ds_2}$.
Scalar Bernstein and a union bound over $j\le N$ now give
\[
\begin{split}
\max_{j\le N}\left|e_j^\top m_i(\widetilde W)
 -\frac{s_1}{d}t_j^\top J_Yy_i\right|
&\le\frac Cd\left(\sqrt{s_2L}
 +\max_{a\le d}\sigma_a L^{\alpha+3/2}\right)\\
&\le C\frac{\sqrt{s_2}}dL^{\alpha+3/2}
\end{split}
\]
with conditional failure probability tending to zero. The constants
depend on $i$ only through $\|y_i\|\le2$, so the same failure bound
holds uniformly over deterministic $i\le N$.

Finally, the replacement error obeys
\[
|e_j^\top(m_i(W_C)-m_i(\widetilde W))|
\le\|G(K^\top K)^{-1/2}t_j\|_2\,
\|W_C-\widetilde W\|_{\op}\,\|y_i\|_2.
\]
Here
\[
\|G(K^\top K)^{-1/2}t_j\|_2^2
=\frac1d\sum_{a\le d}\sigma_a^2(\xi_a^\top t_j)^2.
\]
The weighted chi-square moment-generating-function bound, together with
$\|t_j\|\le1$, $\sum_a\sigma_a^4\le s_2^2$, and a union bound over $j$,
gives
\[
\max_{j\le N}\|G(K^\top K)^{-1/2}t_j\|_2
\le C\sqrt{\frac{s_2L}{d}}
\]
with conditional probability tending to one uniformly in $Y$.
Combining this with the operator-norm approximation and the proxy-mean
bound yields
\[
\left\|m_i(W_C)-\frac{s_1}{d}K(K^\top K)^{-1/2}J_Yy_i\right\|_\infty
\le C\frac{\sqrt{ns_2}}dL^{\alpha+3/2}
\le C\frac{\sqrt{ns_2}}dL^{\alpha+2},
\]
which proves the last assertion.
\end{proof}

\subsection{Condition on \((Y,G)\): residual score randomness}

Once $(Y,G)$ is fixed, any update $W$ measurable with respect to $(Y,G)$,
its conditional mean, and all query-side quantities are fixed.
The remaining randomness is the component of $X$ not revealed by
$G=\Sigma XK$. Define the retrieval score
\[
S_{ji}(W)=x_j^\top\Sigma Wy_i.
\]

\begin{appxlemma}[Exact conditional score law]
\label{appx:lem:lam-score-law}
Let $W$ be any measurable function of $(Y,G)$. For a fixed query $y_i$,
write $\mathbf S_i(W)=(S_{1i}(W),\ldots,S_{Ni}(W))^\top$.
Conditional on $(Y,G)$,
\[
\mathbf S_i(W)\sim N\!\left(
m_i(W),\frac{\|\Sigma Wy_i\|_2^2}{d}
\bigl[I_N-K(K^\top K)^{-1}K^\top\bigr]\right),
\]
where $m_i(W)=K(K^\top K)^{-1}G^\top Wy_i$.
\end{appxlemma}

\begin{proof}
Conditional on $Y$, put $U=K(K^\top K)^{-1/2}$.
For each answer-coordinate row,
\[
X_{a,:}^\top=\frac1{\sqrt d}U\xi_a+x_{a,\perp},\qquad
x_{a,\perp}=(I_N-UU^\top)X_{a,:}^\top.
\]
These two Gaussian components are independent, and the residual rows
are mutually independent. Given $(Y,G)$, the first components are fixed
because
\[
\xi_a=\sqrt d\,\sigma_a^{-1}(K^\top K)^{-1/2}G_{a,:}^\top.
\]
The residual rows retain mean zero and covariance
$d^{-1}(I_N-UU^\top)$.
Substituting their decomposition into
$\mathbf S_i(W)=X^\top\Sigma Wy_i$ gives the stated covariance.
For the mean, the observed component satisfies
\[
\frac1{\sqrt d}\sum_{a=1}^d\sigma_a U\xi_a(Wy_i)_a
=K(K^\top K)^{-1}G^\top Wy_i,
\]
which is the claimed formula.
\end{proof}

\begin{appxlemma}[Wrong-score maximum]
\label{appx:lem:lam-wrong-max}
Suppose $n\to\infty$ and $N/n\to\infty$.
Let $P$ be any rank-$n$ orthogonal projection in $\mathbb R^N$,
let $a\in\mathbb R^N$ satisfy $\|a\|_2\le C\sqrt n$ for a fixed $C$,
and let $z\sim N(0,I_N-P)$.
Put $\tau_n=\frac18\sqrt{\log(N/n)}$.
There is a deterministic sequence $\varepsilon_n\to0$, depending only
on $n,N,C$, such that for every $P,a,i$,
\[
\mathbb P\!\left(a_i+z_i>\max_{j\ne i}(a_j+z_j)\right)
\le\mathbf1\{|a_i|>\tau_n\}+\varepsilon_n.
\]
\end{appxlemma}

\begin{proof}
Write $z=(I_N-P)g$, where $g\sim N(0,I_N)$, and set $w=a-Pg$.
Then $a+z=g+w$ and, since $\|Pg\|^2\sim\chi_n^2$,
$\|w\|^2\le C'n$ except with probability $e^{-cn}$.
Let $r=\sqrt{\log(N/n)}$ and
$B_N=\#\{j:g_j>r\}$. For all sufficiently large $n$, $r\ge1$ and
\[
\mathbb E B_N=N\mathbb P(g_1>r)\ge c\frac{\sqrt{Nn}}r.
\]
A Chernoff bound gives $B_N\ge\mathbb EB_N/2$ except with probability
$\exp(-c\sqrt{Nn}/r)$. On this event, if every coordinate of $g+w$
were below $r/2$, then
\[
\|w\|^2\ge\frac{r^2}4B_N\ge cr\sqrt{Nn}>C'n
\]
for all sufficiently large $n$, a contradiction.
Hence $\max_j(a_j+z_j)\ge r/2$ except on the two events just bounded.
If $|a_i|\le r/8$, then $a_i+z_i\le r/4$ except with probability at
most $e^{-r^2/128}$, because $\operatorname{Var}(z_i)\le1$.
Thus coordinate $i$ cannot be the maximum on the intersection of these
events. One may take, for sufficiently large $n$,
\[
\varepsilon_n=e^{-cn}+\exp(-c\sqrt{Nn}/r)+e^{-r^2/128}\to0.
\]
The argument does not assume independence of the coordinates of $z$,
or independence between $g$ and $w$.
\end{proof}

\begin{appxtheorem}[\Cref{thm:lam-separation} of the paper]
\label{appx:thm:lam-separation}
Fix \(\alpha>1\), \(1/2<\beta<1\), and a fixed deterministic positive sequence satisfying
\[
c_\sigma a^{-\beta}\le \sigma_a\le C_\sigma a^{-\beta}\,(a\ge1)
\]
for constants \(0<c_\sigma\le C_\sigma<\infty\). Fix \(0<\epsilon<(\beta-1/2)/2\) and \(\Gamma>1/(1-\beta)\). Let \(n\to\infty\), and let \(N=N(n)\) and \(d=d(n)\) be integer sequences satisfying, eventually,
\[
n^{\,1+(1-\beta+\epsilon)/\alpha}
\le N\le
n^{\,1+(1/2-\epsilon)/\alpha},
\,
n^{\,1/(1-\beta)}\le d\le n^\Gamma.
\]
Then, on the same realized dataset,
\[
\bigl(\Acc(W_C),\Acc(W_M),\Acc(W_G)\bigr)
\xrightarrow{\mathbb P}(1,0,0).
\]
\end{appxtheorem}
\begin{proof}
Let $\mathcal E_X$ denote
$\{\|\mathcal H\mathcal J^{-1}\mathcal H\|_{\op}\le Cn\}$.
Lemma~\ref{appx:lem:lam-query-event} and
Lemma~\ref{appx:lem:lam-baseline-X} give
\[
\mathbb P(\mathcal E_Y^c)\to0,\qquad
\sup_{Y\in\mathcal E_Y}\mathbb P_\xi(\mathcal E_X^c\mid Y)\to0.
\]
Both $W_G$ and $W_M$ have their columns in $\col(G)$.
On $\mathcal E_Y\cap\mathcal E_X$, the exact score law and
Corollary~\ref{appx:cor:lam-mean-energy} show that, for a nonzero query,
the scores normalized by $\|\Sigma Wy_i\|/\sqrt d$ have the law
\[
a^{(i)}(W)+z,\qquad \|a^{(i)}(W)\|_2\le C\sqrt n,\qquad
z\sim N\!\left(0,I_N-K(K^\top K)^{-1}K^\top\right).
\]
The query is zero with probability zero; under the zero-denominator
convention, such a query is not recovered in any case.
The projection $K(K^\top K)^{-1}K^\top$ has rank $n$.

For GD, Lemma~\ref{appx:lem:lam-wrong-max} bounds conditional recovery
probability by
$\mathbf1\{|a_i^{(i)}(W_G)|>\tau_n\}+\varepsilon_n$ on the good events.
For any fixed $u>0$, eventually $\tau_n\ge u$.
Lemma~\ref{appx:lem:lam-gd-correct} therefore implies, for every fixed
$b\in(0,1)$,
\[
\sup_{i\ge bN}\mathbb P\{i\text{ is recovered by }W_G\}\to0.
\]
Consequently
\[
\mathbb E[\Acc(W_G)]
\le b+\frac1N+
\sup_{i\ge bN}\mathbb P\{i\text{ is recovered by }W_G\}.
\]
Send $n\to\infty$ and then $b\downarrow0$. Markov's inequality gives
$\Acc(W_G)\xrightarrow{\mathbb P}0$.

For Muon, for any fixed $u>0$ and eventually $\tau_n\ge u$,
\[
\frac1N\#\{i:|a_i^{(i)}(W_M)|>\tau_n\}
\le\frac1N\#\{i:|a_i^{(i)}(W_M)|>u\}
\xrightarrow{\mathbb P}0
\]
by Lemma~\ref{appx:lem:lam-muon-correct} and
$\mathbb P(\mathcal E_Y^c)\to0$.
This fraction lies in $[0,1]$, so its expectation also tends to zero.
Averaging Lemma~\ref{appx:lem:lam-wrong-max} over $i$ on
$\mathcal E_Y\cap\mathcal E_X$, and adding the probabilities of the
complementary events, yields
\[
\mathbb E[\Acc(W_M)]\to0,\qquad
\Acc(W_M)\xrightarrow{\mathbb P}0.
\]

For the chain, fix $b\in(0,1)$ and a deterministic $i\ge bN$.
Lemma~\ref{appx:lem:lam-population-margin} implies that, conditional on
$\mathcal E_Y$, with probability tending to one uniformly over such $i$,
\[
t_i^\top J_Yy_i\ge c_bh,\qquad
\max_{j\ne i}|t_j^\top J_Yy_i|\le\frac{c_b}4h.
\]
Fix a query realization satisfying this margin.
Lemma~\ref{appx:lem:lam-chain-X} then implies
\[
m_{ii}(W_C)-\max_{j\ne i}m_{ji}(W_C)
\ge\frac{c_b}2\frac{s_1}{d}h
\]
except on an observed-gradient event whose conditional probability
tends to zero uniformly in $Y$ and deterministic $i$.
Indeed, the ratio of its mean-approximation error to $(s_1/d)h$ is
\[
\frac{\sqrt{ns_2}\,L^{\alpha+2}}{s_1h}
\le Cn^{-\epsilon}L^{\alpha+2}\to0.
\]
Intersect this event with the row-norm event in the same lemma, whose
conditional failure probability also tends to zero uniformly in $Y$.
On that intersection,
\[
\frac{\|\Sigma W_Cy_i\|_2}{\sqrt d}
\le\frac{\|y_i\|_2}{\sqrt d}
 \left(\sum_{a\le d}\sigma_a^2\|(W_C)_{a,:}\|_2^2\right)^{1/2}
\le C\frac{\sqrt{ns_2}}dL^{\alpha+1}.
\]
Now condition on $(Y,G)$.
Lemma~\ref{appx:lem:lam-score-law} shows that every centered score is
Gaussian with variance bounded by the square of the last display.
A union bound over $j\le N$ gives
\[
\mathbb P\!\left(
\max_{j\le N}|S_{ji}(W_C)-m_{ji}(W_C)|
>C\frac{\sqrt{ns_2}}dL^{\alpha+3/2}
\ \middle|\ Y,G\right)\to0
\]
uniformly over the stated good events.
The fluctuation threshold is $o((s_1/d)h)$ by the same scaling
calculation. Therefore the correct score is strictly greater than all
wrong scores with probability tending to one.
Combining the query, observed-gradient, and residual-score bounds yields
\[
\sup_{i\ge bN}\mathbb P\{i\text{ is not recovered by }W_C\}\to0.
\]
This is a bound uniform over deterministic $i$, not a claim that a
single union bound recovers every query.
It suffices because
\[
\mathbb E[1-\Acc(W_C)]
\le b+\frac1N+
\sup_{i\ge bN}\mathbb P\{i\text{ is not recovered by }W_C\}.
\]
Send $n\to\infty$ and then $b\downarrow0$, and use Markov's inequality
to obtain $\Acc(W_C)\xrightarrow{\mathbb P}1$.

All three updates are defined on the same realized dataset.
A union bound over their three accuracy deviation events proves
\[
\bigl(\Acc(W_C),\Acc(W_M),\Acc(W_G)\bigr)
\xrightarrow{\mathbb P}(1,0,0).
\]
Equivalently, for every fixed $\delta\in(0,1)$,
\[
\mathbb P\!\left\{\Acc(W_C)\ge1-\delta,\quad
\max\{\Acc(W_M),\Acc(W_G)\}\le\delta\right\}\to1.
\]
\end{proof}

\newpage
\section{Appendix: Proofs in \cref{sec:nav}}
\label[appendix]{sec:tcproof}

\begin{proof}[Proof of \Cref{prop:tensorchain} of the paper]
Write
\[
\mathcal T
=
\mathcal A^{\circ k}
\circ \mathrm{NS}_{\mathrm{slice}}
\circ \mathcal A^{\circ k},
\qquad
\mathcal A
=
\mathrm{AN}_{\mathrm{layer}}
\circ \mathrm{AN}_{\mathrm{col}}
\circ \mathrm{AN}_{\mathrm{row}} .
\]

We first verify the chained-LMO representation.
For an axis \(a\in\{\mathrm{layer},\mathrm{row},\mathrm{col}\}\), let
\(\mathcal F_a\) denote the fibers along that axis and define
\[
\mathcal K_a
=
\left\{
Z:\|Z_f\|_2\le 1
\ \text{for every }f\in\mathcal F_a
\right\}.
\]
Since the Frobenius inner product decomposes over fibers,
\[
\max_{Z\in\mathcal K_a}\langle M,Z\rangle
=
\sum_{f\in\mathcal F_a}\|M_f\|_2,
\]
and a maximizer is obtained by normalizing every nonzero fiber independently.
Hence \(\mathrm{AN}_a\) is, up to the conventional sign, an exact LMO over
\(\mathcal K_a\).

Similarly, with
\[
\mathcal K_{\mathrm{slice}}
=
\left\{
Z:\|Z_{\ell,:,:}\|_{\mathrm{op}}\le 1
\ \text{for every slice }\ell
\right\},
\]
we have
\[
\max_{Z\in\mathcal K_{\mathrm{slice}}}\langle M,Z\rangle
=
\sum_\ell \|M_{\ell,:,:}\|_*,
\]
whose maximizer is obtained by applying the matrix sign independently to
each slice. Thus every constituent of \(\mathcal T\) is an exact LMO, and
\(\mathcal T\) is a chained LMO.

It remains to show that the composition is not itself an LMO. For
\(G\in\mathbb R^{m\times n}\), let
\[
\iota(G)_{1,:,:}=G,
\qquad
\iota(G)_{\ell,:,:}=0
\quad (\ell>1).
\]
Row- and column-axis normalization preserve the signs and zero pattern of
the nonzero slice. The subsequent layer-axis normalization acts on fibers
of the form \((c,0,\ldots,0)\), and therefore
\[
\mathcal A(\iota(G))
=
\iota(\mathrm{Sgn}(G)).
\]
Consequently, for every \(k\ge1\),
\[
\mathcal A^{\circ k}(\iota(G))
=
\iota(\mathrm{Sgn}(G)),
\]
and hence
\[
\mathcal T(\iota(G))
=
\iota\!\left(
\mathrm{Sgn}\circ\mathrm{msign}\circ\mathrm{Sgn}(G)
\right).
\]

By \Cref{appx:lem:mixed-sign}, for every \(m,n\ge5\) there exists
\(G\in\mathbb R^{m\times n}\) such that
\[
\left\langle
G,
\mathrm{Sgn}\circ\mathrm{msign}\circ\mathrm{Sgn}(G)
\right\rangle
<0.
\]
Hence
\[
\langle \iota(G),\mathcal T(\iota(G))\rangle<0.
\]
Thus TensorChain admits a negative-alignment input, and by \Cref{lem:alignment}, it is not a LMO of any fixed set. A smooth counterexample where TensorChain optimization diverges can be found by using \Cref{prop:divergence}.
\end{proof}

\newpage
\section{Appendix: Replacing NS with Cheaper Chains}
\label[appendix]{app:ns_replacement}

\begin{figure}[htbp!]
\centering
\begin{subfigure}[b]{0.47\linewidth}
    \centering
    \includegraphics[width=\linewidth]{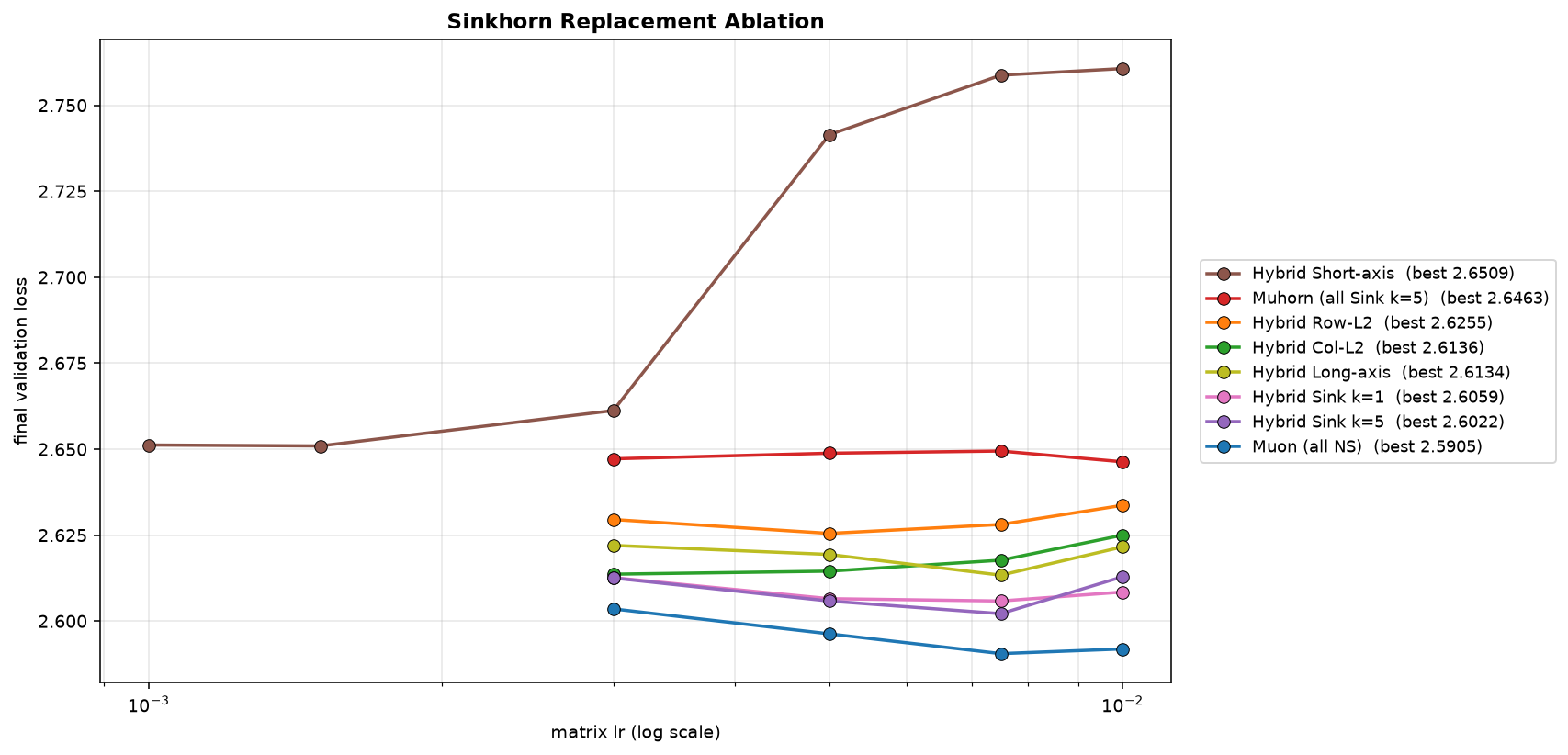}
    \caption{FFN replacement ablation.}
    \label{fig:ns_replace_ffn}
\end{subfigure}\hfill
\begin{subfigure}[b]{0.47\linewidth}
    \centering
    \includegraphics[width=\linewidth]{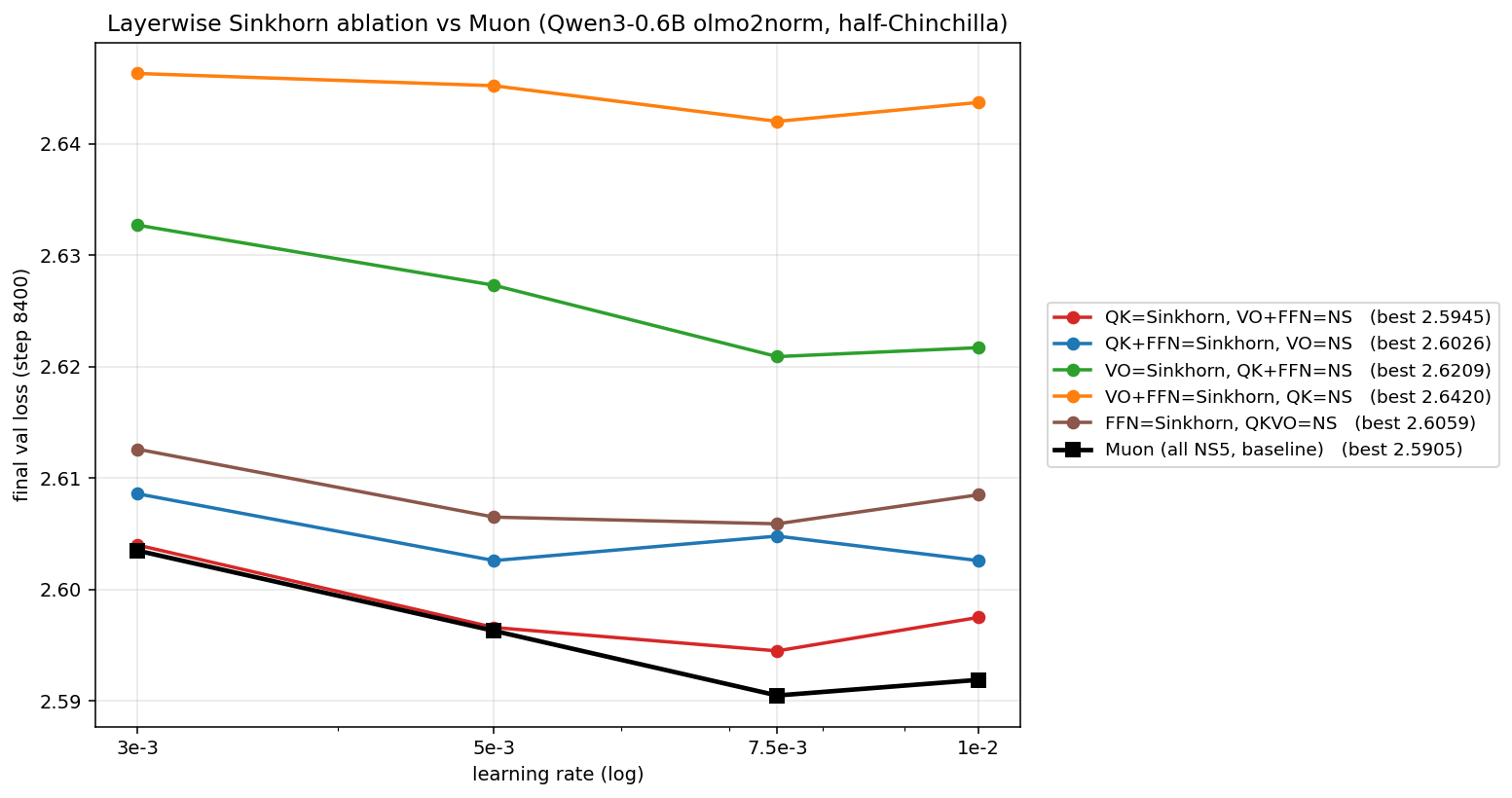}
    \caption{Layerwise ablation.}
    \label{fig:ns_replace_weights}
\end{subfigure}
\caption{
Replacing Newton--Schulz with cheaper axis-wise normalizations on Qwen3-0.6B after 8.8B training tokens.
(a) NS5 is retained on $Q,K,V,O$ while the FFN uses the indicated replacement.
(b) A single row--column equilibration sweep is applied to the indicated parameter classes, with NS5 retained elsewhere.
}
\label{fig:ns_replacement}
\end{figure}

\paragraph{Experimental setup.}
We study whether inexpensive axis-wise normalizations can replace the
Newton--Schulz (NS) step of Muon on the \texttt{olmo2norm} Qwen3-0.6B
architecture. We use QK normalization, unfused QKV projections, Kaiming
initialization, tied input/output embeddings, variable-length attention,
selective activation checkpointing, and compilation.
Training and validation use \texttt{fineweb\_edu\_350bt} and
\texttt{fineweb\_edu\_350bt\_eval}, respectively.
We train for 8.8B tokens ($1\times$ Chinchilla) with sequence length 4096,
global batch size 256 (1.05M tokens), and 8401 optimization steps.
All runs use one node with 8 GPUs under replicated data parallelism and
seed 42; validation is performed every 200 steps on one batch.
Checkpoints are retained and training metrics are logged to TensorBoard.

We use 10\% learning-rate warmup followed by linear decay to zero, matrix
weight decay $0.1$, momentum coefficient $\beta=0.95$, and AdamW learning
rate $8\times10^{-3}$ for non-matrix parameters.
All matrix updates use the same \texttt{s\_mu\_scale=max} scaling,
$0.2\sqrt{\max(\mathrm{fan\_in},\mathrm{fan\_out})}$.
Muon uses five Newton--Schulz iterations (NS5).
For each variant, we sweep the matrix learning rate over
$\{3,5,7.5,10\}\times10^{-3}$ and report the best validation loss.
For short-axis normalization, whose best point lies at the lower boundary of
the sweep, we additionally evaluate $1\times10^{-3}$ and
$1.5\times10^{-3}$.

\paragraph{Replacing NS on the FFN.}
Our primary motivation is computational: the large FFN matrices account for
the largest share of the NS computation, making them a natural target for a
cheaper replacement.
We therefore retain NS5 on the attention projections $Q,K,V,O$ and replace
NS on the FFN matrices $W_1,W_2,W_3$.

We compare row and column $\ell_2$ normalization, normalization along the
long or short matrix axis, and alternating row--column equilibration.
For the latter, we test both a single sweep ($k=1$) and five repeated sweeps
($k=5$).
We additionally evaluate row normalization and repeated row--column
equilibration as replacements for NS on all matrix parameters.
Among the FFN replacements, row--column equilibration performs best, while
increasing from one to five sweeps yields only a small improvement.
We therefore fix $k=1$ for the subsequent weight-wise ablation.

\paragraph{Which weights can use the cheaper replacement?}
Having identified a single row--column sweep as the strongest inexpensive
replacement, we next vary which parameter classes use it while retaining
NS5 on all remaining matrices.
We consider the FFN alone, $Q,K$ alone, $V,O$ alone, and combinations of
these groups.
The strongest configuration replaces NS only on $Q,K$, reaching validation
loss $2.5945$ compared with $2.5905$ for Muon.
Replacing additional parameter classes consistently degrades performance,
and no configuration outperforms Muon.

\begin{table}[t]
\centering
\small
\caption{
Selected Newton--Schulz replacement results on Qwen3-0.6B after
8.8B training tokens. Each entry reports the best result over the matrix
learning-rate sweep. $\Delta$ denotes the difference from Muon.
}
\label{tab:ns_replacement}
\begin{tabular}{lcc}
\toprule
Variant & Val.\ loss & $\Delta$ \\
\midrule
Muon (NS5 on all matrices) & 2.5905 & -- \\
\midrule
\multicolumn{3}{l}{\emph{NS5 on QKVO, replacement on FFN}}\\
Row--column, $k=5$ & 2.6022 & +0.012 \\
Row--column, $k=1$ & 2.6059 & +0.015 \\
Long-axis $\ell_2$ & 2.6134 & +0.023 \\
Column $\ell_2$ & 2.6136 & +0.023 \\
Row $\ell_2$ & 2.6255 & +0.035 \\
Short-axis $\ell_2$ & 2.6509 & +0.060 \\
\midrule
\multicolumn{3}{l}{\emph{Row--column ($k=1$) on listed weights, NS5 elsewhere}}\\
$Q,K$ & 2.5945 & +0.004 \\
$Q,K$, FFN & 2.6026 & +0.012 \\
FFN & 2.6059 & +0.015 \\
$V,O$ & 2.6209 & +0.030 \\
$V,O$, FFN & 2.6420 & +0.052 \\
All matrices (row--column, $k=5$) & 2.6463 & +0.056 \\
\bottomrule
\end{tabular}
\end{table}

Overall, cheap axis-wise normalization can approach NS when used selectively,
but none of the tested replacements matches Muon.
The small difference between one and five row--column sweeps also suggests
that simply iterating equilibration further does not close the gap to
spectral normalization.





\clearpage
\section{Appendix: Raw Experiment Results}
\label[appendix]{app:rawresult}

This section provides the complete experimental record underlying the comparisons in the main text. We report the full Qwen3-0.6B baseline sweeps under WSD and cosine schedules, the Muon learning-rate sweeps used for tuning, the full train/validation trajectories and endpoint losses, the TensorChain grouping ablation, and the cosine token-efficiency analysis. Unless otherwise stated, the experimental configuration follows \Cref{sec:setup}. We retain the cosine results in addition to the main WSD experiments to document schedule-dependent differences in optimizer ranking and learning-rate sensitivity.

\subsection{Full Qwen3-0.6B baseline comparison}

The two tables below report the complete common learning-rate grids for the Qwen3-0.6B, 1M-token-batch, $1\times$ Chinchilla baseline comparison. TensorChain grouping variants are reported separately in \Cref{app:grouping-results}. For AdaMuon, the learning rate grid was insufficient so we ran two additional experiments, $lr = 27e-3$ and $40e-3$ and could find that the optimal learning rate is $27e-3$ with val loss 2.8405, which is strictly worse than TensorChain's optimal val loss 2.8379.
\begin{table}[H]
\centering
\scriptsize
\caption{Full Qwen3-0.6B, 1M-token-batch, $1\times$ Chinchilla baseline sweep under WSD. Entries are final validation losses; methods are ordered by their best validation loss, and bold marks the best learning rate for each method.}
\label{tab:raw-full-wsd}
\setlength{\tabcolsep}{3.2pt}
\resizebox{\textwidth}{!}{%
\begin{tabular}{lccccc}
\toprule
Method & $1\times10^{-3}$ & $2\times10^{-3}$ & $3\times10^{-3}$ & $5\times10^{-3}$ & $8\times10^{-3}$ \\
\midrule
TensorChain     & 2.9089 & 2.9002 & \textbf{2.9000} & 2.9040 & 2.9131 \\
Muon$^2$        & 2.9171 & 2.9106 & \textbf{2.9104} & 2.9157 & 2.9260 \\
Aurora-LN       & 2.9204 & 2.9128 & \textbf{2.9116} & 2.9130 & 2.9225 \\
AdaMuon         & 2.9290 & 2.9156 & \textbf{2.9142} & 2.9188 & 2.9294 \\
MuonEq          & 2.9235 & 2.9163 & \textbf{2.9151} & 2.9186 & 2.9278 \\
MuonEq-diag     & 2.9258 & 2.9184 & \textbf{2.9167} & 2.9211 & 2.9305 \\
Muon+           & 2.9279 & 2.9203 & \textbf{2.9180} & 2.9234 & 2.9334 \\
Aurora          & 2.9280 & 2.9196 & \textbf{2.9184} & 2.9238 & 2.9335 \\
Muon            & 2.9299 & 2.9224 & \textbf{2.9201} & 2.9251 & 2.9354 \\
NorMuon         & 2.9295 & 2.9214 & \textbf{2.9201} & 2.9256 & 2.9357 \\
\bottomrule
\end{tabular}
}
\end{table}
\begin{table}[H]
\centering
\scriptsize
\caption{Full Qwen3-0.6B, 1M-token-batch, $1\times$ Chinchilla baseline sweep under cosine decay. Entries are final validation losses; methods are ordered by their best validation loss, and bold marks the best learning rate for each method.}
\label{tab:raw-full-cos}
\setlength{\tabcolsep}{3.2pt}
\resizebox{\textwidth}{!}{%
\begin{tabular}{lccccc}
\toprule
Method & $5\times10^{-3}$ & $8\times10^{-3}$ & $12\times10^{-3}$ & $16\times10^{-3}$ & $20\times10^{-3}$ \\
\midrule
TensorChain     & 2.8501 & 2.8409 & \textbf{2.8379} & 2.8383 & 2.8399 \\
AdaMuon         & 2.8674 & 2.8609 & 2.8568 & 2.8519 & \textbf{2.8458} \\
Aurora-LN       & 2.8614 & 2.8520 & \textbf{2.8474} & 2.8476 & 2.8482 \\
Muon$^2$        & 2.8626 & 2.8543 & \textbf{2.8497} & 2.8498 & 2.8515 \\
MuonEq          & 2.8675 & 2.8578 & \textbf{2.8540} & 2.8548 & 2.8563 \\
MuonEq-diag     & 2.8708 & 2.8624 & 2.8572 & \textbf{2.8562} & 2.8567 \\
Aurora          & 2.8739 & 2.8658 & \textbf{2.8624} & 2.8627 & 2.8646 \\
Muon+           & 2.8739 & 2.8663 & 2.8631 & \textbf{2.8630} & 2.8684 \\
Muon            & 2.8756 & 2.8680 & \textbf{2.8639} & 2.8645 & 2.8659 \\
NorMuon         & 2.8767 & 2.8685 & \textbf{2.8650} & 2.8652 & 2.8657 \\
\bottomrule
\end{tabular}
}
\end{table}

\begin{figure}[H]
\centering
\begin{subfigure}[b]{0.48\linewidth}
    \centering
    \includegraphics[width=\linewidth]{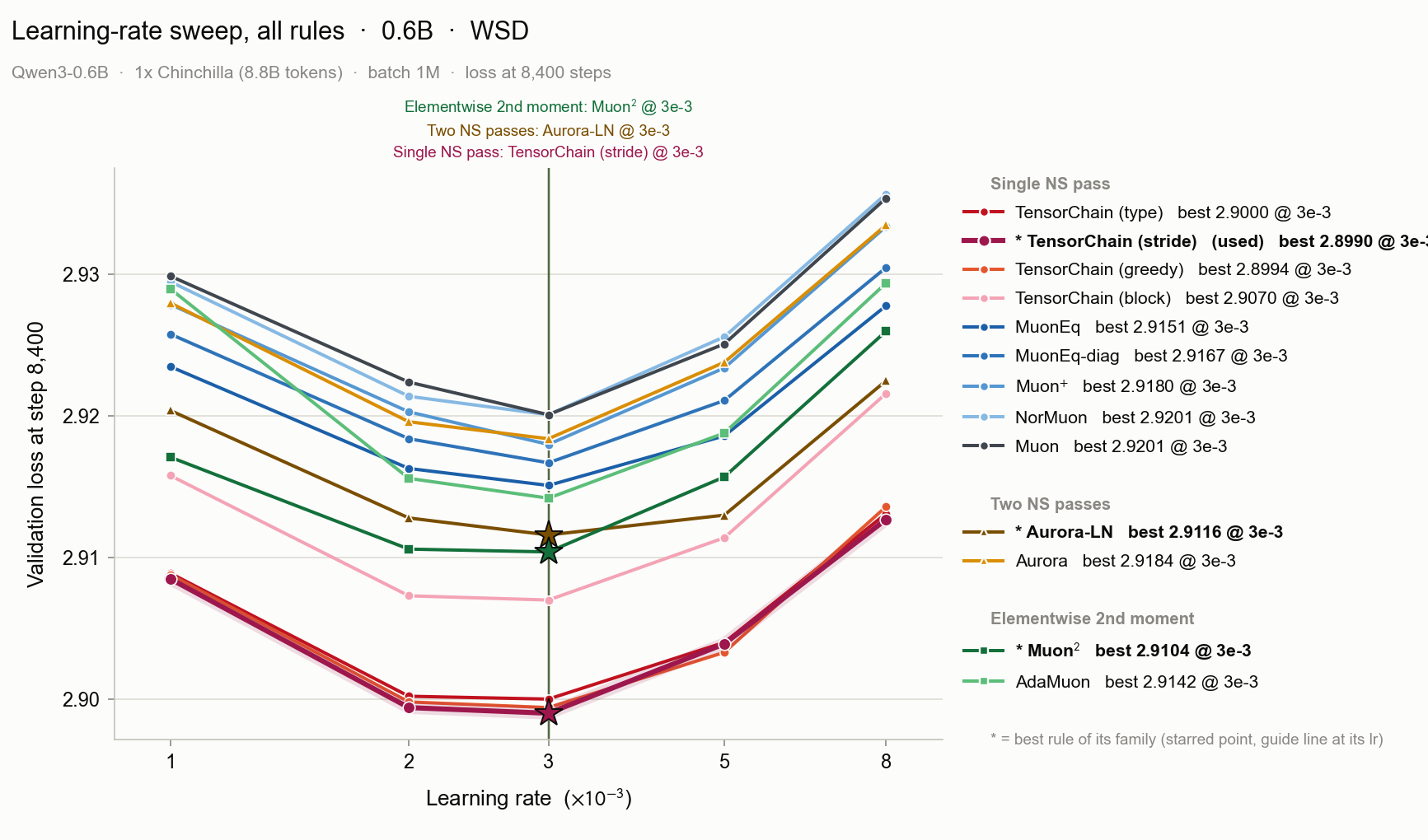}
    \caption{WSD learning rate tuning loss.}
\end{subfigure}\hfill
\begin{subfigure}[b]{0.48\linewidth}
    \centering
    \includegraphics[width=\linewidth]{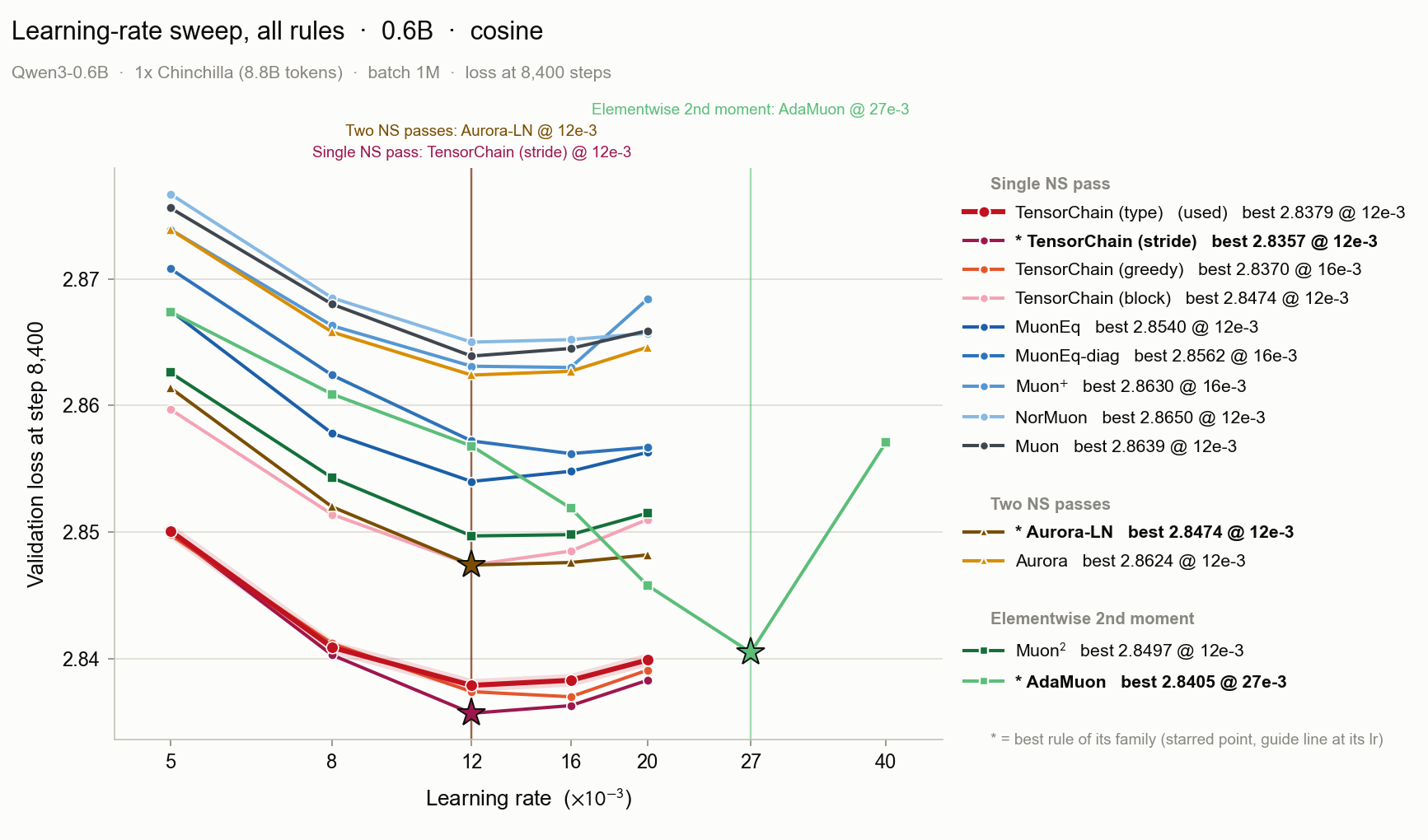}
    \caption{COS learning rate tuning}
\end{subfigure}
\caption{Full Qwen3-0.6B baseline comparison.}
\label{fig:raw-06b-full}
\end{figure}

\subsection{Muon learning-rate sweeps}

We report each tuning setup separately for Muon and use it for all baselines. The results are below. 
\begin{table}[H]
\centering
\small
\caption{Muon learning-rate sweep for Qwen3-0.6B, 1M-token batch, $1\times$ Chinchilla under WSD.}
\label{tab:lr-wsd-06b}
\setlength{\tabcolsep}{6pt}
\begin{tabular}{lccccc}
\toprule
Matrix LR
& $1\times10^{-3}$
& $2\times10^{-3}$
& $3\times10^{-3}$
& $5\times10^{-3}$
& $8\times10^{-3}$ \\
\midrule
Final val.\ loss
& 2.9299
& 2.9224
& \textbf{2.9201}
& 2.9251
& 2.9354 \\
\bottomrule
\end{tabular}
\end{table}

\begin{table}[H]
\centering
\small
\caption{Muon learning-rate sweep for Qwen3-1.7B, 4M-token batch, $1\times$ Chinchilla under WSD.}
\label{tab:lr-wsd-17b}
\setlength{\tabcolsep}{6pt}
\begin{tabular}{lccccc}
\toprule
Matrix LR
& $0.5\times10^{-3}$
& $1\times10^{-3}$
& $2\times10^{-3}$
& $4\times10^{-3}$
& $8\times10^{-3}$ \\
\midrule
Final val.\ loss
& 2.6878
& 2.6672
& 2.6523
& \textbf{2.6493}
& 2.6589 \\
\bottomrule
\end{tabular}
\end{table}

\begin{table}[H]
\centering
\small
\caption{Muon learning-rate sweep for Qwen3-0.6B, 1M-token batch, $1\times$ Chinchilla under cosine decay.}
\label{tab:lr-cos-06b-1m}
\setlength{\tabcolsep}{5pt}
\begin{tabular}{lcccccccc}
\toprule
Matrix LR
& $1\times10^{-3}$
& $2\times10^{-3}$
& $3\times10^{-3}$
& $5\times10^{-3}$
& $8\times10^{-3}$
& $12\times10^{-3}$
& $16\times10^{-3}$
& $20\times10^{-3}$ \\
\midrule
Final val.\ loss
& 2.9201
& 2.8993
& 2.8857
& 2.8756
& 2.8680
& \textbf{2.8639}
& 2.8645
& 2.8659 \\
\bottomrule
\end{tabular}
\end{table}

\begin{table}[H]
\centering
\small
\caption{Muon learning-rate sweep for Qwen3-0.6B, 4M-token batch, $1\times$ Chinchilla under cosine decay.}
\label{tab:lr-cos-06b-4m}
\setlength{\tabcolsep}{6pt}
\begin{tabular}{lcccccc}
\toprule
Matrix LR
& $8\times10^{-3}$
& $12\times10^{-3}$
& $18\times10^{-3}$
& $27\times10^{-3}$
& $40\times10^{-3}$
& $60\times10^{-3}$ \\
\midrule
Final val.\ loss
& 2.9650
& 2.9568
& 2.9516
& 2.9462
& \textbf{2.9459}
& 2.9468 \\
\bottomrule
\end{tabular}
\end{table}

\begin{table}[H]
\centering
\small
\caption{Muon learning-rate sweep for Qwen3-1.7B, 1M-token batch, $1\times$ Chinchilla under cosine decay.}
\label{tab:lr-cos-17b-1m}
\setlength{\tabcolsep}{6pt}
\begin{tabular}{lccccc}
\toprule
Matrix LR
& $4\times10^{-3}$
& $6\times10^{-3}$
& $9\times10^{-3}$
& $13.5\times10^{-3}$
& $20\times10^{-3}$ \\
\midrule
Final val.\ loss
& 2.6019
& \textbf{2.5974}
& 2.5987
& 2.6042
& 2.6139 \\
\bottomrule
\end{tabular}
\end{table}

\begin{table}[H]
\centering
\small
\caption{Muon learning-rate sweep for Qwen3-1.7B, 4M-token batch, $1\times$ Chinchilla under cosine decay.}
\label{tab:lr-cos-17b-4m}
\setlength{\tabcolsep}{6pt}
\begin{tabular}{lccccc}
\toprule
Matrix LR
& $5\times10^{-3}$
& $8\times10^{-3}$
& $10\times10^{-3}$
& $15\times10^{-3}$
& $20\times10^{-3}$ \\
\midrule
Final val.\ loss
& 2.6228
& 2.6120
& 2.6091
& \textbf{2.6047}
& 2.6049 \\
\bottomrule
\end{tabular}
\end{table}

\subsection{WSD training and validation trajectories}

The following two tables aggregate the final losses across WSD setups. Since all methods within a setup use the same selected learning rate, the learning rate is omitted here.

\begin{table}[H]
\centering
\small
\caption{Final training loss for the WSD trajectory experiments. Bold marks the best method in each row.}
\label{tab:wsd-train-endpoints}
\setlength{\tabcolsep}{4pt}
\resizebox{\textwidth}{!}{%
\begin{tabular}{lcccccc}
\toprule
Setup & Muon & Aurora & MuonEq & Muon$^2$ & Aurora-LN & TensorChain \\
\midrule
0.6B, 1M, $1\times$ & 2.9446 & 2.9437 & 2.9397 & 2.9350 & 2.9362 & \textbf{2.9249} \\
1.7B, 4M, $4\times$ & 2.5254 & 2.5244 & 2.5226 & 2.5208 & 2.5188 & \textbf{2.5181} \\
\bottomrule
\end{tabular}
}
\end{table}

\begin{table}[H]
\centering
\small
\caption{Final validation loss for the WSD trajectory experiments. Bold marks the best method in each row.}
\label{tab:wsd-val-endpoints}
\setlength{\tabcolsep}{4pt}
\resizebox{\textwidth}{!}{%
\begin{tabular}{lcccccc}
\toprule
Setup & Muon & Aurora & MuonEq & Muon$^2$ & Aurora-LN & TensorChain \\
\midrule
0.6B, 1M, $1\times$ & 2.9201 & 2.9184 & 2.9151 & 2.9104 & 2.9116 & \textbf{2.8990} \\
1.7B, 4M, $4\times$ & 2.5310 & 2.5301 & 2.5279 & 2.5264 & 2.5242 & \textbf{2.5236} \\
\bottomrule
\end{tabular}
}
\end{table}

\begin{figure}[H]
\centering
\begin{subfigure}[b]{0.48\linewidth}
    \centering
    \includegraphics[width=\linewidth]{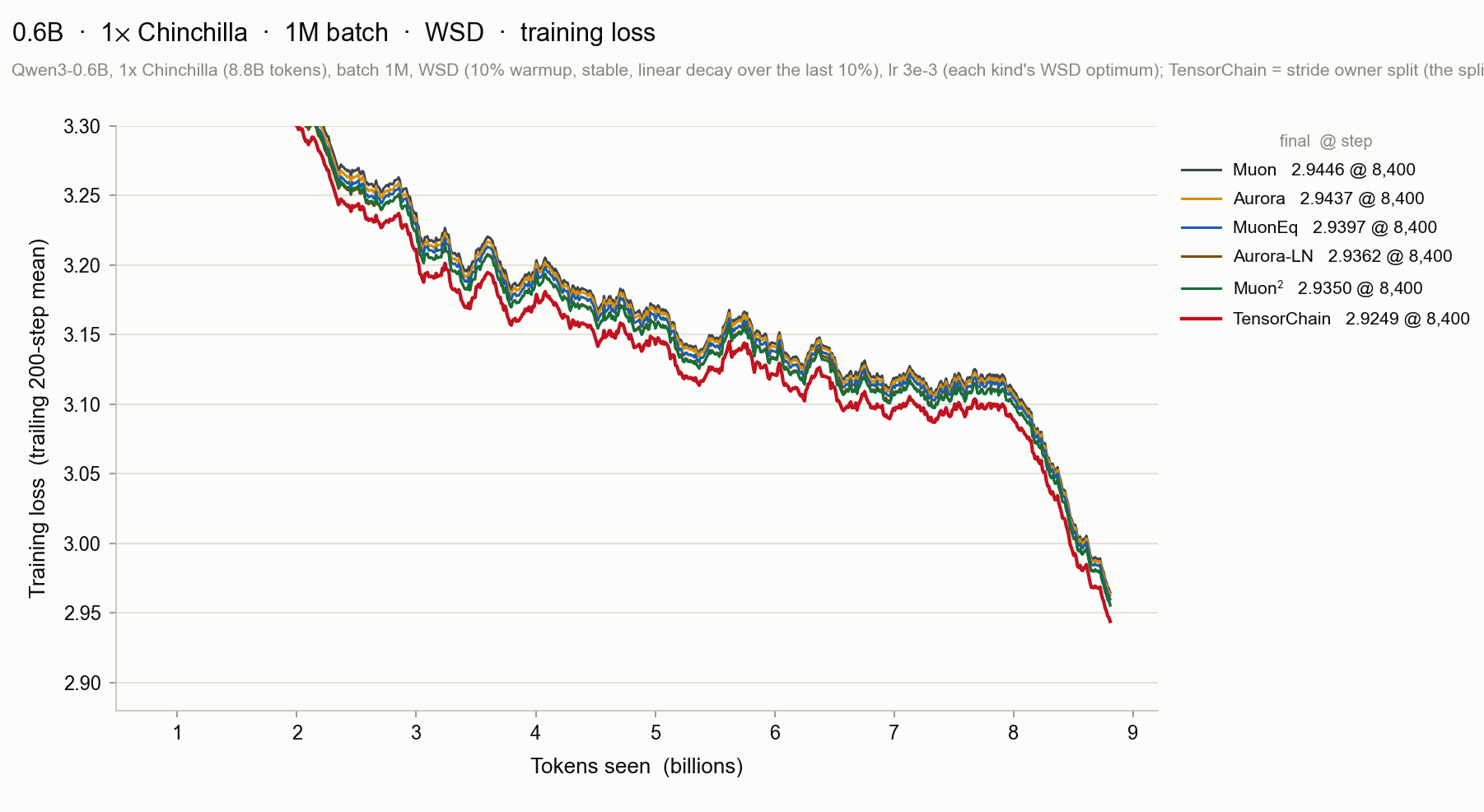}
    \caption{Training loss.}
\end{subfigure}\hfill
\begin{subfigure}[b]{0.48\linewidth}
    \centering
    \includegraphics[width=\linewidth]{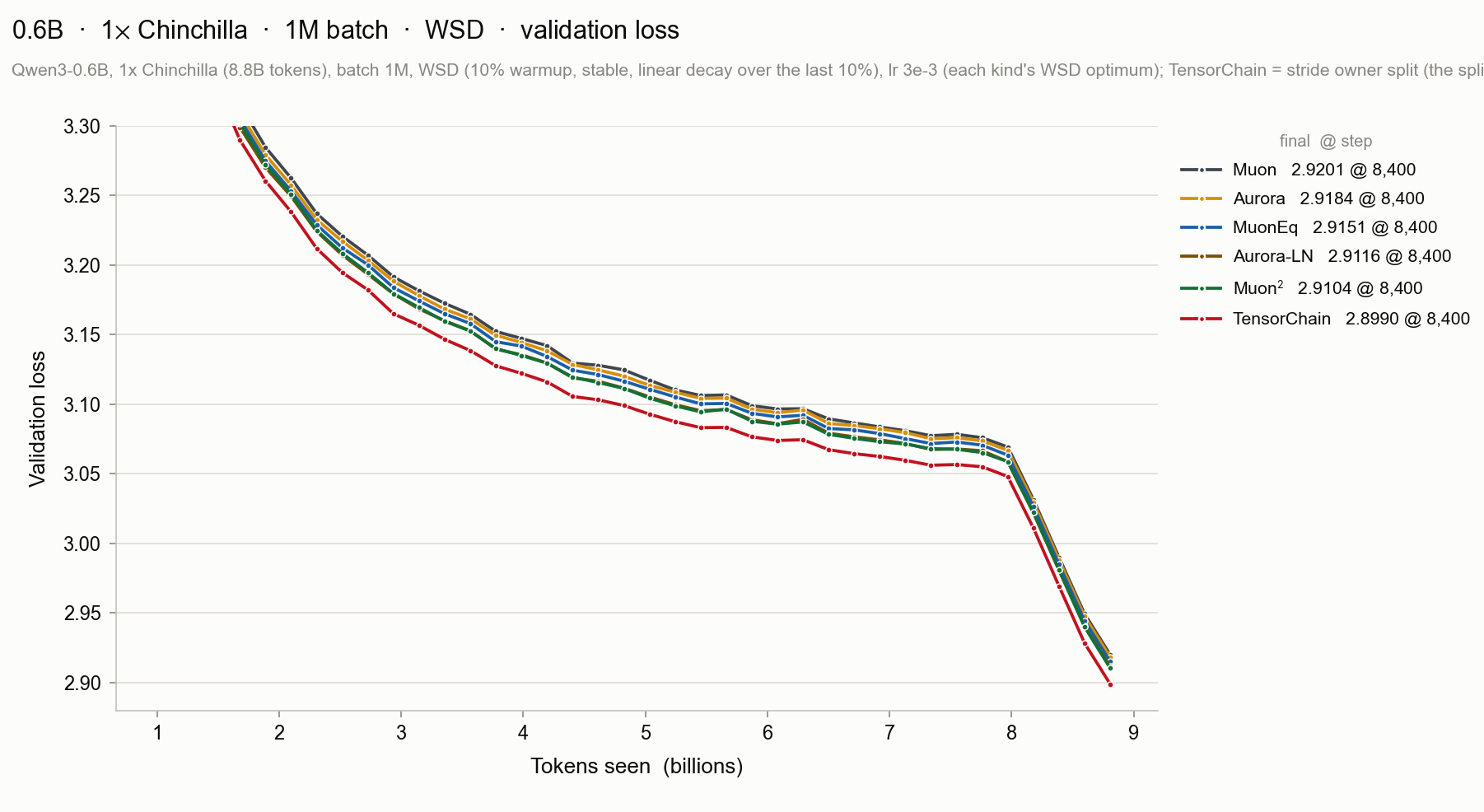}
    \caption{Validation loss.}
\end{subfigure}
\caption{WSD trajectories for Qwen3-0.6B, 1M-token batch, $1\times$ Chinchilla.}
\label{fig:raw-wsd-06b}
\end{figure}

\begin{figure}[H]
\centering
\begin{subfigure}[b]{0.48\linewidth}
    \centering
    \includegraphics[width=\linewidth]{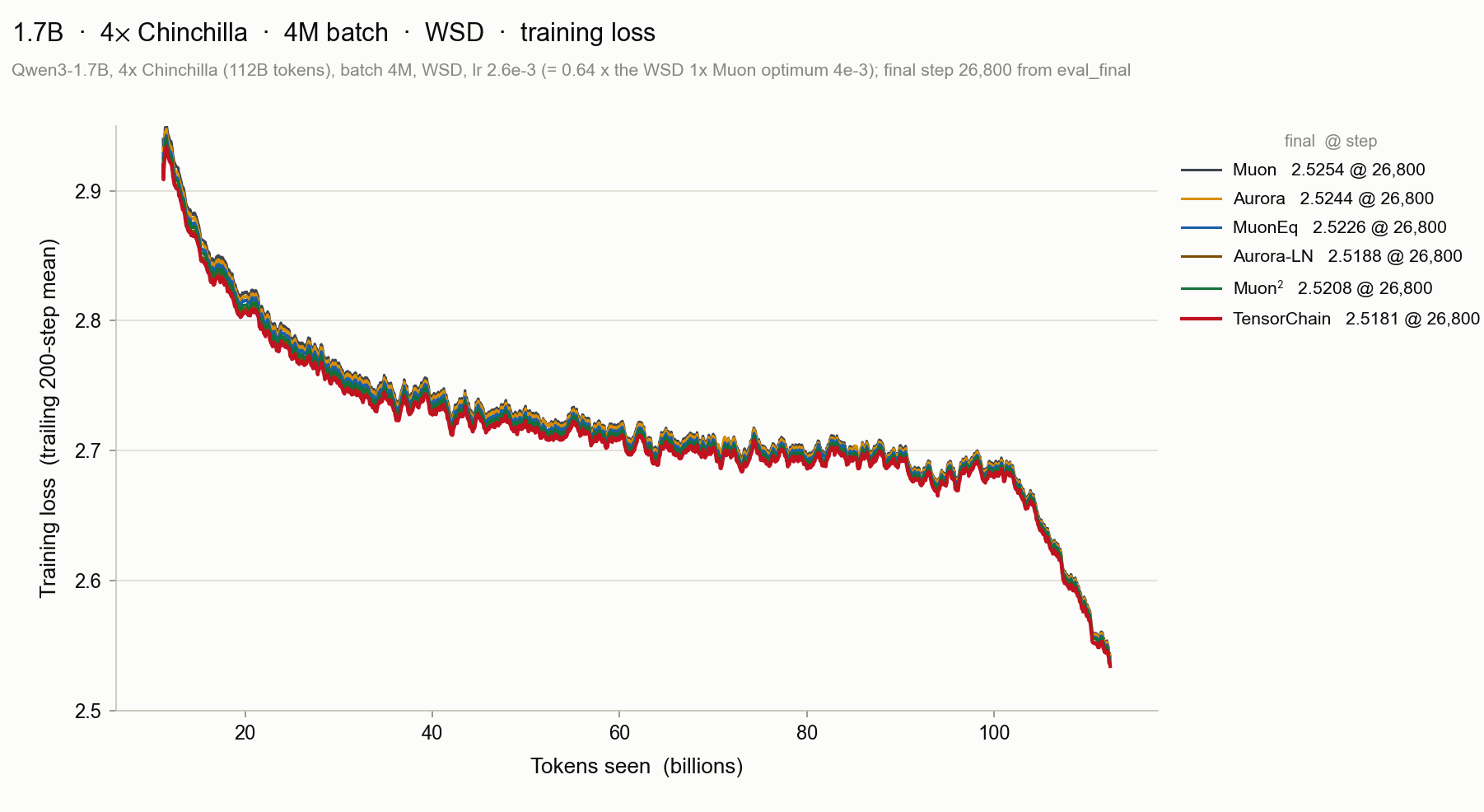}
    \caption{Training loss.}
\end{subfigure}\hfill
\begin{subfigure}[b]{0.48\linewidth}
    \centering
    \includegraphics[width=\linewidth]{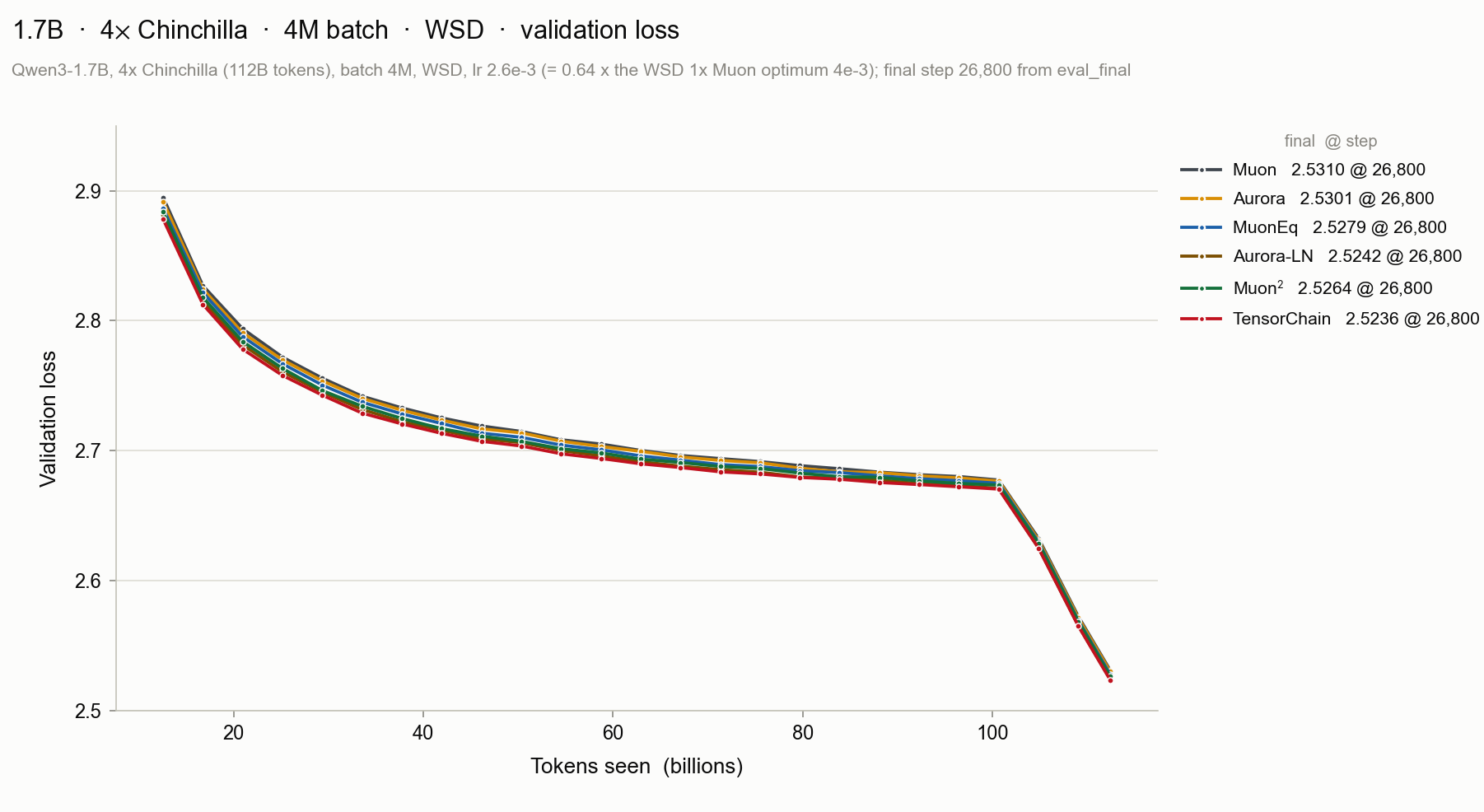}
    \caption{Validation loss.}
\end{subfigure}
\caption{WSD trajectories for Qwen3-1.7B, 4M-token batch, $4\times$ Chinchilla.}
\label{fig:raw-wsd-17b}
\end{figure}

\subsection{Cosine training and validation trajectories}

We analogously aggregate the four cosine settings. Again, learning rates are omitted because all methods within a given setup use the same selected learning rate.

\begin{table}[H]
\centering
\small
\caption{Final training loss for the cosine trajectory experiments. Bold marks the best method in each row.}
\label{tab:cos-train-endpoints}
\setlength{\tabcolsep}{4pt}
\resizebox{\textwidth}{!}{%
\begin{tabular}{lcccccc}
\toprule
Setup & Muon & Aurora & MuonEq & Muon$^2$ & Aurora-LN & TensorChain \\
\midrule
0.6B, 1M, $1\times$ & 2.8899 & 2.8889 & 2.8806 & 2.8771 & 2.8750 & \textbf{2.8636} \\
0.6B, 4M, $4\times$ & 2.7934 & 2.7897 & 2.7845 & 2.7851 & 2.7804 & \textbf{2.7748} \\
1.7B, 1M, $1\times$ & 2.6036 & 2.6017 & 2.5978 & 2.5941 & 2.5942 & \textbf{2.5922} \\
1.7B, 4M, $4\times$ & 2.3859 & 2.3850 & 2.3814 & 2.3780 & \textbf{2.3751} & 2.3778 \\
\bottomrule
\end{tabular}
}
\end{table}

\begin{table}[H]
\centering
\small
\caption{Final validation loss for the cosine trajectory experiments. Bold marks the best method in each row.}
\label{tab:cos-val-endpoints}
\setlength{\tabcolsep}{4pt}
\resizebox{\textwidth}{!}{%
\begin{tabular}{lcccccc}
\toprule
Setup & Muon & Aurora & MuonEq & Muon$^2$ & Aurora-LN & TensorChain \\
\midrule
0.6B, 1M, $1\times$ & 2.8639 & 2.8624 & 2.8540 & 2.8497 & 2.8474 & \textbf{2.8379} \\
0.6B, 4M, $4\times$ & 2.7527 & 2.7487 & 2.7434 & 2.7443 & 2.7391 & \textbf{2.7344} \\
1.7B, 1M, $1\times$ & 2.5974 & 2.5964 & 2.5924 & 2.5885 & 2.5882 & \textbf{2.5857} \\
1.7B, 4M, $4\times$ & 2.4823 & 2.4810 & 2.4784 & 2.4744 & \textbf{2.4734} & 2.4747 \\
\bottomrule
\end{tabular}
}
\end{table}

\begin{figure}[H]
\centering
\begin{subfigure}[b]{0.48\linewidth}
    \centering
    \includegraphics[width=\linewidth]{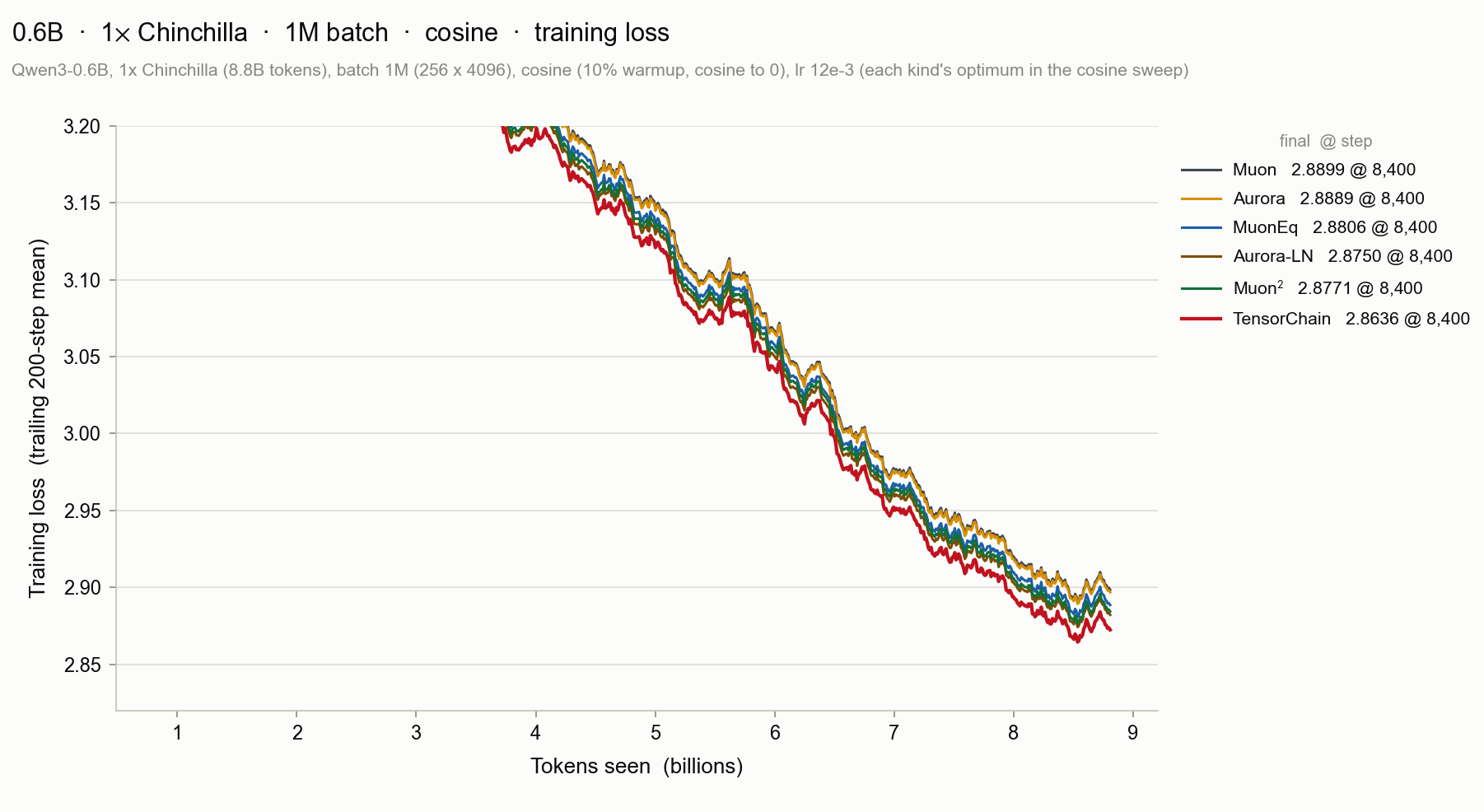}
    \caption{Training loss.}
\end{subfigure}\hfill
\begin{subfigure}[b]{0.48\linewidth}
    \centering
    \includegraphics[width=\linewidth]{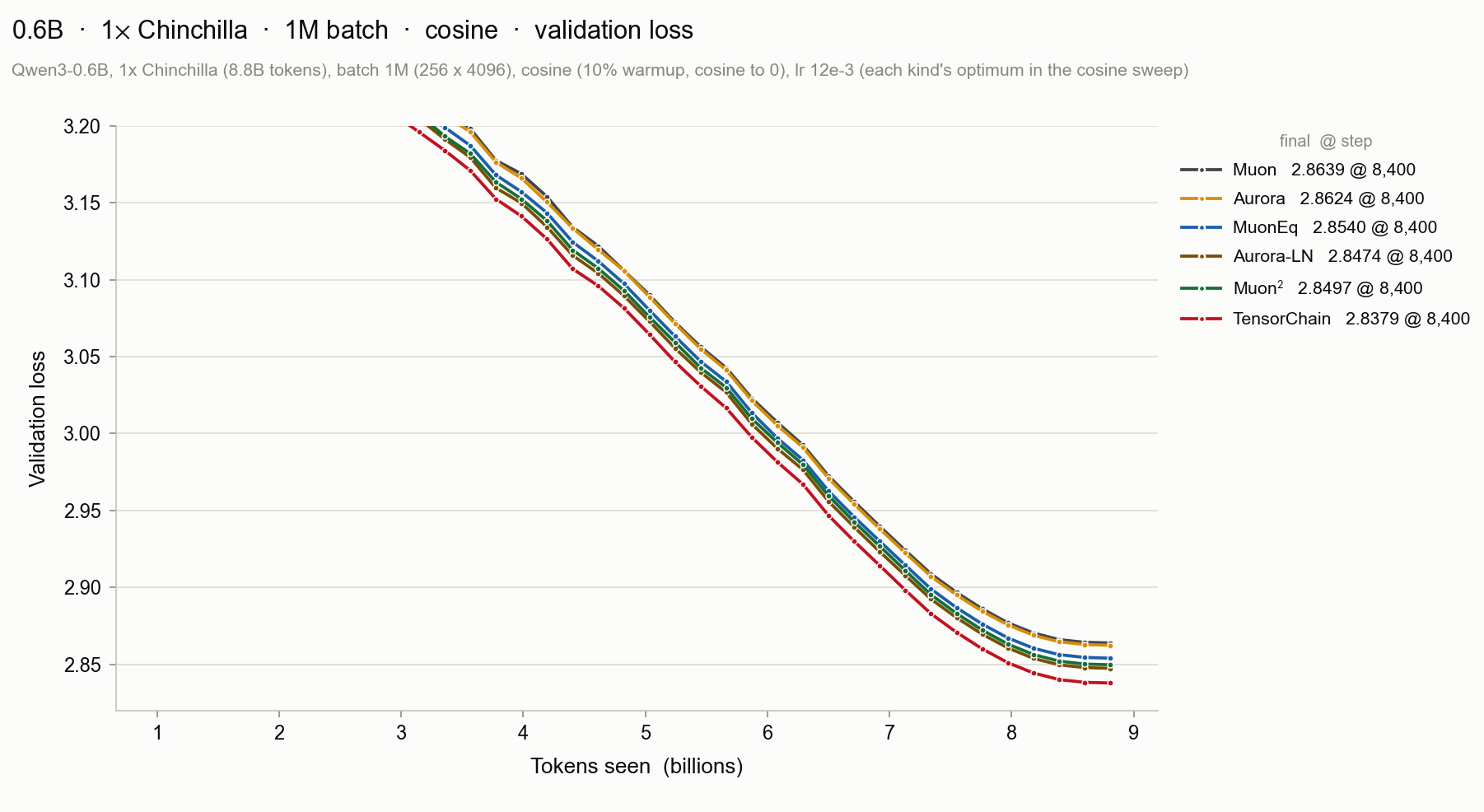}
    \caption{Validation loss.}
\end{subfigure}
\caption{Cosine trajectories for Qwen3-0.6B, 1M-token batch, $1\times$ Chinchilla.}
\label{fig:raw-cos-06b-1m}
\end{figure}

\begin{figure}[H]
\centering
\begin{subfigure}[b]{0.48\linewidth}
    \centering
    \includegraphics[width=\linewidth]{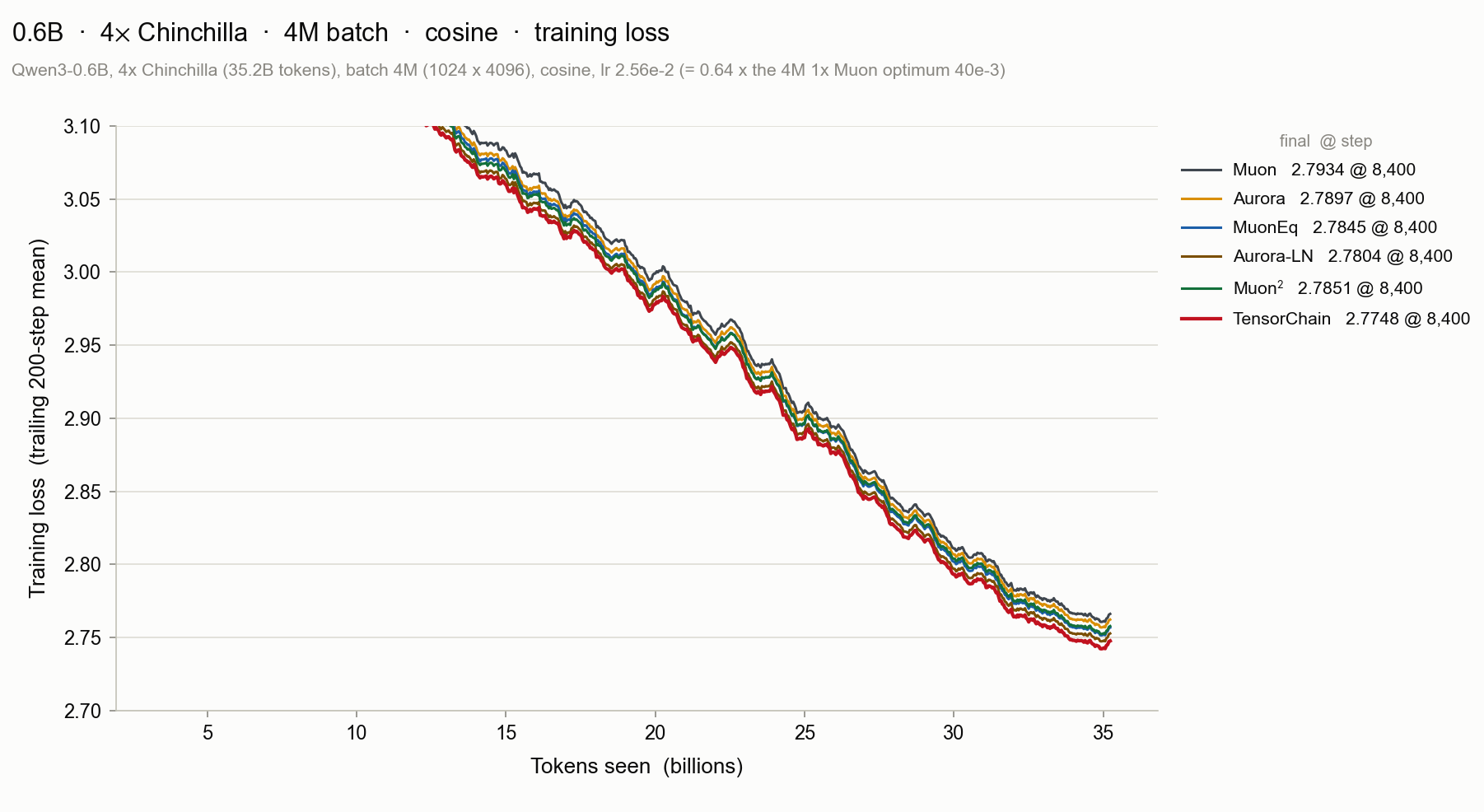}
    \caption{Training loss.}
\end{subfigure}\hfill
\begin{subfigure}[b]{0.48\linewidth}
    \centering
    \includegraphics[width=\linewidth]{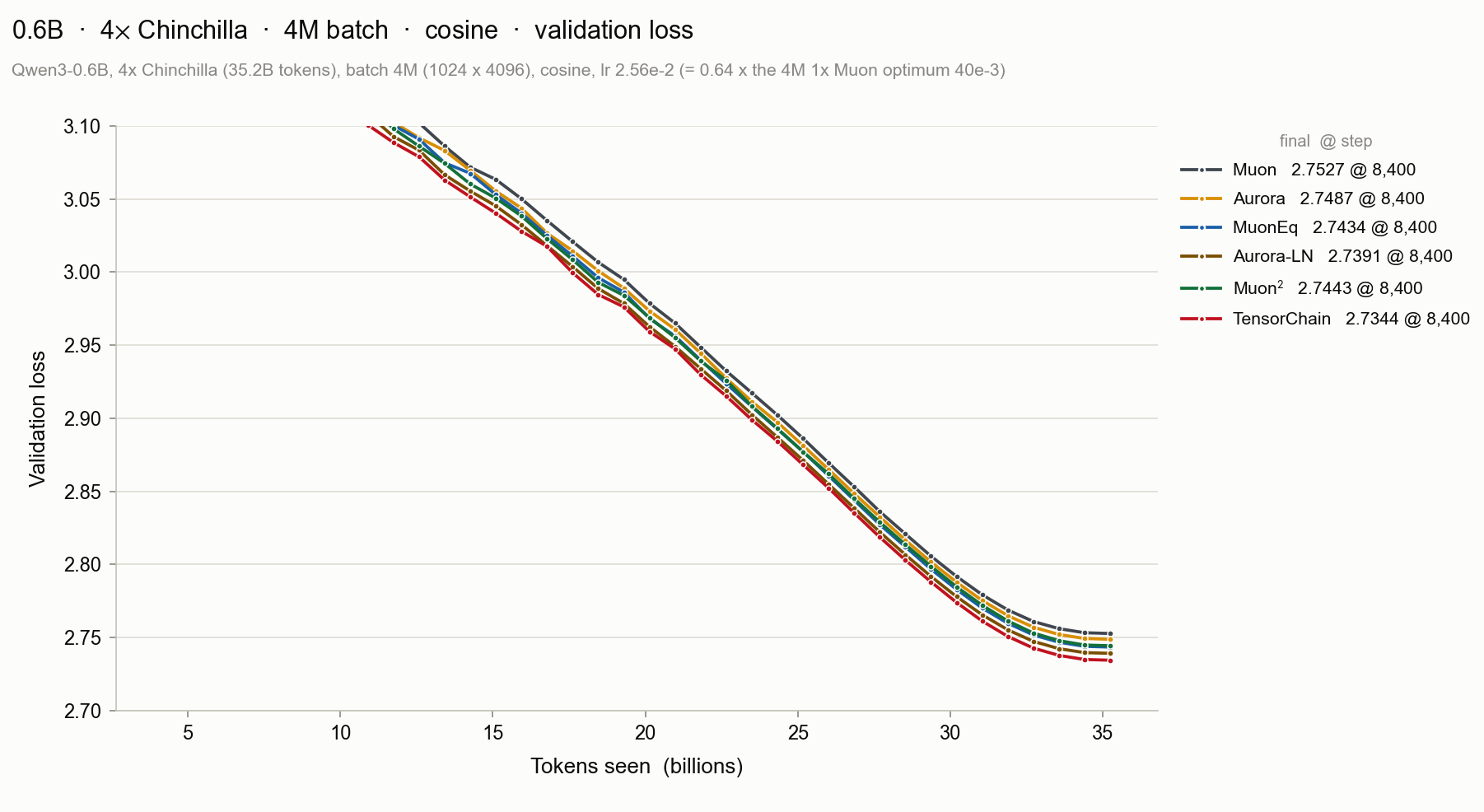}
    \caption{Validation loss.}
\end{subfigure}
\caption{Cosine trajectories for Qwen3-0.6B, 4M-token batch, $4\times$ Chinchilla.}
\label{fig:raw-cos-06b-4m}
\end{figure}

\begin{figure}[H]
\centering
\begin{subfigure}[b]{0.48\linewidth}
    \centering
    \includegraphics[width=\linewidth]{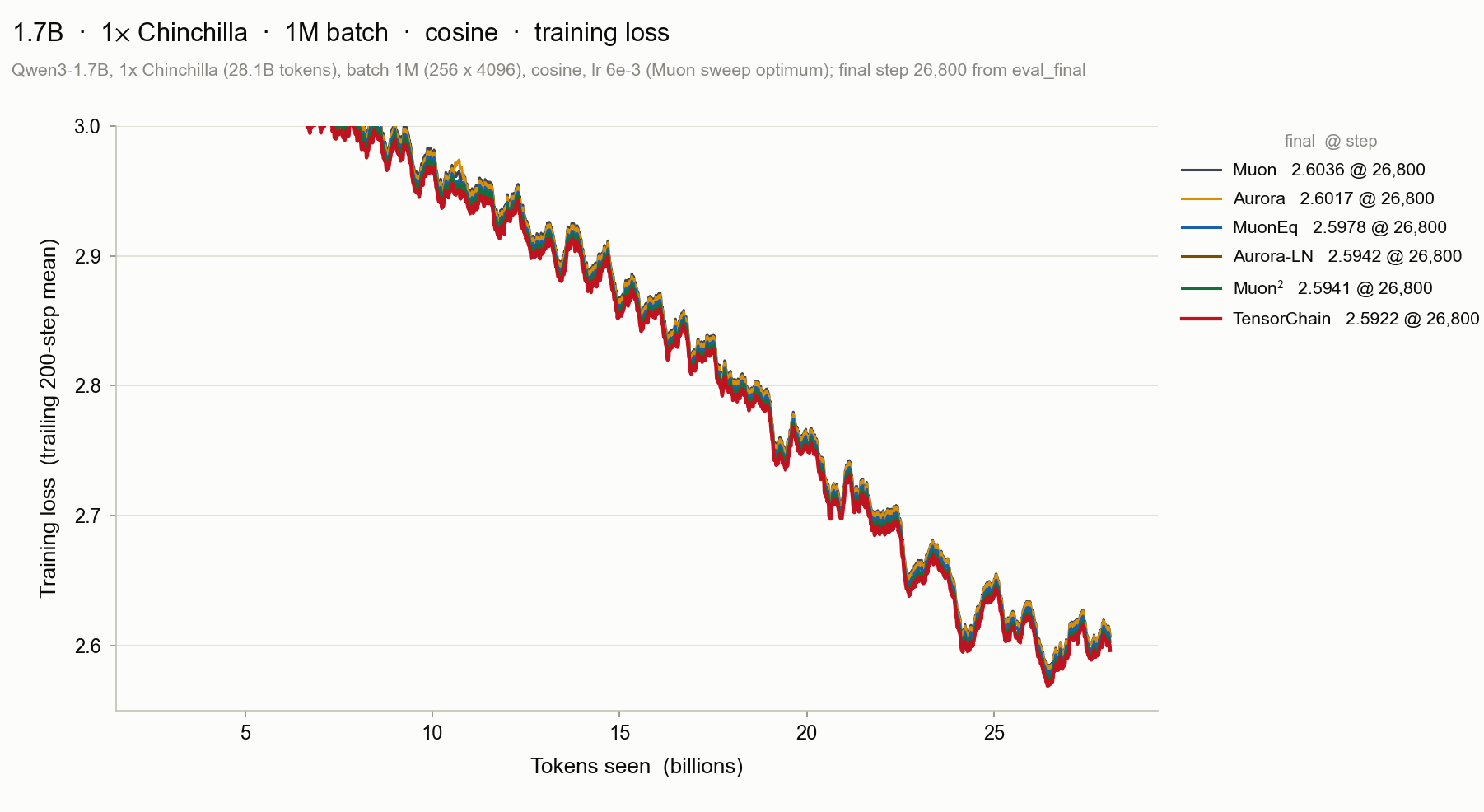}
    \caption{Training loss.}
\end{subfigure}\hfill
\begin{subfigure}[b]{0.48\linewidth}
    \centering
    \includegraphics[width=\linewidth]{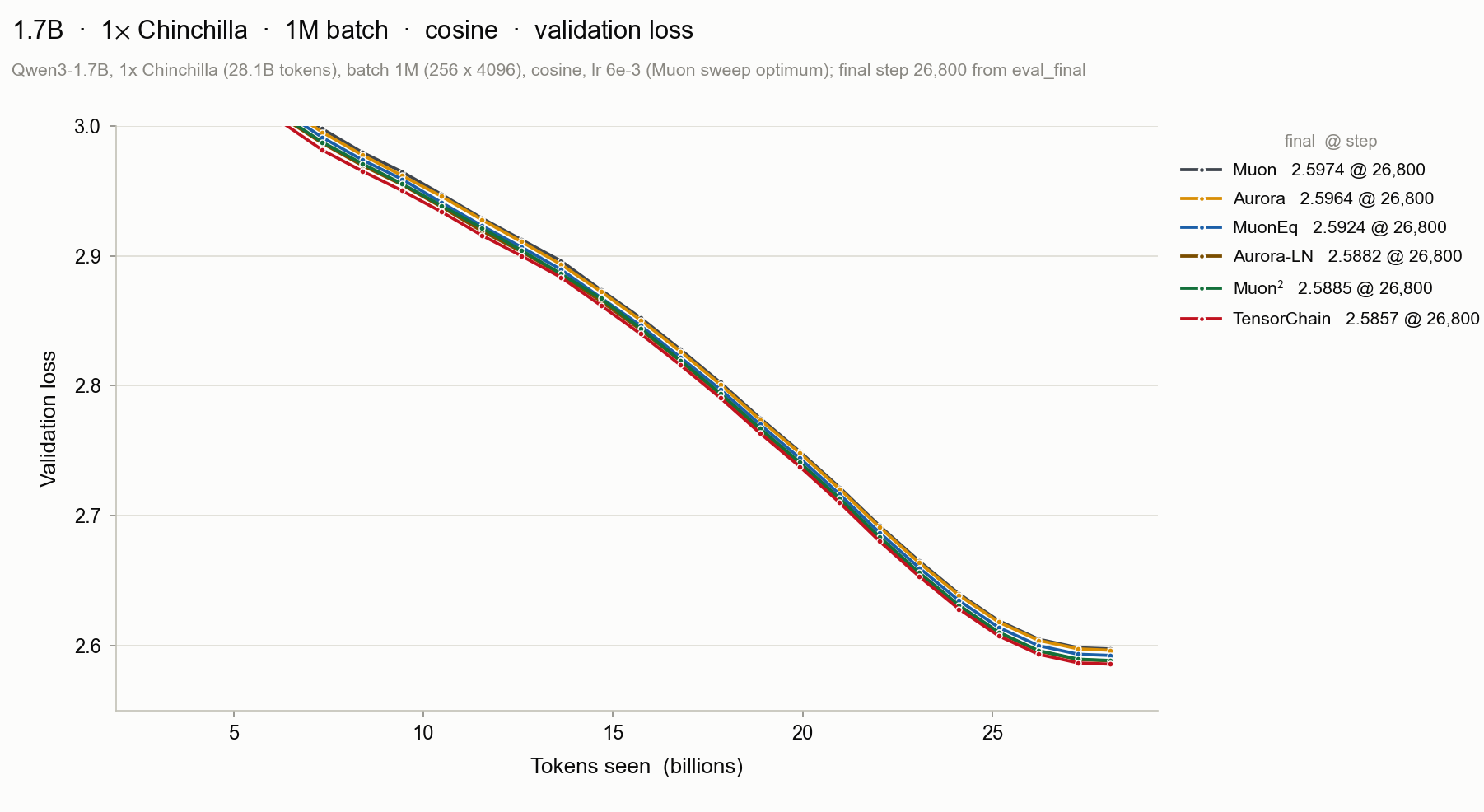}
    \caption{Validation loss.}
\end{subfigure}
\caption{Cosine trajectories for Qwen3-1.7B, 1M-token batch, $1\times$ Chinchilla.}
\label{fig:raw-cos-17b-1m}
\end{figure}

\begin{figure}[H]
\centering
\begin{subfigure}[b]{0.48\linewidth}
    \centering
    \includegraphics[width=\linewidth]{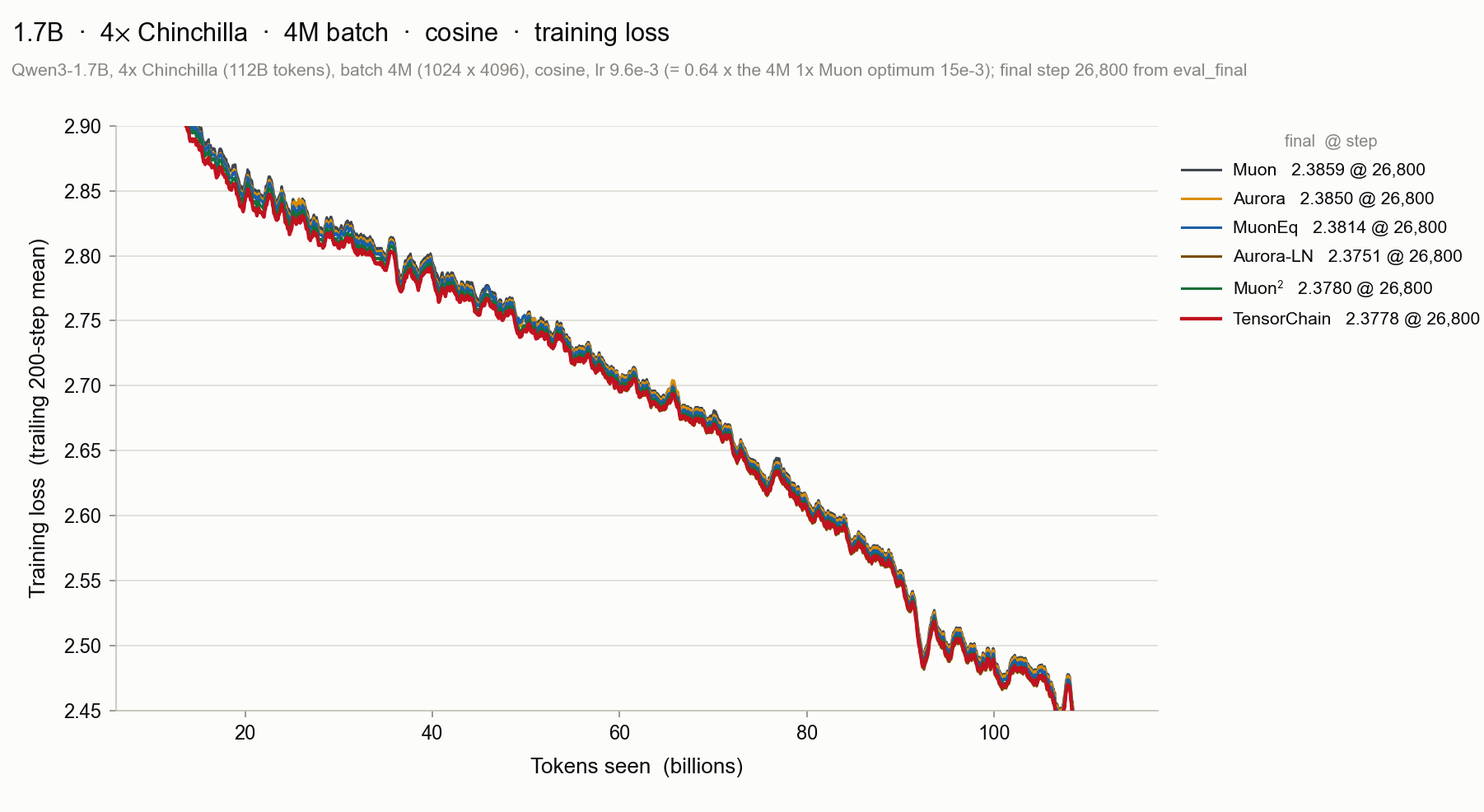}
    \caption{Training loss.}
\end{subfigure}\hfill
\begin{subfigure}[b]{0.48\linewidth}
    \centering
    \includegraphics[width=\linewidth]{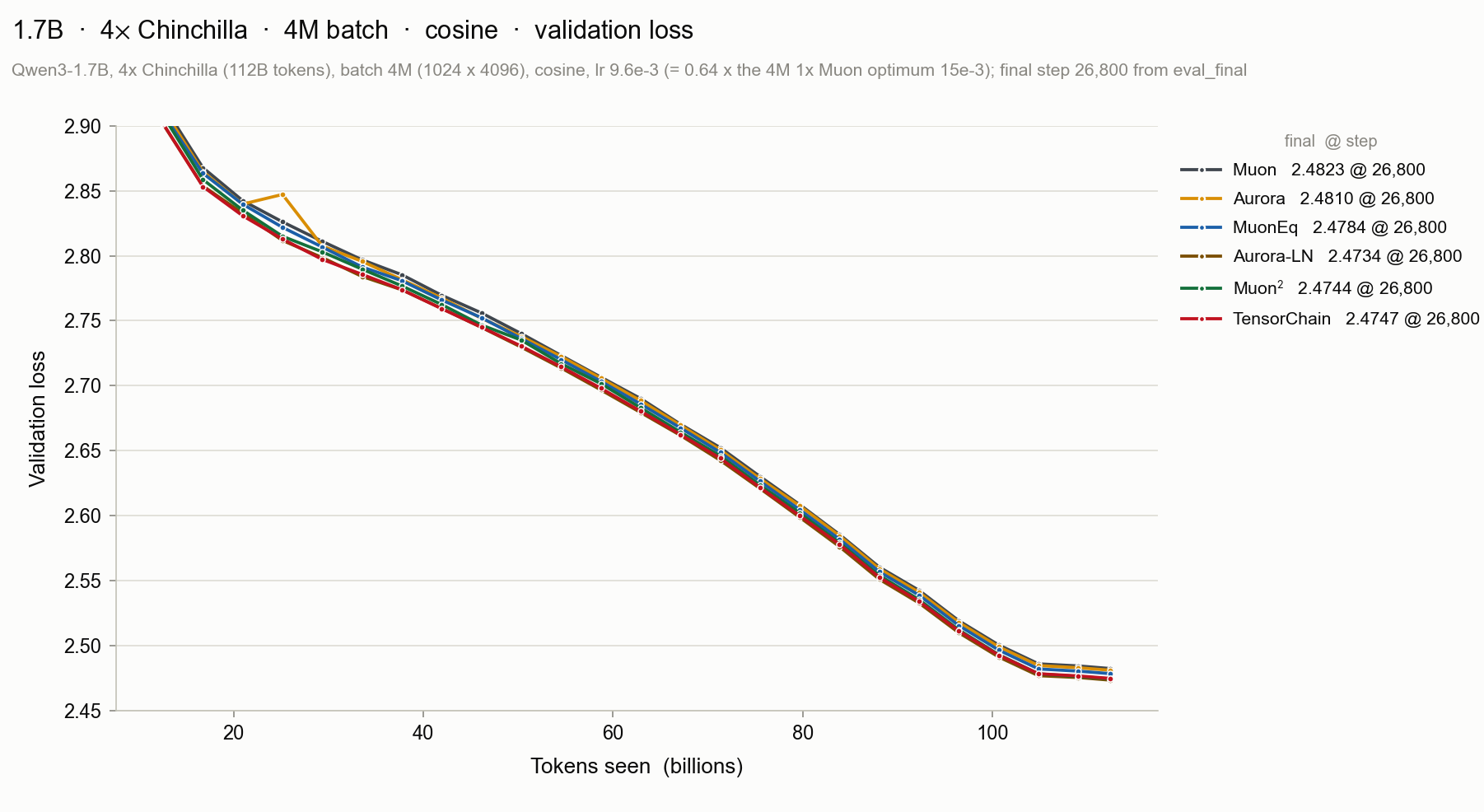}
    \caption{Validation loss.}
\end{subfigure}
\caption{Cosine trajectories for Qwen3-1.7B, 4M-token batch, $4\times$ Chinchilla.}
\label{fig:raw-cos-17b-4m}
\end{figure}

\begin{figure}[H]
\centering
\begin{subfigure}[b]{0.48\linewidth}
    \centering
    \includegraphics[width=\linewidth]{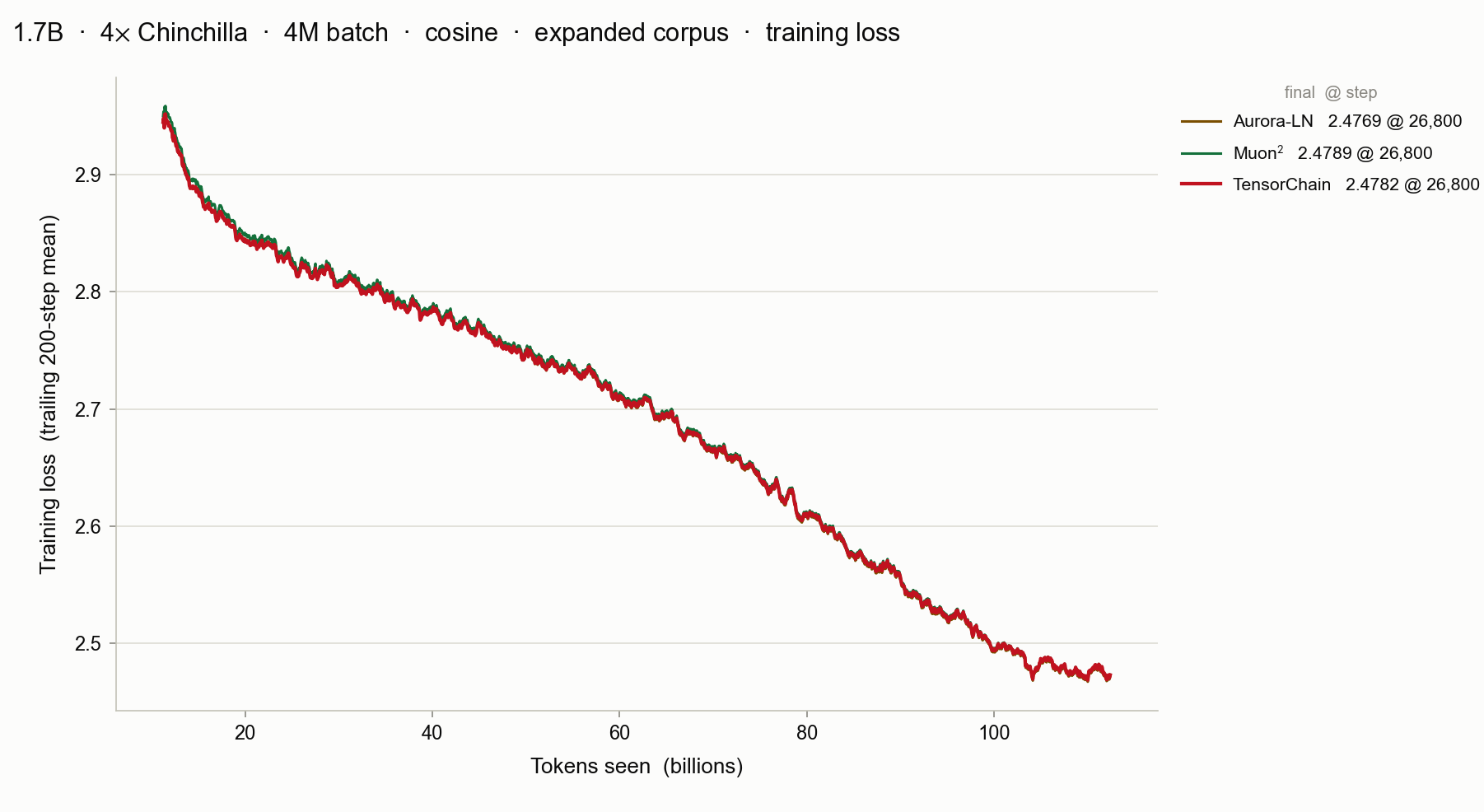}
    \caption{Training loss.}
\end{subfigure}\hfill
\begin{subfigure}[b]{0.48\linewidth}
    \centering
    \includegraphics[width=\linewidth]{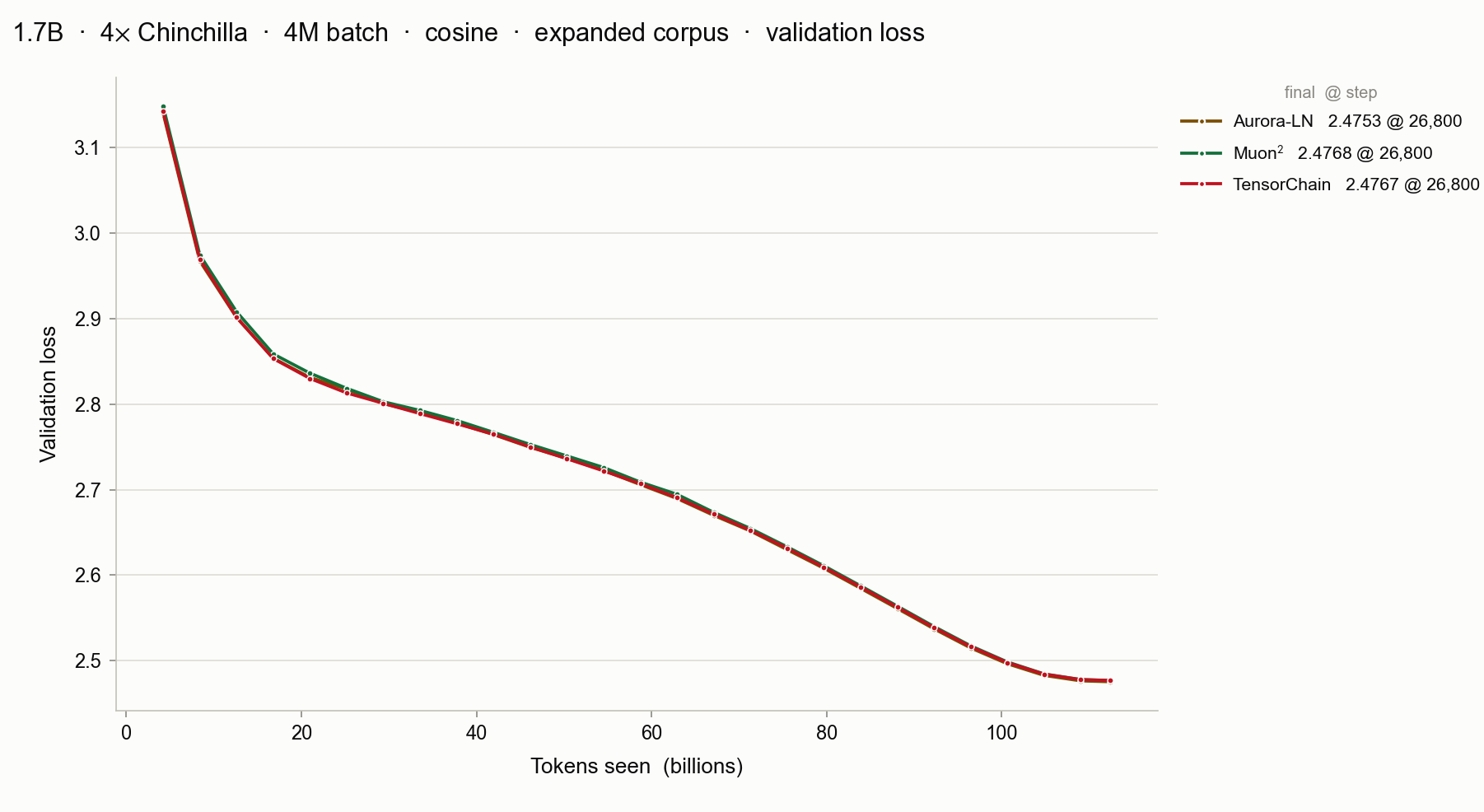}
    \caption{Validation loss.}
\end{subfigure}
\caption{Cosine trajectories for Qwen3-1.7B, 4M-token batch, $4\times$ Chinchilla, distinct tokens.}
\label{fig:raw-cos-17b-4m-fixed}
\end{figure}

\subsection{TensorChain grouping ablation}
\label{app:grouping-results}

TensorChain requires choosing which compatible matrices are grouped along the layer axis. For a 28-layer model, \emph{consecutive} forms four contiguous groups, e.g., $\{1{:}7\},\{8{:}14\},\{15{:}21\},\{22{:}28\}$, whereas
\emph{stride} interleaves depth, e.g.,
$\{1,5,9,\ldots,25\},\{2,6,10,\ldots,26\},\{3,7,11,\ldots,27\},
\{4,8,12,\ldots,28\}$. \emph{Greedy} assigns compatible matrices to
parameter-balanced groups, while \emph{all} jointly groups all compatible matrices $\{1,2,\cdots 28\}$. The results are extracted from the same Qwen3-0.6B full sweeps above.

\begin{table}[H]
\centering
\small
\caption{TensorChain grouping ablation on Qwen3-0.6B, 1M-token batch, $1\times$ Chinchilla. Each entry reports the best point from the corresponding full sweep.}
\label{tab:tc-grouping}
\begin{tabular}{lcccc}
\toprule
& \multicolumn{2}{c}{WSD} & \multicolumn{2}{c}{Cosine} \\
\cmidrule(lr){2-3}\cmidrule(lr){4-5}
Grouping & Best LR & Val.\ loss & Best LR & Val.\ loss \\
\midrule
Consecutive & $3\times10^{-3}$ & 2.9070 & $12\times10^{-3}$ & 2.8474 \\
Stride & $3\times10^{-3}$ & 2.8990 & $12\times10^{-3}$ & 2.8357 \\
Greedy & $3\times10^{-3}$ & 2.8994 & $16\times10^{-3}$ & 2.8370 \\
All & $3\times10^{-3}$ & 2.9000 & $12\times10^{-3}$ & 2.8379 \\
\bottomrule
\end{tabular}
\end{table}

Consecutive grouping performs worse near its optimum than stride, greedy, and all-layer grouping. Since the latter three mix matrices across more distant depths despite using different assignment rules, this suggests that cross-depth mixing matters more than the precise partition.

\subsection{Cosine token efficiency}
\label{app:token_efficiency}

The cosine schedule yields smooth validation trajectories throughout training, so we additionally compare methods by the number of training tokens required to reach the same validation loss. We do not use this metric for WSD because its terminal decay concentrates much of the loss reduction near the end of training, making within-run crossing times less informative.

For TensorChain and a baseline $B$, define the common target
\[
\ell^\star(\mathrm{TC},B)
=
\max\!\left\{
\ell_{\mathrm{TC}}(D_{\max}),
\ell_B(D_{\max})
\right\},
\]
the worse of the two final validation losses. Both methods therefore reach the target within the observed training horizon, so no extrapolation is required. For a method $M$, let
\[
D_M(\ell)
=
\inf\!\left\{
D:\widehat{\ell}_M(D)\le \ell
\right\},
\]
where $\widehat{\ell}_M(D)$ is obtained by linear interpolation between the two adjacent validation checkpoints bracketing the target. We define
\[
\mathrm{TE}(\mathrm{TC};B)
=
100\left(
1-\frac{D_{\mathrm{TC}}(\ell^\star)}{D_B(\ell^\star)}
\right)\%.
\]
Positive values indicate that TensorChain reaches the same validation loss using fewer training tokens, while negative values favor the baseline. Since each pairwise comparison fixes the batch size, the same percentage is equivalently the reduction in optimization steps.

\begin{table}[H]
\centering
\small
\caption{Cosine-schedule token-efficiency gain of TensorChain. Each entry is the percentage reduction in training tokens required to reach the same validation-loss target. Positive values favor TensorChain.}
\label{tab:cosine-token-efficiency}
\begin{tabular}{lccccc}
\toprule
Setting & Aurora-LN & Muon$^2$ & MuonEq & Aurora & Muon \\
\midrule
0.6B, 1M, $1\times$ & 8.2 & 9.0 & 10.3 & 12.4 & 12.8 \\
0.6B, 4M, $4\times$ & 5.4 & 7.6 & 7.4 & 8.9 & 10.0 \\
1.7B, 1M, $1\times$ & 3.9 & 4.0 & 6.2 & 7.5 & 7.8 \\
1.7B, 4M, $4\times$ & -2.0 & -0.5 & 6.7 & 7.4 & 7.8 \\
\bottomrule
\end{tabular}
\end{table}

\subsection{Comparison with the Teon optimizer}
\label[appendix]{app:teon}

Since we could not find an open-sourced implementation of Teon \citep{zhang2026teon}, we implemented our own version of Teon. We concatenated the two consecutive Q, K, V weights and applied Newton--Schulz, following the suggestion of \citep{zhang2026teon}. We tuned the learning rate of Teon in Qwen3 0.6B pretraining setup, with 1M batch size and 1x chinchilla tokens, with WSD schedule with 10$\%$ warmup and decay. To our surprise, the Teon optimizer performed worse than Muon for our setup; The best validation loss for Teon was 2.9222, whereas Muon was 2.9201. Hence we decided to not include Teon in our baselines.
\begin{figure}[t]
    \centering
    \includegraphics[width=\linewidth]{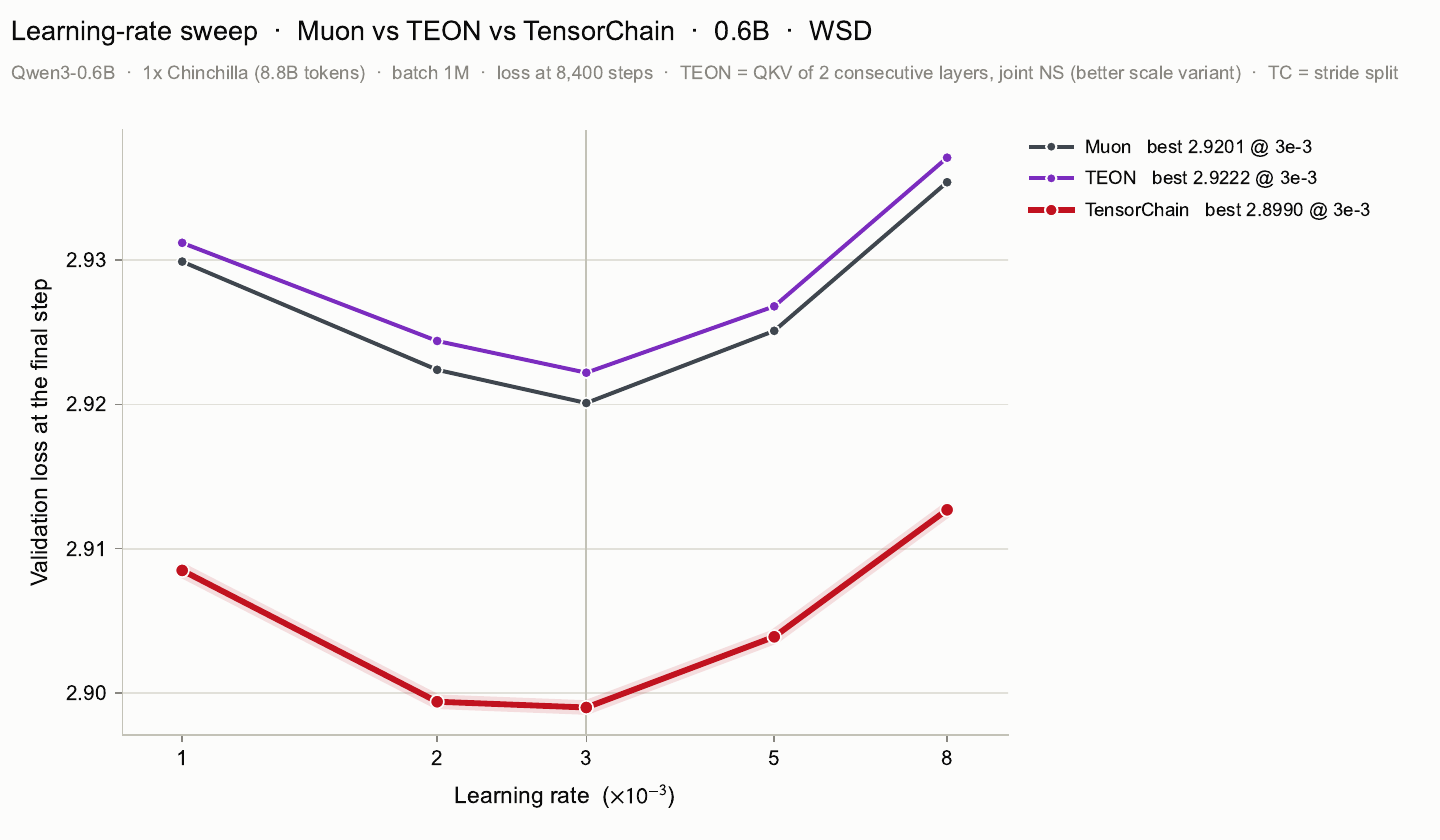}
\end{figure}

%% file: iclr2027_conference.bib
@article{chang2026muoneq,
  title={Muoneq: Balancing before orthogonalization with lightweight equilibration},
  author={Chang, Da and Shi, Qiankun and Zhang, Lvgang and Li, Yu and Zhang, Ruijie and Lu, Yao and Liu, Yongxiang and Yuan, Ganzhao},
  journal={arXiv preprint arXiv:2603.28254},
  year={2026}
}

@article{liu2026muon,
  title={Muon2: Boosting Muon via Adaptive Second-Moment Preconditioning},
  author={Liu, Ziyue and Zhang, Ruijie and Wang, Zhengyang and Zhao, Yequan and Su, Yupeng and Yang, Zi and Zhang, Zheng},
  journal={arXiv preprint arXiv:2604.09967},
  year={2026}
}

@article{zhang2026muon+,
  title={MUON+: Towards More Effective Muon via One Additional Normalization Step for LLM Pre-training},
  author={Zhang, Ruijie and Zhao, Yequan and Liu, Ziyue and Wang, Zhengyang and Su, Yupeng and Tan, Liyan and Zhang, Zheng},
  journal={arXiv preprint arXiv:2602.21545},
  year={2026}
}

@article{li2025normuon,
  title={Normuon: Making muon more efficient and scalable},
  author={Li, Zichong and Liu, Liming and Liang, Chen and Chen, Weizhu and Zhao, Tuo},
  journal={arXiv preprint arXiv:2510.05491},
  year={2025}
}

@article{si2025adamuon,
  title={Adamuon: Adaptive muon optimizer},
  author={Si, Chongjie and Zhang, Debing and Shen, Wei},
  journal={arXiv preprint arXiv:2507.11005},
  year={2025}
}

@article{dewulf2026aurora,
  title={Aurora: A Leverage-Aware Spectral Optimizer},
  author={Dewulf, Alec and Pai, Dhruv and Yang, Li and Zhang, Ashley and Keigwin, Ben},
  journal={arXiv preprint arXiv:2606.27715},
  year={2026}
}

@article{bernstein2024old,
  title={Old optimizer, new norm: An anthology},
  author={Bernstein, Jeremy and Newhouse, Laker},
  journal={arXiv preprint arXiv:2409.20325},
  year={2024}
}

@article{veprikov2025preconditioned,
  title={Preconditioned norms: A unified framework for steepest descent, quasi-newton and adaptive methods},
  author={Veprikov, Andrey and Bolatov, Arman and Bogdanov, Aleksandr and Horv{\'a}th, Samuel and Beznosikov, Aleksandr and Tak{\'a}{\v{c}}, Martin and Hanzely, Slavomir},
  journal={arXiv preprint arXiv:2510.10777},
  year={2025}
}

@article{ma2026schattor,
  title={Schattor: Schatten-family methods for deep learning optimization},
  author={Ma, Bohao and Zhang, Junyu and He, Chuan},
  journal={arXiv preprint arXiv:2606.15702},
  year={2026}
}

@article{pethick2025training,
  title={Training deep learning models with norm-constrained lmos},
  author={Pethick, Thomas and Xie, Wanyun and Antonakopoulos, Kimon and Zhu, Zhenyu and Silveti-Falls, Antonio and Cevher, Volkan},
  journal={arXiv preprint arXiv:2502.07529},
  year={2025}
}

@article{kovalev2025understanding,
  title={Understanding gradient orthogonalization for deep learning via non-euclidean trust-region optimization},
  author={Kovalev, Dmitry},
  journal={arXiv preprint arXiv:2503.12645},
  year={2025}
}

@article{crawshaw2025exploration,
  title={An exploration of non-euclidean gradient descent: Muon and its many variants},
  author={Crawshaw, Michael and Modi, Chirag and Liu, Mingrui and Gower, Robert M},
  journal={arXiv preprint arXiv:2510.09827},
  year={2025}
}

@inproceedings{chen2024lion,
  title={Lion secretly solves a constrained optimization: As lyapunov predicts},
  author={Chen, Lizhang and Liu, Bo and Liang, Kaizhao and Liu, Qiang},
  booktitle={International Conference on Learning Representations},
  volume={2024},
  pages={35404--35439},
  year={2024}
}

@article{chen2025muon,
  title={Muon optimizes under spectral norm constraints},
  author={Chen, Lizhang and Li, Jonathan and Liu, Qiang},
  journal={arXiv preprint arXiv:2506.15054},
  year={2025}
}

@article{lau2025polargrad,
  title={Polargrad: A class of matrix-gradient optimizers from a unifying preconditioning perspective},
  author={Lau, Tim Tsz-Kit and Long, Qi and Su, Weijie},
  journal={arXiv preprint arXiv:2505.21799},
  year={2025}
}

@article{kingma2014adam,
  title={Adam: A method for stochastic optimization},
  author={Kingma, Diederik P and Ba, Jimmy},
  journal={arXiv preprint arXiv:1412.6980},
  year={2014}
}

@inproceedings{gupta2018shampoo,
  title={Shampoo: Preconditioned stochastic tensor optimization},
  author={Gupta, Vineet and Koren, Tomer and Singer, Yoram},
  booktitle={International Conference on Machine Learning},
  pages={1842--1850},
  year={2018},
  organization={PMLR}
}

@article{mishchenko2023prodigy,
  title={Prodigy: An expeditiously adaptive parameter-free learner},
  author={Mishchenko, Konstantin and Defazio, Aaron},
  journal={arXiv preprint arXiv:2306.06101},
  year={2023}
}

@article{lau2026symmetry,
  title={Symmetry-Compatible Principle for Optimizer Design: Embeddings, LM Heads, SwiGLU MLPs, and MoE Routers},
  author={Lau, Tim Tsz-Kit and Su, Weijie},
  journal={arXiv preprint arXiv:2605.18106},
  year={2026}
}

@article{fletcher1970new,
  title={A new approach to variable metric algorithms},
  author={Fletcher, Roger},
  journal={The computer journal},
  volume={13},
  number={3},
  pages={317--322},
  year={1970},
  publisher={Oxford University Press}
}

@article{pethick2026use,
  title={When to use what Schatten-$ p $ norm in deep learning?},
  author={Pethick, Thomas},
  journal={arXiv preprint arXiv:2606.15268},
  year={2026}
}

@article{shumaylov2026muon,
  title={Muon is Not That Special: Random or Inverted Spectra Work Just as Well},
  author={Shumaylov, Zakhar and Da Costa, Natha{\"e}l and Zaika, Peter and Mucs{\'a}nyi, B{\'a}lint and Massucco, Alex and Gelberg, Yoav and Sch{\"o}nlieb, Carola-Bibiane and Gal, Yarin and Hennig, Philipp},
  journal={arXiv preprint arXiv:2605.11181},
  year={2026}
}

@article{duchi2011adaptive,
  title={Adaptive subgradient methods for online learning and stochastic optimization.},
  author={Duchi, John and Hazan, Elad and Singer, Yoram},
  journal={Journal of machine learning research},
  volume={12},
  number={7},
  year={2011}
}

@article{penedo2024fineweb,
  title={The fineweb datasets: Decanting the web for the finest text data at scale},
  author={Penedo, Guilherme and Kydl{\'\i}{\v{c}}ek, Hynek and Lozhkov, Anton and Mitchell, Margaret and Raffel, Colin and Von Werra, Leandro and Wolf, Thomas and others},
  journal={Advances in Neural Information Processing Systems},
  volume={37},
  pages={30811--30849},
  year={2024}
}

@inproceedings{liang2025torchtitan,
  title={Torchtitan: One-stop pytorch native solution for production ready LLM pretraining},
  author={Liang, Wanchao and Liu, Tianyu and Wright, Less and Constable, Will and Gu, Andrew and Huang, Chien-Chin and Zhang, Iris and Feng, Wei and Huang, Howard and Wang, Junjie and others},
  booktitle={International Conference on Learning Representations},
  volume={2025},
  pages={92161--92181},
  year={2025}
}

@article{yang2025qwen3,
  title={Qwen3 technical report},
  author={Yang, An and Li, Anfeng and Yang, Baosong and Zhang, Beichen and Hui, Binyuan and Zheng, Bo and Yu, Bowen and Gao, Chang and Huang, Chengen and Lv, Chenxu and others},
  journal={arXiv preprint arXiv:2505.09388},
  year={2025}
}

@article{zhang2019root,
  title={Root mean square layer normalization},
  author={Zhang, Biao and Sennrich, Rico},
  journal={Advances in neural information processing systems},
  volume={32},
  year={2019}
}

@article{liu2025muon,
  title={Muon is scalable for llm training},
  author={Liu, Jingyuan and Su, Jianlin and Yao, Xingcheng and Jiang, Zhejun and Lai, Guokun and Du, Yulun and Qin, Yidao and Xu, Weixin and Lu, Enzhe and Yan, Junjie and others},
  journal={arXiv preprint arXiv:2502.16982},
  year={2025}
}

@misc{modded_nanogpt_2024,
  author       = {Keller Jordan and Jeremy Bernstein and Brendan Rappazzo and @fernbear.bsky.social and Boza Vlado and You Jiacheng and Franz Cesista and Braden Koszarsky and @Grad62304977},
  title        = {modded-nanogpt: Speedrunning the NanoGPT baseline},
  year         = {2024},
  url          = {https://github.com/KellerJordan/modded-nanogpt}
}

@article{zhong2025understanding,
  title={Understanding transformer from the perspective of associative memory},
  author={Zhong, Shu and Xu, Mingyu and Ao, Tenglong and Shi, Guang},
  journal={arXiv preprint arXiv:2505.19488},
  year={2025}
}

@article{meng2022locating,
  title={Locating and editing factual associations in gpt},
  author={Meng, Kevin and Bau, David and Andonian, Alex and Belinkov, Yonatan},
  journal={Advances in neural information processing systems},
  volume={35},
  pages={17359--17372},
  year={2022}
}

@article{meng2022mass,
  title={Mass-editing memory in a transformer},
  author={Meng, Kevin and Sharma, Arnab Sen and Andonian, Alex and Belinkov, Yonatan and Bau, David},
  journal={arXiv preprint arXiv:2210.07229},
  year={2022}
}

@inproceedings{wang2026muon,
  title={Muon outperforms adam in tail-end associative memory learning},
  author={Wang, Shuche and Zhang, Fengzhuo and Li, Jiaxiang and Du, Cunxiao and Du, Chao and Pang, Tianyu and Yang, Zhuoran and Hong, Mingyi and Tan, Vincent},
  booktitle={International Conference on Learning Representations},
  volume={2026},
  pages={37382--37419},
  year={2026}
}

@article{kim2026sharp,
  title={Sharp capacity scaling of spectral optimizers in learning associative memory},
  author={Kim, Juno and Nichani, Eshaan and Wu, Denny and Bietti, Alberto and Lee, Jason D},
  journal={arXiv preprint arXiv:2603.26554},
  year={2026}
}

@article{li2026muon,
  title={Muon in associative memory learning: Training dynamics and scaling laws},
  author={Li, Binghui and Wang, Kaifei and Zhong, Han and Lu, Pinyan and Wang, Liwei},
  journal={arXiv preprint arXiv:2602.05725},
  year={2026}
}

@article{hoffmann2022training,
  title={Training compute-optimal large language models},
  author={Hoffmann, Jordan and Borgeaud, Sebastian and Mensch, Arthur and Buchatskaya, Elena and Cai, Trevor and Rutherford, Eliza and Casas, Diego de Las and Hendricks, Lisa Anne and Welbl, Johannes and Clark, Aidan and others},
  journal={arXiv preprint arXiv:2203.15556},
  year={2022}
}

@inproceedings{bjorck2025scaling,
  title={Scaling optimal lr across token horizons},
  author={Bjorck, Johan and Benhaim, Alon and Chaudhary, Vishrav and Wei, Furu and Song, Xia},
  booktitle={International Conference on Learning Representations},
  volume={2025},
  pages={83640--83657},
  year={2025}
}

@article{loshchilov2017decoupled,
  title={Decoupled weight decay regularization},
  author={Loshchilov, Ilya and Hutter, Frank},
  journal={arXiv preprint arXiv:1711.05101},
  year={2017}
}

@inproceedings{timkey2021all,
  title={All bark and no bite: Rogue dimensions in transformer language models obscure representational quality},
  author={Timkey, William and Van Schijndel, Marten},
  booktitle={Proceedings of the 2021 Conference on Empirical Methods in Natural Language Processing},
  pages={4527--4546},
  year={2021}
}

@article{dettmers2022gpt3,
  title={Gpt3. int8 (): 8-bit matrix multiplication for transformers at scale},
  author={Dettmers, Tim and Lewis, Mike and Belkada, Younes and Zettlemoyer, Luke},
  journal={Advances in neural information processing systems},
  volume={35},
  pages={30318--30332},
  year={2022}
}

@article{deng2026rmnp,
  title={Rmnp: Row-momentum normalized preconditioning for scalable matrix-based optimization},
  author={Deng, Shenyang and Ouyang, Zhuoli and Pang, Tianyu and Liu, Zihang and Jin, Ruochen and Yu, Shuhua and Yang, Yaoqing},
  journal={arXiv preprint arXiv:2603.20527},
  year={2026}
}

@inproceedings{glentis2025minimalist,
  title={A minimalist optimizer design for LLM pretraining},
  author={Glentis, Athanasios and Li, Jiaxiang and Han, Andi and Hong, Mingyi},
  booktitle={ES-FoMo III: 3rd Workshop on Efficient Systems for Foundation Models},
  year={2025}
}

@article{zhang2026teon,
  title={Teon: Tensorized orthonormalization beyond layer-wise muon for large language model pre-training},
  author={Zhang, Ruijie and Zhao, Yequan and Liu, Ziyue and Wang, Zhengyang and Li, Dongyang and Su, Yupeng and Liu, Sijia and Zhang, Zheng},
  journal={arXiv preprint arXiv:2601.23261},
  year={2026}
}

@article{tropp2012user,
  title={User-friendly tail bounds for sums of random matrices},
  author={Tropp, Joel A},
  journal={Foundations of computational mathematics},
  volume={12},
  number={4},
  pages={389--434},
  year={2012},
  publisher={Springer}
}
